\documentclass[twoside,11pt]{article}
\usepackage{jmlr2e_draft}

\usepackage{amsmath, amssymb} 
\usepackage{stmaryrd}
\usepackage{epsfig,calc, graphicx}
\usepackage{wasysym}
\usepackage[usenames]{color}
\usepackage{url}
\usepackage{enumitem}
\usepackage[lined,algoruled,algosection]{algorithm2e}

\usepackage{float}
\usepackage{booktabs}

\newcommand\ip[2]{\langle #1, #2\rangle}
\newcommand\absip[2]{\left|\langle #1, #2\rangle\right|}
\newcommand\round[1]{\lfloor #1 \rceil}
\newcommand\floor[1]{\lfloor #1 \rfloor}
\newcommand\natoms{K}

\newcommand\sparsity{S}

\newcommand\eps{\varepsilon}

\newcommand\dico{\Phi}

\newcommand\atom{\phi}

\newcommand\pdico{\Psi}
\newcommand\ppdico{\bar\Psi}
\newcommand\patom{\psi}

\newcommand\ppatom{\bar\psi}

\newcommand\ratom{\gamma}

\newcommand\noise{r}

\newcommand\nsigma{\rho}

\newcommand\Sset{\mathbb{S}}

\newcommand\gap{{\gamma_{gap}}}
\newcommand\dynr{{\gamma_{dyn}}}
\newcommand\apperr{{\gamma_{app}}}
\newcommand\ncr{\gamma_{\nsigma}}
\newcommand\SNR{\operatorname{SNR}}

\newcommand\signop{\operatorname{sign}}
\newcommand\tr{\operatorname{tr}}

\newcommand{\R}{{\mathbb{R}}}

\newcommand{\E}{{\mathbb{E}}}
\newcommand{\I}{{\mathbb{I}}}
\renewcommand{\P}{{\mathbb{P}}}
 
\newcommand\toolboxlink{\url{https://www.uibk.ac.at/mathematik/personal/schnass/code/adl.zip}}

\begin{document}

\title{Dictionary learning - from local towards global and adaptive}

\author{\name Marie Christine Pali \email marie-christine.pali@uibk.ac.at\\
\name Karin Schnass \email karin.schnass@uibk.ac.at\\
\addr  Department of Mathematics\\
University of Innsbruck\\
Technikerstra\ss e 13\\
6020 Innsbruck, Austria}

\editor{}

\maketitle

\noindent
\begin{abstract}
This paper studies the convergence behaviour of dictionary learning via the Iterative Thresholding and K-residual Means (ITKrM) algorithm. On one hand it is proved that ITKrM is a contraction under much more relaxed conditions than previously necessary. On the other hand it is shown 
that there seem to exist stable fixed points that do not correspond to the generating dictionary, which can be characterised as very coherent. Based on an analysis of the residuals using these bad dictionaries, replacing coherent atoms with carefully designed replacement candidates is proposed. In experiments on synthetic data this outperforms random or no replacement and always leads to full dictionary recovery. Finally the question how to learn dictionaries without knowledge of the correct dictionary size and sparsity level is addressed. Decoupling the replacement strategy of coherent or unused atoms into pruning and adding, and slowly carefully increasing the sparsity level, leads to an adaptive version of ITKrM. In several experiments this adaptive dictionary learning algorithm is shown to recover a generating dictionary from randomly initialised dictionaries of various sizes on synthetic data and to learn meaningful dictionaries on image data.
\end{abstract}


\section{Introduction}\label{sec:intro}

\noindent
The goal of dictionary learning is to decompose a data matrix $Y = (y_1,\ldots, y_N)$, where $y_n \in \mathbb{R}^d$, into
a dictionary matrix $\Phi = (\varphi_1,\ldots,\varphi_K)$, where each column also referred to as atom is normalised, $\|\varphi_k\|_2=1$
and a sparse coefficient matrix $X= (x_1,\ldots,x_N)$,
\begin{align}\label{abstractdl}
	\quad Y \approx \dico X \quad \mbox{and}\quad X \mbox{ sparse}.
\end{align}
The compact data representation provided by a dictionary can be used for data restoration, such as denoising or reconstruction from incomplete information, \cite{doelte06,masazi08,mabapo12} and data analysis, such as blind source separation, \cite{olsfield96, lese00,krra00, kreutz03}. Due to these applications dictionary learning is of interest to both the signal processing community, where it is also known as sparse coding, and the independent component analysis (ICA) and the blind source separation (BSS) community, where it is also known as sparse component analysis. 
It also means that there are not only many algorithms to choose from, \cite{olsfield96, ahelbr06, enaahu99, kreutz03, lese00, mabaposa10, sken10, mabapo12, rudu12, nayeoz16}, but also that theoretical results have started to accumulate, \cite{grsc10, spwawr12, argemo13, aganjaneta13, sc14, sc14b, bagrje14, bakest14, argemamo15, sc15, suquwr17a, suquwr17b, bach17,quzhli19}. As our reference list grows more incomplete every day, we point to the surveys \cite{rubrel10, sc15imn} as starting points for digging into algorithms and theory, respectively.
\\
One way to concretise the abstract formulation of the dictionary learning problem in \eqref{abstractdl} is to formulate it as optimisation programme.
For instance, choosing a sparsity level~$S$ and a dictionary size~$K$, we define $\mathcal X_S$ to be the set of all columnwise $S$-sparse coefficient matrices, $\mathcal D_K$ to be the set of all dictionaries with $K$ atoms and for some $p\geq 1$ try to find
\begin{align}
	\underset{\Psi \in \mathcal D_K, X\in \mathcal X_S}{\operatorname{argmin}} \sum_n \| y_n - \pdico x_n \|_2^p. 
\end{align}
Unfortunately this problem is highly non-convex and as such difficult to solve even in the simplest and most commonly used case $p=2$. However, randomly initialised alternating projection algorithms, which 
alternate between (trying to) find the best dictionary $\Psi$, based on coefficients $X$, and the best coefficients $X$,
based on a dictionary $\Psi$, such as K-SVD (K Singular Value Decompositions) for $p=2$, \cite{ahelbr06}, and ITKrM (Iterative Thresholding and K residual Means) related to $p=1$, \cite{sc15}, tend to be very successful on synthetic data - usually recovering 90 to 100\% of all atoms - and to provide useful dictionaries on image data.\\
Apart from needing both the sparsity level and the dictionary size as input, the main drawback of these algorithms is that - assuming that the data $Y$ is synthesized from a generating dictionary~$\Phi$ and randomly drawn $S$-sparse coefficients $X$ - they have almost no (K-SVD) or comparatively weak (ITKrM) theoretical dictionary recovery guarantees. This is in sharp contrast to more involved algorithms, which - given the correct $S,K$ - have gobal recovery guarantees but due to their computational complexity can at best be used in small toy examples, \cite{argemo13, aganne13, bakest14}.\\
There are two interesting exceptions. In \cite{suquwr17a, suquwr17b}, Sun, Qu and Wright study an algorithm based on gradient descent with a Newton trust region method to escape saddle points and prove recovery if the generating dictionary is a basis. In \cite{quzhli19}, Qu et. al. study an $\ell^4$-norm opimisation programme with spherical constraints for overcomplete dictionary learning. They show that every local minimizer is close to an atom of the target dictionary and that around every saddle point is a large region with negative curvature. These results together with several results in machine learning which prove that non-convex problems can be well behaved, meaning all local minima are global minima, give rise to hope that a similar result can be proved for learning overcomplete dictionaries via alternating projection. \\
{\bf Contribution:} In this paper we first study the contractive areas of ITKrM and show that the algorithm contracts towards the generating dictionary under much relaxed conditions compared to those from \cite{sc15}. We then have a closer a look at experiments where the learned dictionaries do not coincide with the generating dictionary. These spurious dictionaries, which at least experimentally are fixed points, have a very special structure that violates the theoretical conditions for contractivity, that is, they contain two nearly identical atoms. Unfortunately, these experimental findings indicate that for alternating projection methods not all fixed points correspond to the generating dictionary. \\
However, based on an analysis of the residuals at dictionaries with the discovered structure, we develop a strategy for finding good candidates to replace coherent atoms.
With the help of these replacement candidates, we then tackle one of the most challenging problems in dictionary learning - the automatic choice of the sparsity level $S$ and the dictionary size $K$. This leads to a version of ITKrM that adapts both the sparsity level and the dictionary size in each iteration. Synthetic experiments show that the resulting algorithm is able recover a generating dictionary without prescribing its size or the sparsity level even in the presence of noise and outliers. Complementary experiments on image data further show that the algorithm learns sensible dictionaries even in practice, where several synthetic assumptions such as homogeneous use of all atoms are unlikely to hold.
\\
{\bf Organisation:}
In the next section we summarise our notational conventions and introduce the sparse signal model on which all our theoretical results are based.
In Section~\ref{sec:behaviour} we familiarise the reader with the ITKrM algorithm and existing convergence results. We analyse the limitations of existing proofs, develop strategies to overcome them and prove that ITKrM is a contraction towards the generating dictionary on an area much larger than indicated by the convergence radius in \cite{sc15}. 
To see whether the non-contractive areas are only an artefact of our proof strategy, we conduct several small experiments. These show that indeed there are fixed points of ITKrM, which are not equivalent to the generating dictionary through reordering and sign flips and violate our conditions for contractivity; in particular, they are very coherent. 
In Section~\ref{sec:replace} we analyse the residuals at such bad dictionaries and use those insights to develop a strategy for learning good replacement candidates for coherent atoms. The resulting algorithm is then tested and compared to random replacement on synthetic data.\\
In Section~\ref{sec:adaptive} we then address the big problem how to learn dictionaries without being given the generating sparsity level and dictionary size. This is done by slowly increasing the sparsity level and by decoupling the replacement strategy into separate pruning of the dictionary and adding of promising replacement candidates. Numerical experiments show that the resulting algorithm can indeed recover the generating dictionary from initialisations with various sizes on synthetic data and learn meaningful dictionaries on image data. \\
In the last section we will sketch how the concepts leading to adaptive ITKrM can be extended to other algorithms such as K-SVD or MOD. Finally, based on a discussion of our results, we will map out future directions of research.

\section{Notations and Sparse Signal Model}\label{sec:notations}

Before we hit the strings, we will fine tune our notation and introduce some definitions. Usually subscripted letters will denote vectors with the exception of $\eps, \alpha, \omega$, where they are numbers, for instance $x_n \in \R^K$ vs. $\eps_k \in \R$, however, it should always be clear from the context what we are dealing with. \\
For a matrix $M$ we denote its (conjugate) transpose by $M^\star$ and its Moore-Penrose pseudo-inverse by $M^\dagger$. We denote its operator norm by $\|M\|_{2,2}=\max_{\|x\|_2=1}\|Mx\|_2$ and its Frobenius norm by $\|M\|_F= \tr(M^\star M)^{1/2}$, remember that we have $\|M\|_{2,2}\leq \|M\|_F$.\\
We consider a {\bf dictionary} $\dico$ a collection of $K$ unit norm vectors $\atom_k\in \R^d$, $\|\atom_k\|_2=1$. By abuse of notation we will also refer to the $d \times K$ matrix collecting the atoms as its columns as the dictionary, that is, $\dico=(\atom_1, \ldots \atom_K)$. The maximal absolute inner product between two different atoms is called the {\bf coherence} $\mu(\dico)$ of a dictionary, $\mu(\dico)=\max_{k \neq j}|\ip{\atom_k}{\atom_j}|$.\\
By $\dico_I$ we denote the restriction of the dictionary to the atoms indexed by $I$, that is, $\dico_I=(\atom_{i_1},\ldots, \atom_{i_\sparsity} )$, $i_j\in I$, and by $P(\dico_I)$ the orthogonal projection onto the span of the atoms indexed by $I$, that is, $P(\dico_I)=\dico_I \dico_I^\dagger$. Note that in case the atoms indexed by $I$ are linearly independent we have $\dico_I^\dagger = (\dico_I^\star \dico_I)^{-1} \dico_I^\star$. We also define $Q(\dico_I)$ to be the orthogonal projection onto the orthogonal complement of the span of $\dico_I$, that is, $Q(\dico_I) = \I_d - P(\dico_I)$, where $\I_d$ is the identity operator (matrix) in $\R^d$.\\
(Ab)using the language of compressed sensing we define $\delta_I(\dico)$ as the smallest number such that all eigenvalues of $\dico^\star_I\dico_I$ are included in $[1-\delta_I(\dico), 1+\delta_I(\dico)]$ and the {\bf isometry constant} $\delta_S(\dico)$ of the dictionary as $\delta_S(\dico):=\max_{|I|\leq S} \delta_I(\dico)$. When clear from the context we will usually omit the reference to the dictionary. For more details on isometry constants see for instance \citep{carota06}.\\
For a (sparse) signal $y = \sum_k \atom_k x_k$ we will refer to the indices of the $S$ coefficients with largest absolute magnitude as the $S$-support of $y$. Again, we will omit the reference to the sparsity level $S$ if clear from the context.\\
To keep the sub(sub)scripts under control we denote the {\bf indicator function of a set} $\mathcal V$ by $\chi(\mathcal V,\cdot)$, that is $\chi(\mathcal V, v)$ is one if $v \in \mathcal V$ and zero else. The set of the first $S$ integers we abbreviate by $\mathbb{S} = \{1,\ldots, S\}$.\\
We define the {\bf distance} of a dictionary $\pdico$ to a dictionary $\dico$ as
\begin{align}
	d(\dico,\pdico):=\max_k \min_\ell \|\atom_k \pm \patom_\ell\|_2 = \max_k \min_\ell \sqrt{2-2|\ip{\atom_k}{\patom_\ell}|}.
\end{align}
Note that this distance is not a metric since it is not symmetric. For example, if $\dico$ is the canonical basis and $\pdico$ is defined by $\patom_i=\atom_i$ for $i\geq 3$, $\patom_1=(e_1 + e_2)/\sqrt{2}$, and $\patom_2=\sum_i \atom_1/\sqrt{d}$ then we have $d(\dico,\pdico)= 1/\sqrt{2}$ while $d(\pdico,\dico)=\sqrt{2-2/\sqrt{d}}$. The advantage is that this distance is well defined also for dictionaries of different sizes.
A {\bf symmetric distance} between two dictionaries $\dico,\pdico$ of the same size could be defined as the maximal distance between two corresponding atoms, that is,
\begin{align}
	d_s(\dico,\pdico):=\min_{p \in \mathcal P} \max_k \|\atom_k\pm \patom_{p(k)}\|_2,
\end{align}
where $\mathcal P$ is the set of permutations of $\{1,\ldots, K\}$. The distances are equivalent whenever there exists a permutation $p$ such that after rearrangement, the cross-Gram matrix $\dico^\star \pdico$ is diagonally dominant, that is, $\min_ k \absip{\atom_k}{\patom_k} > \max_{k \neq j} \absip{\atom_k}{\patom_j}$. Since the main assumption for our results will be such a diagonal dominance we will state them in terms of the easier to calculate asymmetric distance and assume that $\pdico$ is already signed and rearranged in a way that $d(\dico,\pdico)=\max_k \|\atom_k-\patom_k\|_2$. We then use the abbreviations $\alpha_{\min}  = \min_ k \absip{\atom_k}{\patom_k}$ and $\alpha_{\max} = \max_ k \absip{\atom_k}{\patom_k}$.
The maximal absolute inner product between two non-corresponding atoms will be called the {\bf cross-coherence} $\mu(\dico,\pdico)$ of the two dictionaries, $\mu(\dico,\pdico)=\max_{k \neq j}|\ip{\atom_k}{\patom_j}|$.\\
We will also use the following decomposition of a dictionary $\pdico$ into a given dictionary $\dico$ and a perturbation dictionary $Z$. If $d(\pdico,\dico)=\eps$ we set $\|\patom_k - \atom_k\|_2=\eps_k$, where by definition $\max_k \eps_k = \eps$. We can then find unit vectors $z_k$ with $\langle \atom_k,z_k\rangle = 0$ such that 
\begin{align}\label{atomdecomp}
	\patom_k = \alpha_k \atom_k + \omega_k z_k, \quad \mbox{for}, \quad\alpha_k:= 1-\eps^2_k/2  \quad\mbox {and} \quad\omega_k := (\eps_k^2 - \eps_k^4/4)^{\frac{1}{2}}.
\end{align}
Note that if the cross-Gram matrix $\dico^\star \pdico$ is diagonally dominant we have $\alpha_{\min}= \min_k \alpha_k$, $\alpha_{\max}= \max_k \alpha_k$ and $d(\pdico,\dico) = \sqrt{2-2\alpha_{\min}}$.

\subsection{Sparse signal model}\label{sec:signalmodel}

As basis for our results we use the following signal model, already used in \citep{sc14, sc14b, sc15}. 
Given a $d\times K$ dictionary $\dico$, we assume that the signals are generated as
\begin{align}\label{noisymodel1}
y=\frac{ \dico x +\noise}{\sqrt{1+\|\noise \|_2^2}},
\end{align}
where $x\in \R^K$ is a sparse coefficient sequence and $\noise\in \R^d$ is some noise. We assume that $\noise$ is a centered subgaussian vector with parameter $\nsigma$, that is, $\E(\noise) = 0$ and for all vectors $v$ the marginals $\ip{v}{\noise}$ are subgaussian with parameter $\nsigma$, meaning they satisfy $\E (e^{t \ip{v}{\noise}}) \leq e^{t^2 \nsigma^2/2}$ for all $t>0$.\\
To model the coefficient sequences $x$ we first assume that there is a measure $\nu_c$ on a subset $\mathcal C$ of the positive, non increasing sequences with unit norm, meaning for $c \in\mathcal C$ we have $c(1)\geq c(2) \ldots \geq c(K) > S$ and $\|c\|_2=1$. A coefficient sequence $x$ is created by drawing a sequence $c$ according to $\nu_c$, and both a permutation $p$ and a sign sequence $\sigma$ uniformly at random and setting $x =x_{c, p, \sigma}$, where $x_{c, p, \sigma}(k)= \sigma(k) c(p(k))$. The signal model then takes the form
\begin{align}\label{noisymodel2}
	y=\frac{ \dico x_{c, p,\sigma} +\noise}{\sqrt{1+\|\noise \|_2^2}}.
\end{align}
Using this model it is quite simple to incorporate sparsity via the measure $\nu_c$. To model approximately $S$-sparse signals we require
that the $S$ largest absolute coefficients, meaning those inside the support $I = p^{-1}(\Sset)$, are well balanced and much larger than the remaining ones outside the support. Further, we need that the expected energy of the coefficients outside the support is relatively small and that the sparse coefficients are well separated from the noise. Concretely we require that almost $\nu_c$-surely we have
\begin{align}
	\frac{c(1)}{c(S)} \leq \dynr, \qquad \frac{c(S+1)}{c(S)} \leq \gap, \qquad \frac{\| c(\Sset^c)\|_2}{c(1)}  \leq \apperr \qquad \mbox{and} \qquad \frac{\rho}{c(S)} \leq \ncr.
\end{align}
We will refer to the worst case ratio between coefficients inside the support, $\dynr$, as dynamic (sparse) range and to the worst case ratio between coefficients outside the support to those inside the support, $\gap$, as the (sparse) gap. Since for a noise free signal the expected squared sparse approximation error is
$$\E(\| \sum_{k \notin I } \sigma(k) c(p(k))\atom_k \|^2_2 )= \| c(\Sset^c)\|^2_2,$$ 
we will call $\apperr$ the relative (sparse) approximation error. Finally, $\ncr$ is called the noise to (sparse) coefficient ratio. \\
Apart from these worst case bounds we will also use three other signal statistics,
\begin{align}
	\gamma_{1,S} :=  \E_c\left(\| c(\Sset)\|_1)\right), \qquad \gamma_{2,S} :=  \E_c\left(\| c(\Sset)\|_2^2\right), \qquad C_r := \E_r\left( \frac{1}{\sqrt{1+\|r\|_2^2}}\right).
\end{align}
The constant $\gamma_{1,S}$ helps to characterise the average size of the sparse coefficients, $\gamma_{1,S} = \E( |x_i| : i\in I) \cdot S \leq \sqrt{S}$, while $\gamma_{2,S}$ characterises the average sparse approximation quality, $\gamma_{2,S}=\E( \|\dico_I x_I\|_2^2 )\leq 1$. The noise constant can be bounded by 
\begin{align}\label{Cr_bound}
	C_r \geq \frac{1- e^{-d}}{\sqrt{1+5d\nsigma^2}},
\end{align}
and for large $\nsigma$ approaches the signal-to-noise ratio, $C^2_r \approx \frac{1}{d\nsigma^2} \approx \frac{\E(\| \dico x \|_2^2 ) }{\E(\|r\|_2^2) }$, see \citep{sc14b} for details.\\
To get a better feeling for all the involved constants, we will calculate them for the case of perfectly sparse signals where $c(i)=1/\sqrt{S}$ for $i\leq S$ and $c(i)=0$ else. We then have $\dynr =1$, $\gap =0$ and $\apperr=0$ as well as $\gamma_{1,S}=\sqrt{S}$ and $ \gamma_{2,S}=1$. In the case of noiseless signals we have $C_r =1$ and $\ncr = 0$. In the case of Gaussian noise the noise-to-coefficient ratio is related to the signal-to-noise ratio via $\SNR =  S/(\ncr^{2} d)$.

\section{Global behaviour patterns of ITKrM}\label{sec:behaviour}

The iterative thresholding and K residual means algorithm (ITKrM) was introduced in \citep{sc15} as modification of its much simpler predecessor ITKsM, which uses signal means instead of residual means. As can be seen from the summary in Algorithm~\ref{algo:itkrm} the signals can be processed sequentially, thus making the algorithm suitable for an online version and parallelisation. 
\begin{algorithm}[bt]
	\small
	\BlankLine
	\caption{ITKrM (one iteration)} \label{algo:itkrm}
	\BlankLine
	\KwIn{$\pdico, Y, S$ \tcp*{dictionary, signals, sparsity}}  
	\BlankLine
	Initialise: $\ppdico = 0$ \;
	\BlankLine
	\ForEach{$ n $}{
		\BlankLine
		$I_{n}^t= \arg\max_{I: | I |=S} \| \pdico_I^\star y_n\|_1$ \tcp*{thresholding}
		$ a_n =y_n - P(\pdico_{I_n^t}) y_n$ \tcp*{residual}
		\ForEach{$k \in I_{n}^t$}{
			$\ppatom_{k} \leftarrow \ppatom_{k} + \big[a_n + P(\patom_k) y_n\big] \cdot \signop(\ip{\patom_k}{y_n})$ \tcp*{atom update}
		}
	}
	\BlankLine
	$\pdico \leftarrow \left( \ppatom_{1}/\| \ppatom_{1} \|_2, \dots, \ppatom_{\natoms}/\| \ppatom_{\natoms} \|_2 \right)$ \tcp*{atom normalisation}
	\BlankLine
	\KwOut{$\pdico$ }
	
\end{algorithm}
The determining factors for the computational complexity are the matrix vector products $\pdico^\star y_n$ between the current estimate of the dictionary $\pdico$ and the signals, $O(dKN)$, and the projections $P(\pdico_{I_{n}^t}) y_n$. If computed with maximal numerical stability these would have an overall cost $O(S^2 dN)$, corresponding to the QR decompositions of $\pdico_{I^t_n}$. However, since usually the achievable precision in the learning is limited by the number of available training signals rather than the numerical precision, it is computationally more efficient to precompute the Gram matrix $\pdico^\star \pdico$ and calculate the projections less stably via the eigenvalue decompositions of $\pdico_{I^t_n}^\star \pdico_{I^t_n}$, corresponding to an overall cost $O(S^3N)$. Another good property of the ITKrM algorithm is that it is proven to converge locally to a generating dictionary. This means that if the data is homogeneously $S$-sparse in a dictionary $\dico$, where $S\lesssim \mu^{-2}$, and we initialise with a dictionary $\pdico$ within radius $O(1/\sqrt{S})$, $d(\pdico,\dico)\lesssim 1/\sqrt{S}$, then ITKrM using $N = O(K\log K)$ samples in each iteration will converge to the generating dictionary, \citep{sc15}. In simulations on synthetic data ITKrM shows even better convergence behaviour. Concretely, if the atoms of the generating dictionary are perturbed with vectors $z_k$ chosen uniformly at random from the sphere, $\patom_k = \alpha_k \atom_k + \omega_k z_k$, ITKrM converges also for ratios $\alpha_k: \omega_k = 1: 4$. For completely random initialisations, $\patom_k=z_k$, it finds between 90\% and 100\% of the atoms - depending on the noise and sparsity level.\\
Last but not least, ITKrM is not just a pretty toy with theoretical guarantees but on image data produces dictionaries of the same quality as K-SVD in a fraction of the time, \citep{nasc18}.\\
Considering the good practical performance of ITKrM, it is especially frustrating that we only get a convergence radius of size $O(1/\sqrt{S})$, while for its simpler cousin ITKsM, which when initialised randomly performs much worse both on synthetic and image data, we can prove a convergence radius of size $O(1/\sqrt{\log K})$. Therefore, in the next section we will take a closer look at the two algorithms and the differences in the convergence proofs. This will allow us to show that ITKrM behaves well on a much larger area.

\subsection{Contractive areas of ITKrM}

To better understand the idea behind the convergence proofs we first rewrite the atom update formula before normalisation, which for one iteration of ITKrM becomes
\begin{align*}
\ppatom_k = \sum_{n:k\in I_n^t} \big[\I_d - P(\pdico_{I_{n}^t})+ P(\patom_k)\big]y_n \cdot \signop(\ip{\patom_k}{y_n}) ,
\end{align*}
while for ITKsM we can take the formula above and simply ignore the operators in the square brackets. Adding and replacing some terms we expand the sum as
\begin{align*}
&\left.
\begin{array}{l}\displaystyle
\ppatom_k\quad = \quad \sum_{n:k\in I_n^t}\big[\I_d - P(\pdico_{I_{n}^t})+ P(\patom_k)\big]y_n \cdot \signop(\ip{\patom_k}{y_n}) \\
\displaystyle \hspace{4.2cm}  - \sum_{n: k\in I_n}\big[\I_d - P(\pdico_{I_{n}})+ P(\patom_k)\big]y_n \cdot \sigma_n(k)
\end{array}
\quad \right\} \text{$S_1$}
\\
&\left.
\begin{array}{l}\displaystyle
\phantom{\ppatom_k}\quad \quad + \sum_{n:k\in I_n}\big[\I_d - P(\pdico_{I_{n}})+ P(\patom_k)\big]y_n \cdot \sigma_n(k)\quad
\\
\displaystyle
\hspace{4.2cm} - \sum_{n:k\in I_n}\big[\I_d - P(\dico_{I_{n}})+ P(\atom_k)\big]y_n \cdot \sigma_n(k)
\end{array}
\quad\right\} \text{$S_2$}
\\
&\left.
\begin{array}{l}\displaystyle
\phantom{\ppatom_k}\quad  \quad+ \sum_{n: k\in I_n} \big[y_n - P(\dico_{I_{n}}) y_n + P(\atom_k) y_n\big] \cdot \sigma_n(k).\hspace{2.4cm}
\end{array} 
\quad \:\right\} \text{$S_3$}
\end{align*}
The term $S_1$ captures the errors thresholding makes in estimating the supports $I_n$ and signs $\sigma_n(k)$ when using the current estimate $\pdico$. It is (sufficiently) small as long as $d(\dico,\pdico) \lesssim 1/\sqrt{\log K}$. The second term $S_2$ captures the difference between the residual using the current estimate and the true dictionary, which is small as long as $d(\dico,\pdico) \lesssim 1/\sqrt{S}$. In expectation the last term is simply a multiple of the true atom $\atom_k$, so as long as the number of signals~$N$ is large enough, the last term will concentrate arbitrarily close to $\atom_k$.\\
As we can see, the main constraint on the convergence radius for ITKrM stems from the second term $S_2$, which simply vanishes in case of ITKsM. The problem is that we need to invert the $S\times S$ matrix $\pdico_{I_{n}}^\star \pdico_{I_n}$, which is a perturbed version of the matrix $\dico_{I_{n}}^\star \dico_{I_n}$. If the difference between the dictionaries scales as $d(\dico,\pdico) \approx 1/\sqrt{S}$, there exist perturbations such that $\pdico_{I_{n}}^\star \pdico_{I_n}$ is ill conditioned even if $\dico_{I_{n}}^\star \dico_{I_n}$ is not. \\
However, if the current dictionary estimate $\pdico$ is itself a well-conditioned and incoherent matrix, results on the conditioning of random subdictionaries, \cite{tr08, chda12}, tell us that for most possible supports $I_n$, $\pdico_{I_{n}}^\star \pdico_{I_n}$ will be close to the identity as long as $S\lesssim d/\log{K}$.
This means that the term $S_2$ should be small as long as the current estimate $\pdico$ is well-conditioned and incoherent, a property that we can check after each iteration.\\
Therefore, the next question is if also the first term $S_1$ can be controlled for a larger class of dictionaries $\pdico$. In our previous estimates we bounded the error per atom by the probability of thresholding failing multiplied with the norm bound on the difference of the projections. While simple, this strategy is quite crude as it assigns any error of thresholding to all atoms. However, an atom $\bar\patom_k$ is only affected by a thresholding error if either $k$ was in the original support or if $k$ is not in the original support but is included in the thresholded support. Further, we can take into account that by perturbing an atom $\atom_k$, meaning $\patom_k = \alpha_k \atom_k+  \omega_k z_k$, its coherence to one other atom $\atom_\ell$ may increase dramatically - to the point of it being a better approximant than $\patom_\ell$, that is, if $z_k \approx \atom_\ell$ we get $\ip{\atom_k}{\atom_\ell} \ll \ip{\patom_k}{\atom_\ell} \approx \ip{\patom_\ell}{\atom_\ell}$. However, if the original $\dico$ itself is well-conditioned, $\patom_k$ cannot become coherent to all (many) other atoms.\\
Indeed using both of these ideas we get a refined result characterising the contractive areas of ITKrM. To keep the flow of the paper we will state it in an informal version and refer the reader to Appendix~\ref{appendix} for the exact statement and its proof.

\begin{theorem}\label{maintheorem}
	Assume that the sparsity level of the training signals scales as $S \lesssim \mu(\dico)^{-2} /\log K$ and that the number of training signals scales as $N \approx SK\log K $. Further, assume that the coherence and operator norm of the current dictionary estimate $\pdico$ satisfy,
	\begin{align}\label{pdico_mu_B}
	\mu(\pdico) \lesssim \frac{1}{\log K} \quad \mbox{ and } \quad \| \pdico \|_{2,2}^2 \lesssim \frac{K}{ S \log K} .
	\end{align}
	If the distance of $\pdico$ to the generating dictionary $\dico$ satisfies either \\
	\begin{enumerate}
		\item[\it a)] $\frac{1}{\sqrt{S}}\lesssim d(\pdico,\dico) \lesssim \frac{1}{\sqrt{ \log K}}$ or\\
		\item[\it b)] $d(\pdico,\dico) \gtrsim \frac{1}{\sqrt{ \log K}}$ but the cross-Gram matrix $\dico^\star \pdico$ is diagonally dominant in the sense that
		\begin{align}\label{pdico_cross}
		\min_k|\ip{\atom_k}{\patom_k}|& \gtrsim \log K\cdot  \max\left \{ \mu(\dico,\pdico), \| \dico \|_{2,2}\sqrt{S/(K\log K)}\right\},
		\end{align}
	\end{enumerate}
	then one iteration of ITKrM will reduce the distance by at least a factor $\kappa < 1$, meaning
	\begin{align}
	d(\ppdico,\dico) < \kappa \cdot  d(\pdico,\dico).\notag
	\end{align}
\end{theorem}

The first part of the theorem simply says that, excluding dictionaries $\pdico$ that are coherent or have large operator norm, ITKrM is a contraction on a ball of radius $1/\sqrt{\log K}$ around the generating dictionary $\dico$. To better understand the second part of the theorem, we have a closer look at the conditions on the cross-Gram matrix $\dico^\star \pdico$ in \eqref{pdico_cross}. The fact that the diagonal entries have to be larger than $\| \dico \|_{2,2}\sqrt{(S\log K)/K}$ puts a constraint on the admissible distance $d(\dico,\pdico)$ via the relation $d(\dico,\pdico)^2 = 2 - 2\min_k|\ip{\atom_k}{\patom_k}|$. For a well-conditioned dictionary, satisfying $\| \dico \|^2_{2,2}\approx K/d$, this means that 
\begin{align}
	d(\dico, \pdico) \lesssim \left(2 - 2\sqrt{\frac{S\log K}{d}}\right)^{1/2}.
\end{align} 
Considering that the maximal distance between two dictionaries is $\sqrt{2}$, this is a large improvement over the admissible distance $1/\sqrt{\log K}$ in a). However, the additional price to pay is that also the intrinsic condition on the cross-Gram matrix needs to be satisfied, 
\begin{align}
	\min_k|\ip{\atom_k}{\patom_k}|& \gtrsim  \log K\cdot  \max_{j\neq k}|\ip{\atom_k}{\patom_j}|.
\end{align}
This condition captures our intuition that two estimated atoms should not point to the same generating atom and provides a bound for sufficient separation.\\
One thing that has to be noted about the result above is that it does not guarantee convergence of ITKrM since it is only valid for one iteration. To prove convergence of ITKrM, we need to additionally prove that $\ppdico$ inherits from $\pdico$ the properties that are required for being a contraction, which is part of our future goals. Still, the result goes a long way towards explaining the good convergence behaviour of ITKrM.\\
For example, it allows us to briefly sketch why the algorithm always converges in experiments where the initial dictionary is a large but random perturbation of a well-behaved generating dictionary $\dico$ with coherence $\mu(\dico) \approx 1/\sqrt{d}$ and operatornorm $\|\dico\|^2_{2,2} \approx K/d$.  If $\patom_k = \alpha_k \atom_k + \omega_k z_k$, where the perturbation vectors $z_k$ are drawn uniformly at random from the unit sphere orthogonal to $\atom_k$, then with high probability for all $j\neq k$ we have 
\begin{align}
	|\ip{\atom_k}{z_j}| \lesssim \sqrt{\log K/d}\quad \mbox{and}\quad |\ip{z_k}{z_j}| \lesssim \sqrt{\log K/d}
\end{align} and consequently for all possible $\alpha_k$
\begin{align}
	\mu(\pdico) \lesssim \sqrt{4\log K/d} \quad \mbox{and}\quad \mu(\dico,\pdico) \lesssim \sqrt{2\log K/d}.
\end{align} 
Also with high probability the operator norm of the matrix $Z = (z_1, \ldots z_K)$ is bounded by $\|Z\|_{2,2} \lesssim \sqrt{\log K}$, \citep{tr12}, so that for $\pdico$ we get $\|\pdico\|_{2,2} \lesssim \sqrt{K/d} + \sqrt{\log K}$, again independent of $\alpha_k$. Comparing these estimates with the requirements of the theorem we see that for moderate sparsity levels, $S\geq \log K$, we get a contraction whenever
\begin{align}
	\alpha_{\min} \gtrsim \sqrt{\frac{S (\log K)^2}{d}} \qquad \Leftrightarrow \qquad d(\dico, \pdico) \lesssim \left(2 - 2\sqrt{\frac{S (\log K)^2}{d}}\right)^{1/2}.
\end{align}
A fully random initialisation will have small coherence and operator norm with high probability. While it is also exponentially more likely to satisfy the cross-coherence property than to be within distance $1/\sqrt{S}$ or $1/\sqrt{\log K}$ to the generating dictionary, the absolute probability of having the cross-coherence property is still very small. This leads to the question whether in practice the cross-coherence property is actually necessary for convergence.  
Experiments in the next section will provide evidence that it is practically relevant, in the sense, that whenever ITKrM does not recover the generating dictionary, it produces a dictionary not satisfying the cross-coherence property.

\subsection{Bad dictionaries}

From \citep{sc15} we know that ITKrM is most likely to not recover the full dictionary from a random initialisation when the signals are very sparse ($S$ small) and the noiselevel is small. Since we want to closely inspect the resulting dictionaries, we only run a small experiment in $\R^{32}$, where we try to recover a very incoherent dictionary from 2-sparse vectors\footnote{All experiments and resulting figures can be reproduced using the matlab toolbox available at \toolboxlink}. The dictionary, containing 48 atoms, consists of the Dirac basis and the first half of the vectors from the Hadamard basis, and as such has coherence $\mu =1/\sqrt{32}\approx0.18$. The signals follow the model in \eqref{noisymodel2}, where the coefficient sequences $c$ are constructed by chosing $b \in [0.9,1]$ uniformly at random and setting $c_1 = 1/\sqrt{1+b^2}; c_2= b c_1$ and $c_j=0$ for $j\geq3$. The noise is chosen to be Gaussian with variance $\nsigma^2 =1/ (16d)$, corresponding to $\SNR=16$. Running ITKrM with 20000 new signals per iteration for 25 iterations and 10 different random initialisations we recover 4 times 46 atoms and 6 times 44 atoms. 
\begin{figure}[bt]\label{fig:missingatoms}
	\centering
	\begin{tabular}{cc}
		\includegraphics[width=0.4\textwidth]{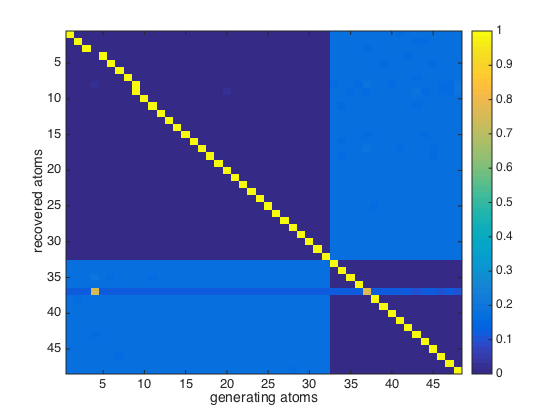}&\includegraphics[width=0.4\textwidth]{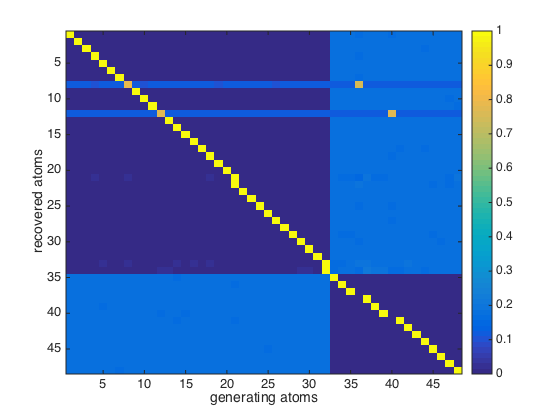}
	\end{tabular}
	\caption{Cross-Gram matrices $\pdico^\star \dico$ for recovered dictionaries with 2 (left) and 4 (right) missing atoms.}
\end{figure}
An immediate observation is that we always miss an even number of atoms. Taking a look at the recovered dictionaries - examples for recovery of 44 and 46 atoms are shown in Figure~\ref{fig:missingatoms} - we see that this is due to their special structure; in case of $2n$ missing atoms, we always observe that $n$ atoms of the generating dictionary are recovered twice and that $n$ atoms in the learned dictionary are a 1:1 linear combinations of 2 missing atoms from the generating dictionary, respectively. \\
So in the most simple case of 2 missing atoms (after rearranging and sign flipping the atoms in $\dico$) the recovered (and rearranged) dictionary $\pdico$ has the form
$$\pdico = (\atom_1, \atom_1, \atom_3 , \ldots , \atom_{K-1}, \patom_K) \quad \mbox{with} \quad \patom_K = \frac{ \atom_2 + \atom_K}{\sqrt{2 + 2\ip{\atom_2}{\atom_K} }}.$$
Looking back to our characterisation of the contractive areas in the last section, we see that such a dictionary or a slightly perturbed version of it clearly cannot have the necessary cross-coherence property with any reasonably incoherent dictionary. 
A complete proof showing that $\pdico$ is indeed a stable spurious fixed point is unfortunately too long to be included here but in preparation. In the meantime we refer the interested reader to the preprint version where a sketch of the proof can be found, \cite{sc18}. We also provide some intuition why dictionaries of the above form are stable in the next section.\\
Here we just want to add, that while a dictionary with coherence $\mu \approx 1$ clearly does not satisfy the conditions for contractivity, the reverse is not true. On the contrary two estimated atoms pointing to the same generating atom $\atom_j$ can be very incoherent even if they are both already quite close to $\atom_j$. For instance, if $\patom_{j_\pm} \approx \alpha_j \atom_j \pm \omega_j z_j$ where $ z_j $ is a balanced sum of the other atoms $z_j \approx \sum_{i\neq j} \atom_i \sigma(i)$, we have $|\ip{\patom_{j_+}}{\patom_{j_-}}| = \alpha_j^2 - \omega_j^2$, meaning approximate orthogonality
at $\alpha_j = 1/\sqrt{2}$. Using these ideas we can construct well-conditioned and incoherent initial dictionaries $\pdico$, with abritrary distances $d(\pdico, \dico) \gtrsim 1/\sqrt{2}$ to the generating dictionary, so that things will go maximally wrong, meaning we end up with a lot of double and 1:1 atoms. 
\begin{figure}[bt]\label{fig:badinit}
	\centering
	\begin{tabular}{ccc}
		\includegraphics[width=0.3\textwidth]{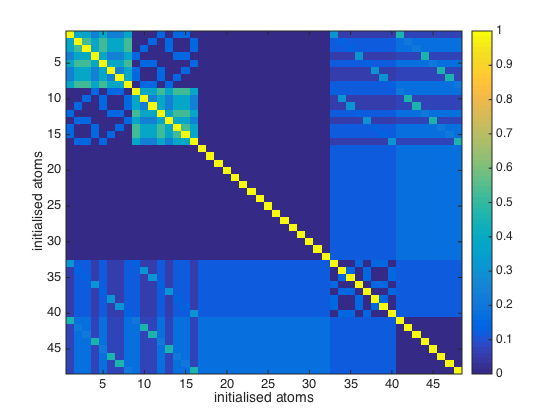}&
		\includegraphics[width=0.3\textwidth]{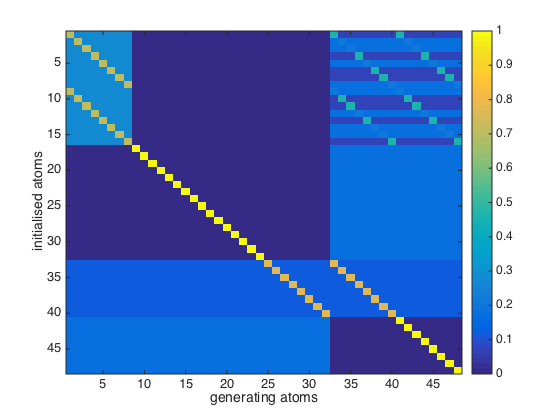}&
		\includegraphics[width=0.3\textwidth]{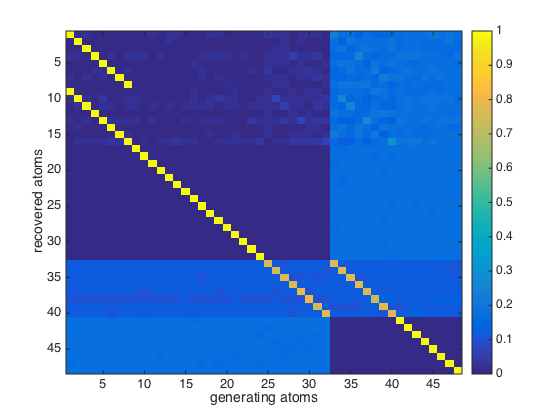}
	\end{tabular}
	\caption{Absolute values of the Gram matrix of a bad initial dictionary $\pdico^\star \pdico$ (left), its cross-Gram matrix with the generating dictionary $\pdico^\star \dico$ (middle) and the cross-Gram matrix of the recovered dictionary after 25 iterations of ITKrM $\bar\pdico^\star\dico$ (right).}
\end{figure}
Figure~\ref{fig:badinit} shows an example of a bad initialisation with coherence $\mu = 0.52$, leading to 16 missing atoms. The accompanying matlab toolbox provides more examples of these bad initialisations to observe convergence, to play around with and to inspire more evil constructions.
\\
Summarising the two last subsections we see that ITKrM may not be a contraction if the current dictionary estimate is too coherent, has large operator norm or if two atoms are close to one generating atom. Both coherence and operator norm of the estimate could be calculated after each iteration to check whether ITKrM is going in a good direction. Unfortunately, the diagonal dominance of the cross-Gram matrix, which prevents two estimated atoms to be close to the same generating atom, cannot be verified immediately. However, the most likely outcome of this situation is that both these estimated atoms converge to the same generating atom, meaning that eventually the estimated dictionary will be coherent. This suggests that in order to improve the global convergence behaviour of ITKrM, we should control the coherence of the estimated dictionaries. One strategy to incorporate incoherence into ITKrM could be adding a penalty for coherent dictionaries. The main disadvantages of this strategy, apart from the fact that ITKrM is not straightforwardly associated to an optimisation programme, are the computational cost and the fact that penalties tend to complicate the high-dimensional landscape of basins of attractions which further complicates convergence. Therefore, we will use a different strategy which allows us to keep the high percentage of correctly recovered atoms and even use the information they provide for identifying the missing ones: replacement.

\section{Replacement}\label{sec:replace}

Replacement of coherent atoms with new, randomly drawn atoms is a simple clean-up step that most dictionary learning algorithms based on alternating minimisation, e.g. K-SVD \citep{ahelbr06}, employ additionally in each iteration. While randomly drawing a replacement is cost-efficient and democratic, the drawback is that the new atom converges only very slowly or not at all to the missing generating atom. \\
To see why a randomly drawn replacement atom is not the best idea and what to do instead, we first have a look at the shape of the signal residuals
at one of the bad dictionaries, identified in the last section.

\subsection{Learning from bad dictionaries} 

We start with an analysis of thresholding in case the current dictionary $\pdico$ contains one double atom, $\patom_1=\patom_2 = \atom_1$, and one 1:1 atom, $\patom_K \propto \atom_2 + \atom_K$; for the other atoms we have $\patom_k = \atom_k$. We will also keep track of what would happen if we replaced one of the double atoms with a vector drawn uniformly at random from the unit sphere, which we label $\patom_0$. For simplicity we assume that the signals follow the sparse model in \eqref{noisymodel2} with constant coefficients $c_i = 1$ for $i\leq S$ and $c_i = 0$ for $i\geq 0$ and no noise,
and that $\ip{\atom_2}{\atom_K} \geq 0$. We also adopt the notation $I_{\ell \leftrightarrow k}: = ( I\setminus \{\ell\} )\cup \{k\}$.
Note that we have
\begin{align*} 
	&\absip{\patom_k}{\atom_k} = 1  \quad \text{for}  \quad k \notin \{2,K\}\\
	\text{and} \quad &\absip{\patom_k}{\atom_i} \leq \mu \quad \text{for} \quad k\neq K, \: i\notin\{1,2,k\}.
\end{align*}
We also have $\absip{\patom_2}{\atom_1} = 1$ as well as
\begin{align*} 
	&\absip{\patom_K}{\atom_2} = \absip{\patom_K}{\atom_K} = \frac{\absip{\atom_2 + \atom_K}{\atom_K} }{\sqrt{2 + 2 \ip{\atom_2}{\atom_K} }} = \sqrt{\frac{1 + \ip{\atom_2}{\atom_K} }{2}} \geq \frac{1}{\sqrt{2}}\\
	\text{and} \quad &\absip{\patom_K}{\atom_i}= \frac{\absip{\atom_2 + \atom_K}{\atom_i}}{\sqrt{2 + 2 \ip{\atom_2}{\atom_K}} } \leq \frac{2\mu}{\sqrt{2}} \leq \sqrt{2}\mu  \quad \text{for}  \quad i \notin \{2,K\} .
\end{align*}
Since $\patom_0$ is drawn uniformly at random from the $d$-dimensional unit sphere, we have for any fixed vector $v$ that
\begin{align*}
	\P(\absip{\patom_0}{v} \geq t ) \leq 2\exp\left(-\frac{t^2 d}{2}\right).
\end{align*}
This means that with very high probability $\absip{\patom_0}{\atom_k} \lesssim \sqrt{\log{K}/d}$ for all $k$.

\noindent If we draw a random support $I$ of size $S$, (a random permutation), then with probability $\binom{K-3}{S}/\binom{K}{S} \approx \left(1-\frac{S}{K}\right)^3$ it does not contain $1,2,K$, meaning $I\cap \{1,2,K\} = \emptyset$. We then have
\begin{align*}
k \in I:\quad &  \absip{\patom_k}{y} = | \ip{\patom_k}{\atom_k} \sigma_k c_k + \sum_{i\in I, i \neq k} \ip{\patom_k}{\atom_i} \sigma_i c_i | \geq 1 - (S-1)\mu,\\
k \in I^c\setminus \{0,K\}:\quad& \absip{\patom_k}{y} = | \sum_{i\in I}  \ip{\patom_{k}}{\atom_i}\sigma_i c_i | \leq S\mu, \\
k = K: \quad & \absip{\patom_k}{y} = | \sum_{i\in I} \ip{\patom_{K}}{\atom_i}\sigma_i c_i | \leq \sqrt{2}S\mu\\
k = 0: \quad & \absip{\patom_k}{y} = | \sum_{i\in I} \ip{\patom_{0}}{\atom_i}\sigma_i c_i | \leq S\sqrt{\log{K}/d}.
\end{align*}
So no matter whether the dictionary contains a double atom or a random replacement atom, thresholding will correctly identify the support, $I^t = I$, and the residual will be zero $$a= y - P(\pdico_{I^t})y = \dico_I x_I - P(\dico_{I})\dico_I x_I =0.$$
Next we have a look at supports containing $1$ but not $2$ or $K$, meaning $I\cap \{1,2,K\} = \{1\}$. Such a support is drawn with probability $\binom{K-3}{S-1}/\binom{K}{S} \approx\frac{S}{K}\left( 1-\frac{S}{K}\right)^2$. As before we get the following bounds
\begin{align*}
	k \in I \cup \{2\}: \quad &  \absip{\patom_k}{y} \geq 1 - (S-1)\mu,\\
	k \in I^c\setminus \{0,2,K\}:\quad& \absip{\patom_k}{y} \leq S\mu,
\end{align*}
as well as $\absip{\patom_K}{y} \leq \sqrt{2}S\mu$ and $\absip{\patom_0}{y}  \leq S\sqrt{\log{K}/d}$, which ensures that $I^t \subseteq I \cup \{2\}$. If we are lucky and $\absip{\patom_2}{y} = \absip{\patom_1}{y} = \min_{i\in I}  \absip{\patom_i}{y}$ or in the case of the dictionary with the random replacement atom, thresholding recovers $I^t=I$ or $I^t = I_{1 \leftrightarrow 2}\equiv I$ and the residual is again zero. If we are less lucky, we miss one of the relevant atoms indexed by $i \in I$, meaning $I^t = I_{i \leftrightarrow 2}$. In this case the residual will be close to $\atom_i$,
$$a= y - P(\pdico_{I^t})y = \dico_I x_I - P(\dico_{I\setminus\{i\}})\dico_I x_I = x_i (\atom_i - P(\dico_{I\setminus\{i\}}) \atom_i) \approx \pm \atom_i,$$
since $\| P(\dico_{I\setminus\{i\}}) \atom_i\|_2 \leq\mu\sqrt{S/ (1-S\mu)}$.
Note that for any fixed $i \notin \{1,2,K\}$, the probability that both $\{1,i \}\subseteq I$ is bounded by $\binom{K-3}{S-2}/\binom{K}{S}\approx \frac{S^2}{K^2}\left( 1-\frac{S}{K}\right) $, so a residual close to $\atom_i$ appears with probability less than $\frac{S^2}{K^2}$.
\\
Finally, we analyse what happens when the support contains $K$ but not $1$ or $2$, meaning $I\cap \{1,2,K\} = \{K\}$. Again this occurs with probability $\approx\frac{S}{K}\left( 1-\frac{S}{K}\right)^2$.
We then have
\begin{align*}
	k \in I \setminus \{K\}:\quad &  \absip{\patom_k}{y} = | \ip{\patom_k}{\atom_k} \sigma_k c_k + \sum_{i\in I, i \neq k} \ip{\patom_k}{\atom_i} \sigma_i c_i | \geq 1 - (S-1)\mu,\\
	k = K: \quad & \absip{\patom_k}{y} =  | \ip{\patom_K}{\atom_K} \sigma_K c_K + \sum_{i\in I, i \neq K} \ip{\patom_K}{\atom_i} \sigma_i c_i | \geq \frac{1-2(S-1)\mu}{\sqrt{2}},
\end{align*}
as well as $\absip{\patom_0}{y}  \leq S\sqrt{\log{K}/d}$ and $\absip{\patom_k}{y} \leq S\mu$ for all other atoms,
which shows that for both types of dictionary $\pdico$ (with double or random replacement atom) thresholding will recover $I^t =I$. For the residual we get
\begin{align*}
	a= y - P(\pdico_{I^t})y &= [\I_d - P(\pdico_{I})] \dico_I x_I = x_K [\I_d - P(\pdico_{I})]\atom_K\\
	&= x_K [\I_d - P(\pdico_{I})][P(\patom_K) + Q(\patom_K)]\atom_K\\
	&= x_K [\I_d - P(\pdico_{I})]Q(\patom_K)\atom_K \\
	& \approx \pm  \tfrac{1}{2} (\atom_K-\atom_2) ,  
\end{align*}
where we have used that $Q(\patom_K)\atom_K = \atom_K - \ip{\patom_K}{\atom_K} \patom_K = \tfrac{1}{2} (\atom_K-\atom_2) $ and
\begin{align*}
	\| P(\pdico_{I}) Q(\patom_K) \atom_K\|_2 &\leq \tfrac{1}{2} \cdot \| \pdico_{I}^\dagger \|_{2,2} \cdot \| \pdico_I^\star (\atom_K-\atom_2)\|_2 
	\leq \mu \cdot \sqrt{S/  (1-2S \mu)}.
\end{align*}
The analysis of the case $I\cap \{1,2,K\} = \{ 2\}$ is analogue to the one above and shows that the residual again is close to $\pm\tfrac{1}{2} (\atom_K-\atom_2)$. 
Summarising our analysis so far, we see that with probability $\left(1-\frac{S}{K}\right)^3$ the residual will be zero (or close to zero in the noisy case), with probability at most $\frac{S^2}{K^2}$ it will be close to $\atom_i$ for each $i\notin \{1,2,K\}$ and with probability $\frac{2S}{K}\left( 1-\frac{S}{K}\right)^2$ it will be close to a scaled version of 
$$\patom_{K'} = \frac{\atom_K-\atom_2}{\sqrt{2 - 2\ip{\atom_2}{\atom_K}}}.$$
Also after covering all supports except those where
$|I\cap \{1,2,K\}| \geq 2$, which together have probability $\approx\frac{3S^2}{K^2}$, we have not encountered a situation, where the randomly chosen atom $\patom_0$ would have been picked. Indeed a more detailed analysis to be found in \cite{pali21} shows that $\patom_0$ only has a chance to be picked if $I\cap \{1,2,K\} = \{2,K\}$. Moreover, for $\sigma_2 = \sigma_K$ we have $a = y - P(\pdico_{I\setminus\{2\}})y = 0$. So even if $\patom_0$ is picked, it will not be pulled in a useful direction. Indeed $\patom_0$ will only be picked and pulled in a good direction if $\sigma_2 = - \sigma_K$ and therefore 
$$a = y - P(\pdico_{I^t})y \approx y - P(\pdico_{I\setminus\{2,K\}})y \approx \alpha \cdot \patom_{K'} ,$$
for some scaling factor $\alpha = \pm \sqrt{2 - 2\ip{\atom_2}{\atom_K}}$.
This case is also the only case, where we have the chance of accidentally picking $\patom_1$ or $\patom_2$ and having them distorted in the direction of $\patom_{K'} $. Note that in case $I\cap \{1,2,K\} = \{1,2\}$ or $\{1,K\}$ we recover $I^t = I$ or $I^t = I_{K\leftrightarrow 2}$ and so $a \approx \pm \atom_2$ or $a \approx \pm \atom_K$, but due to the random signs the pull in these useful directions cancels out, and similarly for $\{1,2,K\}\subseteq I$. The intuition why configurations like $\pdico$ are stable is that this case is so rare that $\patom_1$ resp. $\patom_2$ cannot be sufficiently perturbed to change the behaviour of thresholding in the next iteration.\\
So rather than hoping for the at best unlikely distortion of $\patom_0$ or $\patom_{1/2}$ towards $\patom_{K'}$ we will use the fact that most non-zero residuals (or residuals not just consisting of noise) are close to scaled versions of $\patom_{K'}$ and recover $\patom_{K'}$ directly.
Indeed a more general analysis, to be found in \cite{pali21}, shows that for dictionaries containing several double atoms and 1:1 combinations, meaning after rearranging and resigning the dictionaries $\dico, \pdico$ we have $\patom_k = \atom_k$ for $k> 3L$ as well as
$$\patom_\ell = \patom_{L+\ell} = \atom_\ell \quad\mbox{and} \quad \patom_{2L + \ell} = \frac{\atom_{\ell} + \atom_{2L+ \ell}}{\sqrt{2 + 2\ip{\atom_\ell}{ \atom_{2L+\ell}}}} \quad \mbox{for} \quad \ell = 1\ldots L, $$
the residuals are $1$-sparse in the $L$ complementary 1:1 combinations.
$$\patom^c_{\ell} : = \frac{\atom_{\ell} - \atom_{2L+ \ell}}{\sqrt{2 - 2\ip{\atom_\ell}{ \atom_{2L+\ell}}}} \quad \mbox{for} \quad \ell = 1\ldots L. $$
This suggests to use ITKrM with sparsity level $1$, which reduces to ITKsM or line clustering, on the residuals to directly recover the 1:1 complements as replacement candidates.
Replacing the double atom $\patom_\ell$ by $\patom^c_\ell$, in the next iteration it will be serious competition for $\patom_{2L+ \ell}$ in the thresholding of all signals containing either $\atom_\ell$ or $\atom_{2L+ \ell}$. This iteration will then create a first imbalance of the ratio between $\atom_\ell$ and $\atom_{2L+ \ell}$ within one or both of the estimated atoms, making one the more likely choice for $\atom_\ell$ and the other the more likely choice for $\atom_{2L+ \ell}$ in the subsequent iteration. There the imbalance will be further increased until a few iterations later we finally have $\patom_{\ell} \approx \atom_\ell$ and $\patom_{2L+ \ell} \approx \atom_{2L+ \ell}$ or the other way around. \\
We can also immediately see the advantages this ITKsM/clustering approach provides over other residual based replacement strategies, such as   using the largest residual or using the largest principal components or \cite{rudu12, irofti16}. In the case of noise or outliers, the largest residuals are most likely to be outliers or pure noise, meaning that this strategy effectively corresponds to random replacement. The largest principal components of the residuals on the other hand, will most likely be a linear combination of several 1:1 complementary atoms and as such less serious competition for the original 1:1 combinations during thresholding. Additionally to lower chances of being picked, they will also need more iterations to determine which one will rotate into which place.\\
After learning enough from bad dictionaries to inspire a promising replacement strategy, the next subsection will deal with its practical implementation.

\subsection{Replacement in detail}

Now that we have laid out the basic strategy, it remains to deal with all the details. For instance, if we have used all replacement candidates after one iteration, after the next iteration the replacement candidates might not be mature yet, meaning they might not have converged yet. 
\smallskip

\noindent{\bf Efficient learning of replacement atoms}.\\
To solve this problem, observe that the number of replacement candidates, stored in $\Gamma = (\gamma_1,\ldots \gamma_L)$, will be much smaller than the dictionary size, $L\ll K$. Therefore, we need less training signals per iteration to learn the candidates or equivalently we can update $\Gamma$ more frequently, meaning we renormalise after each batch of $N_\Gamma < N$ signals
and set $\Gamma = \bar \Gamma$. Like this, every augmented iteration of ITKrM will produce $L$ replacement candidates.\smallskip

\noindent{\bf Combining coherent atoms.}\\
The next questions concern the actual replacement procedure. Assume we have fixed a threshold $\mu_{\max}$ for the maximal coherence. If our estimate $\pdico$ contains two atoms whose mutual coherence is above the threshold, $\absip{\patom_k}{\patom_{k'}} > \mu_{\max}$, which atom should we replace? One strategy that has been employed for instance in the context of analysis operator learning, \cite{dowadaplha16}, is to average the two atoms, that is to set $\patom_k^{new} = \patom_k + \signop(\ip{\patom_k}{\patom_{k'}})\patom_{k'}$. The reasoning is that if both atoms are good approximations to the generating atom $\atom_k$ then their average will be an even better approximation. However, if one atom $\patom_k$ is already a very good approximation to the generating atom $\patom_k \approx \atom_k$ while $\patom_{k'}$ is still as far away as indicated by $\mu_{\max}$, that is $\patom_{k'} \approx \mu_{\max} \atom_k + \sqrt{1 - \mu_{\max}^2}z_k$, then the averaged atom will be a worse approximation than $\patom_k$ and it would be preferable to simply keep $\patom_k$. To determine which of two coherent atoms is the better approximation, we note that the better approximation to $\atom_k$ should be more likely to be selected during thresholding. This means that we can simply count how often each atom is contained in the thresholded supports $I_n^t$, $v(k) = \sharp \{n: k \in I_n^t\} $ and in case of two coherent atoms keep the more frequently used one. Based on the value function $v$ we can also employ a weighted merging strategy and set $\patom_k^{new} =v(k) \patom_k + \signop(\ip{\patom_k}{\patom_{k'}}) v(k') \patom_{k'}$. If both atoms are equally good approximations, then their value functions should be similar and the balanced combination will be a better approximation. If one atom is a much better approximation it will be used much more often and the merged atom will correspond to this better atom. \smallskip

\noindent{\bf Selecting a candidate atom.}\\
Having chosen how to combine two coherent atoms, we next need to decide which of our $L$ replacement candidates we are going to use. To keep the dictionary incoherent, we first discard all candidates $\ratom_\ell$, whose maximal coherence with the remaining dictionary atoms is larger than our threshold, that is, $\max_k\absip{\ratom_\ell}{\atom_k}\geq \mu_{\max}$. Note that in a perfectly $S$-sparse setting this is not very likely since the residuals we are summing up contain mainly noise or missing 1:1 complements and therefore add up to noise or the desired 1:1 complements. However it might be a problem if we underestimate the sparsity level in the learning. If we use $\tilde S < S$, the residuals are still at least $S-\tilde S$ sparse in the dictionary, so some of our replacement candidates might be near copies of already recovered atoms in the dictionary.\\
To decide which remaining candidate is likely to be the most valuable, we use a counter similar to the one for the dictionary atoms. However, we have to be more careful here since every residual is added to one candidate. If the residual contains only noise, which happens in most cases, and the candidates are reasonably incoherent to each other, then each candidate is equally likely to have its counter increased. This means that the candidate atom that actually encodes the missing atom (or 1:1 complement) will only be slightly more often used than the other candidates.
So to better distinguish between good and bad candidates, we additionally employ a threshold $\tau$ and set $v_\Gamma(\ell) = \sharp \{n: \ell = i_n,  \absip{\ratom_\ell}{a_n}\geq \tau \|a_n\| \}$. To determine the size of the threshold, observe that for a residual consisting only of Gaussian noise, $a = r$, we have for any $\ratom_\ell$ the bound
\begin{align}
	\P(\absip{\ratom_\ell}{r}\geq \tau \|r\|_2 ) \leq 2\exp\left(-\frac{d\tau^2}{2}\right).
\end{align} 
which for $\tau = \sqrt{2 \log(2K)/d}$ becomes $1/K$. This means that the contribution to $v_\Gamma(\ell)$ from all the pure noise residuals is at best $N/K$. On the other hand, with probability $S/K$, the residual will encode the missing atom or 1:1 complement $a \approx ( \atom_i - \atom_j) \cdot |x_i|/2$. For reasonable sparsity levels, $S \lesssim \frac{d}{4 \log(2K)}$, and signal to noise ratios, the candidate $\ratom_\ell$ closest to the missing atom will be picked and should have inner product of the size $ \absip{\ratom_\ell}{a} \approx |x_i|/2 \approx \frac{1}{2\sqrt{S}} \gtrsim \tau \|a\|_2$. This means that for a good candidate the value function will be closer to $NS/K$.\\
The threshold should also help in the earlier mentioned case of underestimating the sparsity level. There one could imagine the candidates to be poolings of already recovered atoms, that is, $\ratom_\ell \approx \sum_{j\in J_\ell} \pm\atom_j/\sqrt{|J_\ell|} $, which are sufficiently incoherent to the dictionary atoms to pass the coherence test. If the residuals are homogenously $S-\tilde S$ sparse in the original dictionary, the candidate atom $\ratom_\ell$ will be picked if $\atom_j$ approximates the residual best for a $j \in J_\ell$. If additionally the sets $J_\ell$ are disjoint, atoms corresponding to a bigger atom pool are (up to a degree) more likely to be chosen than those corresponding to a smaller pool. The threshold helps favour candidates associated to small pools, which have bigger inner products, since $|\ip{\ratom_\ell}{a_n}| \approx 1/\sqrt{S |J_\ell|}$. This is desirable since the candidate closest to the missing atom will correspond to a smaller pool. After all, a candidate containing in its pool the missing atom (1:1 complement) will be soon distorted towards this atom since the sparse residual coefficient of the missing atom will be on average larger than those of the other atoms, thus reducing the effective size of the pool. \smallskip

\noindent{\bf Dealing with unused atoms.} \\
Before implementing our new replacement strategy, let us address another less frequently activated safeguard included in most dictionary learning algorithms: the handling of dictionary atoms that are never selected and therefore have a zero update. As in the case of coherent atoms the standard procedure is replacement of such an atom with a random redraw, which however comes with the problems discussed above. Fortunately our replacement candidates again provide an efficient alternative. If an atom has never been updated, or more generally, if the norm of the new estimator is too small, we simply do not update this atom but set the associated value function to zero. After replacing all coherent atoms we then proceed to replace these unused atoms.\\
The combination of all the above considerations leads to the augmented ITKrM algorithm, which is summarised in Algorithm~\ref{algo:itkrmplus} while the actual procedure for replacing coherent atoms is described in Algorithm~\ref{algo:replacement}, both to be found in Appendix~\ref{sec:app_code}. With these details fixed, the next step is to see how much the invested effort will improve dictionary recovery.

\subsection{Numerical Simulations}\label{sec:exp_replace}

In this subsection we will verify that replacing coherent atoms improves dictionary recovery and test whether our strategy improves over random replacement. Our main setup is the following:\\
{\bf Generating dictionary:} As generating dictionary $\dico$ we use a dictionary of size $K=192$ in $\R^{d}$ with $d=128$, where the atoms are drawn i.i.d. from the unit sphere.\\
{\bf (Sparse) training signals:} We generate $S$-sparse training signals according to our signal model in \eqref{noisymodel2} as 
\begin{align}
	y=\frac{ \dico x_{c, p,\sigma} +\noise}{\sqrt{1+\|\noise \|_2^2}}.
\end{align}
For every signal a new sequence $c$ is generated by drawing a decay factor $q$ uniformly at random in $[0.9,1]$ and setting $c_i = c_q q^{i-1}$ for $i\leq S$ and $0$ else, where $c_q:=\frac{1-q}{1-q^S}$ so that $\|c\|_2 =1$. The noise is centered Gaussian noise with variance $\rho^2 = (16d)^{-1}$, leading to a signal to noise ratio of $\SNR=16$. We will consider two types of training signals. The first type consists of 6-sparse signals with 5\% outliers, that is, we randomly select 5\% of the sparse signals and replace them with pure Gaussian noise of variance $1/d^2$. The second type
consists of 25\% 4-sparse signals, 50\% 6-sparse signals and 25\% 8-sparse signals, where again 5\% are replaced with pure Gaussian noise. In each iteration of ITKrM we use a fresh batch of $N=120 000$ training signals. Unless specified otherwise, the sparsity level given to the algorithm is $S_e = 6$.\\
{\bf Replacement candidates:} During every iteration of ITKrM we learn $L= \round{\log d} = 5$ replacement candidates using $m=\round{\log d} = 5$ iterations each with $N_\Gamma = \floor{N/m}$ signals. \\
{\bf Initialisations:} The dictionary $\pdico$ containing $K$ atoms as well as the replacement candidates are initialised by drawing vectors i.i.d. from the unit sphere. In case of random replacement we use the initialisations of the replacement candidates. All our results are averaged over 20 different initialisations.\\
{\bf Replacement thresholds:} We will compare the dictionary recovery for various coherence thresholds $\mu_{\max} \in \{0.5,0.7,0.9\}$, and all three combination strategies, adding, deleting and merging. We also employ an additional safeguard and replace atoms, which have not been used at all or which have energy smaller than 0.001 before normalisation, if after replacement of coherent atoms we have candidate atoms left. \\
{\bf Recovery threshold:} We use the convention that a dictionary atom $\atom_k$ is recovered if $\max_j\absip{\atom_k}{\patom_j}\geq 0.99$.\\
The results of our first experiment\footnote{As already mentioned, all experiments can be reproduced using the matlab toolbox available at \toolboxlink}, which explores the efficiency of replacement using our candidate strategy in comparison to random or no replacement on 6-sparse signals as described above, are depicted in Figure~\ref{fig:replace_mu}.
\begin{figure}
	\begin{tabular}{ccc}
		\includegraphics[width=0.35\textwidth]{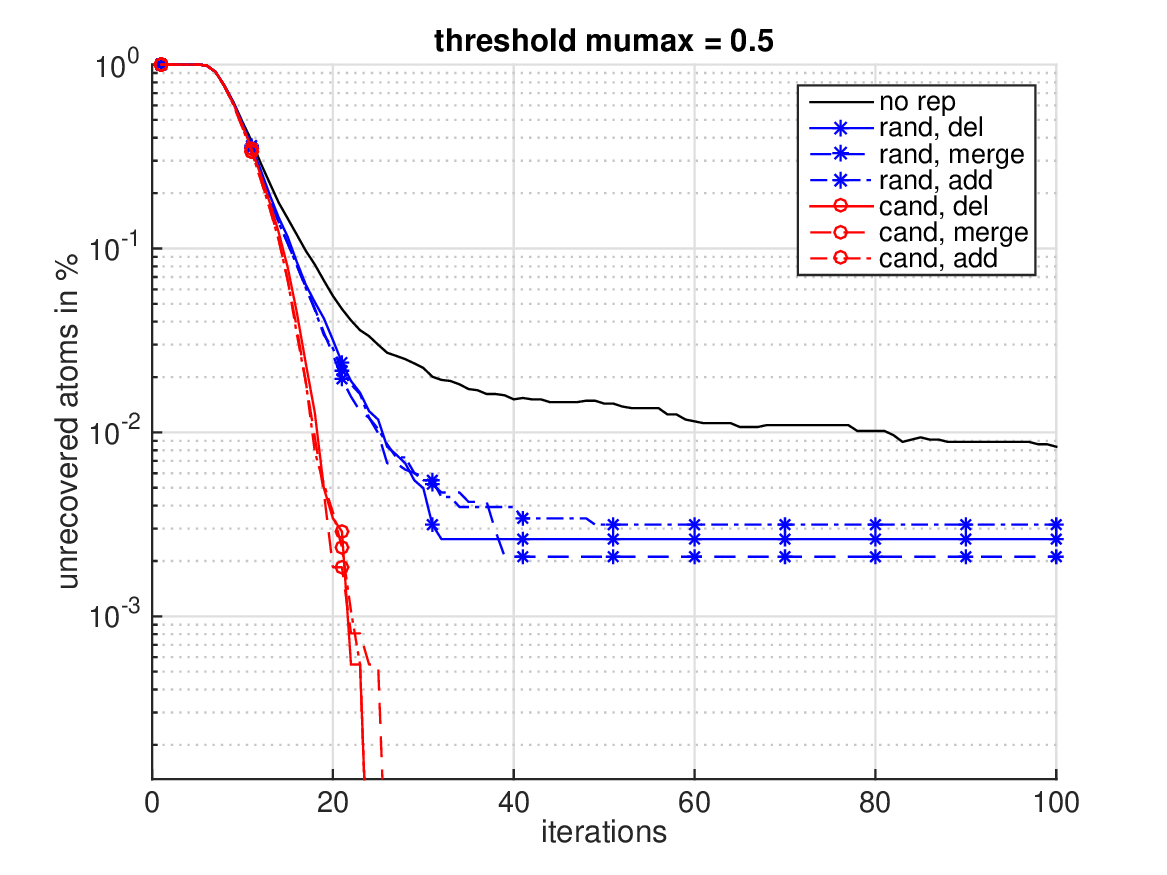} &\hspace{-8mm}\includegraphics[width=0.35\textwidth]{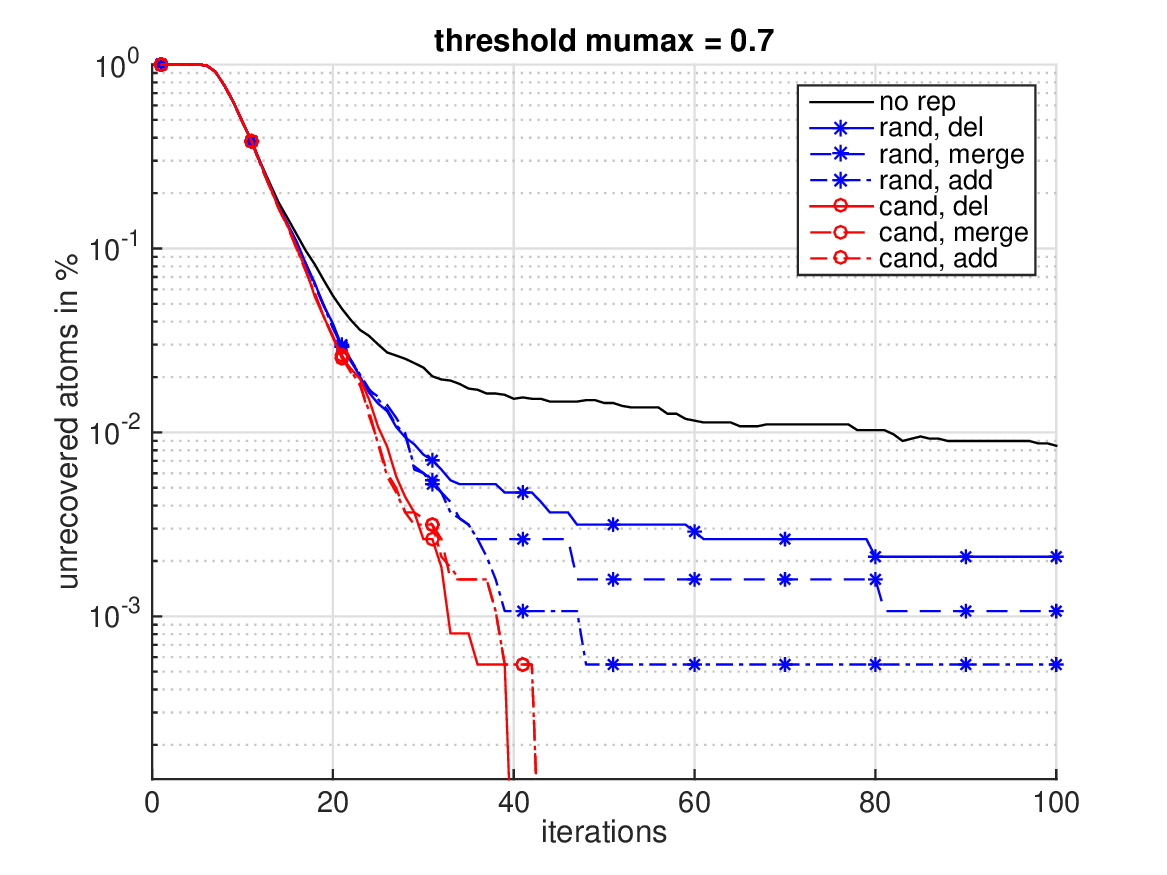}&\hspace{-8mm}\includegraphics[width=0.35\textwidth]{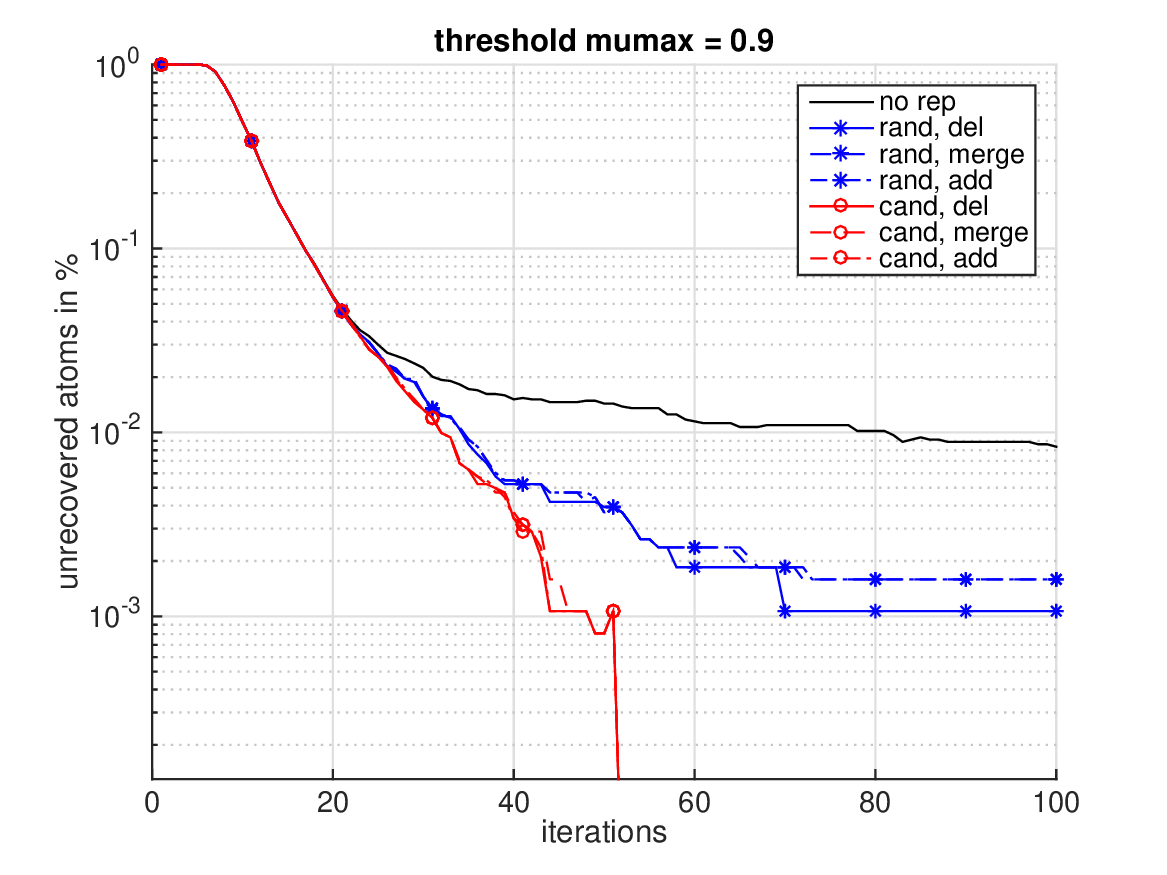}\\
	\end{tabular}
	\caption{Recovery rates of ITKrM without replacement, random and candidate replacement for various coherence thresholds $\mu_{\max}$ and atom combination strategies. \label{fig:replace_mu}}
\end{figure}
We can see that for all three considered coherence thresholds $\mu_{\max} \in \{0.5,0.7,0.9\}$, our replacement strategy improves over random or no replacement. So while after 100 iterations ITKrM without replacement misses about 1\% of the atoms and with random replacement about $0.1\%$, it always finds the full dictionary after at worst 55 iterations using the candidate atoms. Contrary to random replacement the candidate based strategy also does not seem sensitive to the combination method. Another observation is that candidate replacement leads to faster recovery the lower the coherence threshold is, while the average performance for random replacement is slightly better for the higher thresholds. This is connected to the average number of replaced atoms in each run, which is around $16$ for $\mu_{\max}= 0.5$, around $3.8$ for $\mu_{\max}= 0.7$ and around $0.8$ for $\mu_{\max} = 0.9$, since for the candidate replacement there is no risk of replacing a coherent atom that might still change course and converge to a missing generating atom with something useless. For the sake of completeness, we also mention that in none of the trials replacement of unused atoms is ever activated.\\
In our second experiment we explore the performance of candidate replacement for the more interesting (realistic) type of signals with varying sparsity levels. Since the signals can be considered $4$, $6$ or $8$ sparse we compare the performance of ITKrM using all three possibilities, $S_e \in \{4,6,8\}$ and a fixed replacement threshold $\mu_{\max}= 0.7$. 
\begin{figure}[b]
	\begin{tabular}{ccc}
		\includegraphics[width=0.35\textwidth]{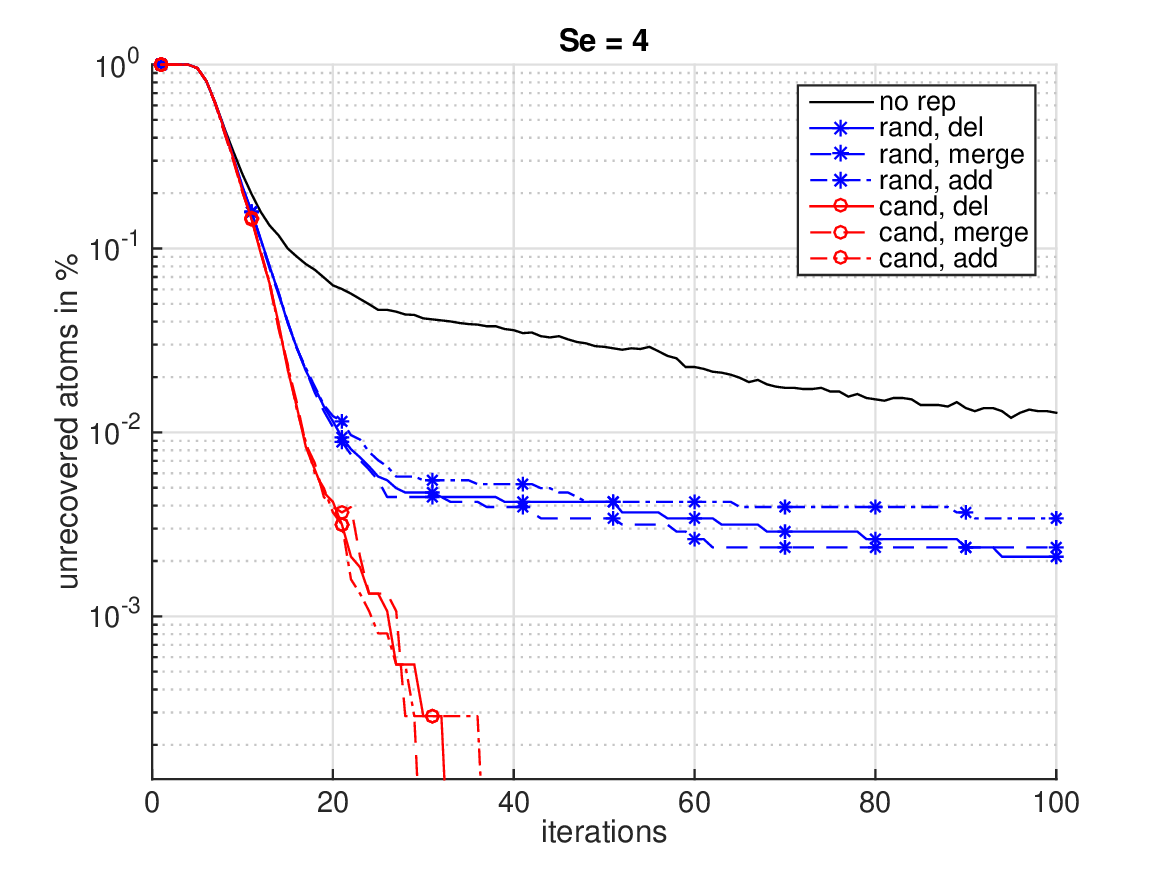} &\hspace{-8mm}\includegraphics[width=0.35\textwidth]{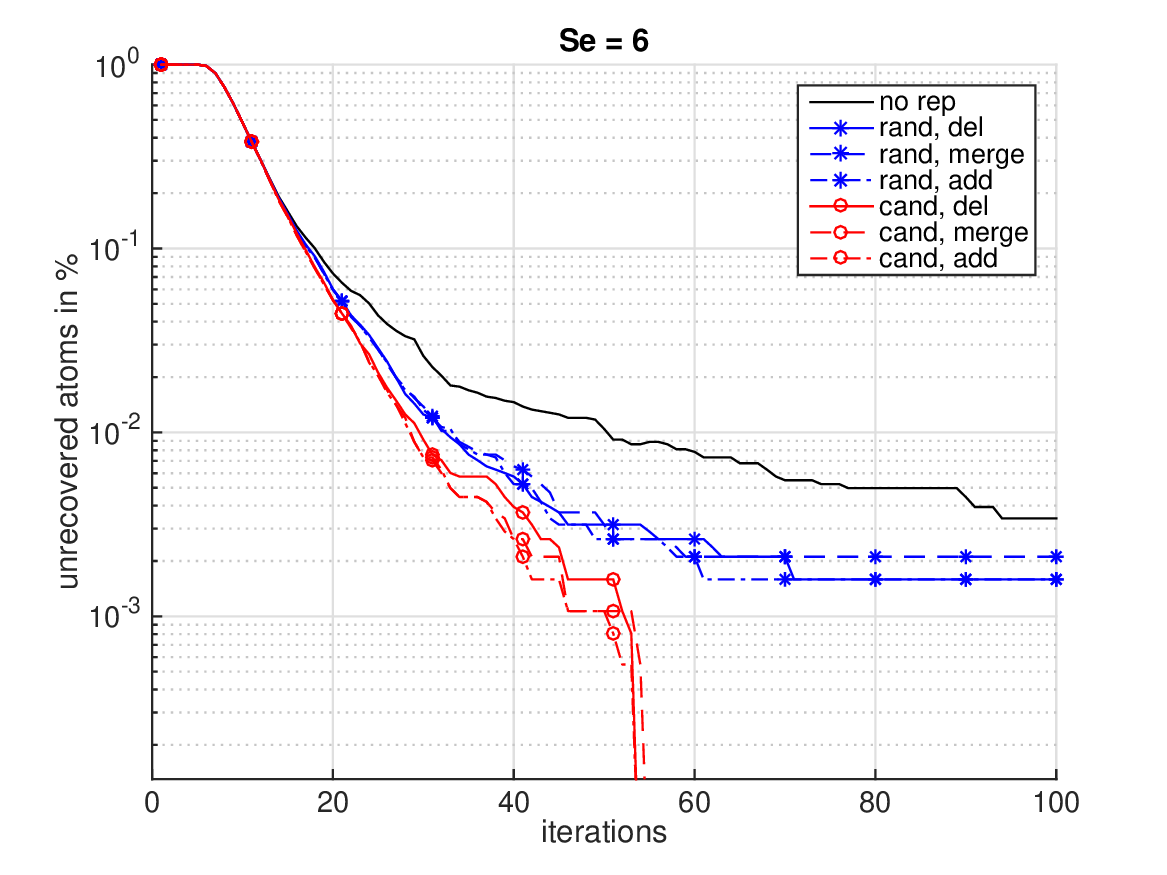}&\hspace{-8mm}\includegraphics[width=0.35\textwidth]{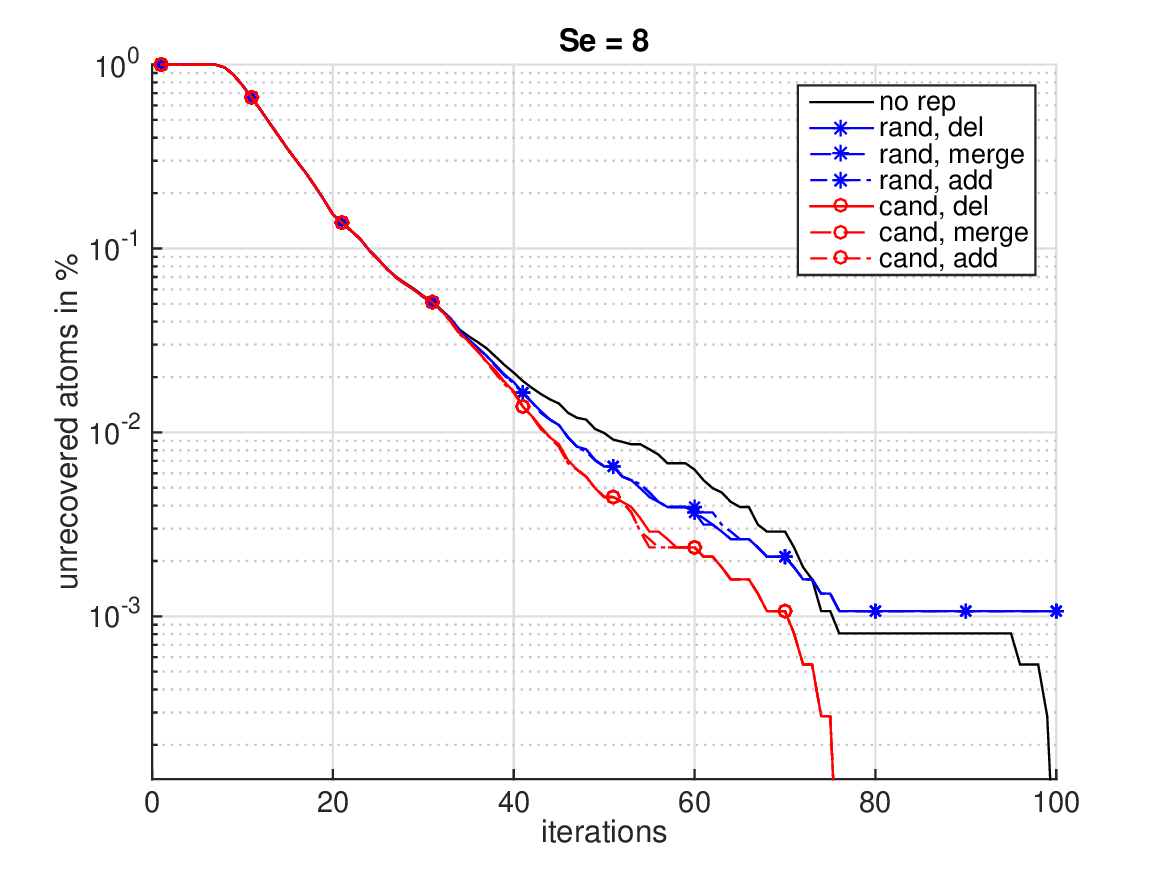}\\
	\end{tabular}
	\caption{Recovery rates of ITKrM without replacement, random and candidate replacement for various input sparsity levels $S_e$ and atom combination strategies, with coherence threshold $\mu_{\max}=0.7$.\label{fig:replace_Se}}
\end{figure}
The results are shown in Figure~\ref{fig:replace_Se}. As before, candidate replacement outperforms random or no replacement and leads to 100\% recovery in all cases. Comparing the speed of convergence we see that it is higher the lower the sparsity level is, so for $S_e =4$ we get 100\% recovery after about 30 iterations, for $S_e = 6$ after about 65 iterations and for $S_e = 8$ after 75 iterations. This would suggest that for the best performance we should always pick a lower than average sparsity level. However, the speed of convergence for $S_e = 4$ comes at the price of precision, as can be seen in the small table below, which lists both the average distance $d(\pdico,\dico)$ of the recovered dictionaries from the generating dictionary after 100 iterations as well as the mean atom distances $d_1(\pdico,\dico): = \frac{1}{K}\sum_k\| \atom_k-\patom_k\|_2$.
\begin{table}[h]
	\centering
	\begin{tabular}{r|ccc}
		&  $S_e=4$  &  $S_e=6$ &    $S_e=8$   \\
		\hline
		$d(\pdico,\dico)$  & 0.0392 & 0.0312 &  0.0304\\
		$d_1(\pdico,\dico)$ & 0.0322  & 0.0256 & 0.0250
	\end{tabular}
\end{table}
For both distances there is an increase in precision, going from $S_e =4$ to $S_e = 6$, but hardly any improvement by going from $S_e = 6$ to $S_e = 8$. This suggests to choose the average or a slightly higher than average sparsity level. Alternatively, to get the best of both worlds, one should start with a smaller sparsity level and then slowly increase to the average sparsity level. Unfortunately this approach relies on the knowledge of the average sparsity level, which in practice is unknown. Considering that also the size of the dictionary is unknown this can be considered a minor problem. After all, if we underestimate the dictionary size, this will limit the final precision more severely. Assume for instance that we set $K-1$ instead of $K$. In this case recovering a dictionary $\pdico$ with $K-2$ generating atoms plus one 1:1 combination of two generating atoms leading to $d(\pdico,\dico)\gtrsim \frac{1}{2}$ is actually the best we can hope for. \\
Therefore, in the next section we will use our candidate atoms to make the big step towards adaptive selection of both sparsity level and dictionary size.

\section{Adaptive dictionary learning}\label{sec:adaptive}

We first investigate how to adaptively choose the sparsity level for a dictionary of fixed size.

\subsection{Adapting the sparsity level}\label{sec:adaptS}

In the numerical simulations of the last section we have seen that the sparsity level $S$ given as parameter to the ITKrM algorithm influences both the convergence speed and the final precision of the learned dictionary. \\
When underestimating the sparsity level, meaning providing $S_e < S$ instead of $S$, the algorithm tends to recover the generating dictionary in less iterations than with the true sparsity level. Note also that the computational complexity of an iteration increases with $S_e$, so a smaller sparsity level leads to faster convergence not only in terms of iterations but also reduces the computation time per iteration. The advantage of overestimating the sparsity level, $S_e > S$ on the other hand, is the potentially higher precision, so the final error between the recovered and the generating dictionary (atoms), can be smaller than for the true sparsity level $S$.
Intuitively this is due to the fact that for $S_e > S$, thresholding with the generating dictionary is more likely to recover the correct support, in the sense that $I \subset I^t$. For a clean signal, $y = \dico_I x_I$ this means that the residual is zero, so that the estimate of every atom in $I^t$, even if not in $I$, is simply reinforced by itself $\ip{\atom_i}{y} \atom_i$. However, in a noisy situation, $y = \dico_I x_I + r$, where the residual has the shape $a=Q_{I^t} r$ the estimate of the additional atom $i \in I^t /I$ is not only reinforced but also disturbed by adding noise in form of the residual once more than necessary. Depending on the size of the noise and the inner product this might not always be beneficial to the final estimate. Indeed, we have seen that for large $S$, where the smallest coefficients in the support are already quite small, overestimating the support does not improve the final precision. \\
To further see that both under- and overestimating the sparsity level comes with risks, assume that we allow $S+1$ instead of the true sparsity level $S$ for perfectly sparse, clean signals. Then any dictionary, derived from the generating dictionary by replacing a pair of atoms $(\atom_i, \atom_j)$ by $(\tilde \atom_i, \tilde \atom_j) = A (\atom_i, \atom_j)$ for an invertible (well conditioned) matrix $A$, will provide perfectly $S+1$-sparse representations to the signals and be a fixed point of ITKrM. Providing $S-1$ instead of $S$ can have even more dire consequences since we can replace any generating atom with a random vector and again have a fixed point of ITKrM. If the original dictionary is an orthonormal basis and the sparse coefficients have equal size in absolute value any such disturbed estimator even gives the same approximation quality. However, in more realistic scenarios, where we have coherence, noise or imbalanced coefficients and therefore the missing atom has the same probability as the others to be among the $S-1$ atoms most contributing to a signal, the generating dictionary should still provide the smallest average approximation error.
Indeed, whenever we have coherence, noise or imbalanced coefficients the signals can be interpreted as being 1-sparse (with enormous error and miniscule gap $c(1)/c(2)$) in the generating dictionary, so learning with $S_e=1$ should lead to a reasonable first estimate of most atoms. Of course if the signals are not actually 1-sparse this estimate will be somewhere between rough, for small $S$, and unrecognisable, for larger $S$, and the question is how to decide whether we should increase $S_e$. If we already had the generating dictionary, the simplest way would be to look at the residuals and see how much we can decrease their energy by adding another atom to the support. A lower bound for the decrease of a residual $a$ can be simply estimated by calculating $\max_k (\ip{\atom_k}{a})^2$.
\\
If we have the correct sparsity level and thresholding recovers the correct support $I^t = I$, the residual consists only of noise, $a = Q(\dico_I)(\dico_I x_I + r) = Q(\dico_I) r \approx r $. 
For a Gaussian noise vector $r$ and a given threshold $\theta \cdot \|r\|_2$, we now estimate how many of the remaining $K-S$ atoms can be expected to have inner products larger than $\theta \cdot \|r\|_2$ as
\begin{align}
	\E\left(\sharp\{ k: |\ip{r}{\atom_k}|^2 > \theta^2 \cdot \|r\|^2_2 \}\right) = \sum_k \P\left( |\ip{r}{\atom_k}|^2 > \theta^2 \cdot \|r\|^2_2 \right)< 2(K-S) e^{-\frac{d \theta^2}{2}}.
\end{align}
In particular, setting $\theta = \theta_K: = \sqrt{2 \log(4K)/d}$ the expectation above is smaller than $\frac{1}{2}$. This means that if we take the empirical estimator of the expectation above, using the approximation $r_n \approx a_n$, we should get 
\begin{align}
	\frac{1}{N} \sum_n \sharp \{ k: |\ip{a_n}{\atom_k}|^2 >  \theta_K^2 \cdot \|a_n\|^2_2\} \lesssim \frac{1}{2},
\end{align}
which rounds to zero indicating that we have the correct sparsity level.\\
Conversely, if we underestimate the correct sparsity level, $S_e = S - m$ for $m>0$, then thresholding can necessarily only recover part of the correct support, $I^t \subset I$. Denote the set of missing atoms by $I^m = I/I^t$. The residual has the shape 
\begin{align*}
	a = Q(\dico_{I^t})(\dico_I x_I + r) = Q(\dico_{I^t})(\dico_{I^m} x_{I^m}+ r )  \approx \dico_{I^m} x_{I^m} + r
\end{align*} 
For all missing atoms $i \in I^m$ the squared inner products are approximately 
\begin{align*}
	|\ip{a}{\atom_i}|^2  \approx (x_i + \ip{r}{\atom_i})^2.
\end{align*} 
Assuming well-balanced coefficients, where $|x_i| \approx 1/\sqrt{S}$ and therefore $\| \dico_{I^m} x_{I^m}\|_2^2 \approx m/S$, a sparsity level $S\lesssim \frac{d}{2 \log(4K)}$ and reasonable noiselevels, this means that with probability at least $\frac{1}{2}$ we have for all $i\in I^m$
\begin{align*}
	|\ip{a}{\atom_i}|^2  \gtrsim  |x_i|^2 \gtrsim \frac{1}{2m} (\| \dico_{I^m} x_{I^m}\|_2^2 + \|r\|_2^2) \gtrsim \theta_K^2 \|a\|_2^2,
\end{align*}
and in consequence
\begin{align}
	\frac{1}{N} \sum_n \sharp \{ k: |\ip{a_n}{\atom_k}|^2 >  \theta_K^2 \cdot \|a_n\|^2_2\} \gtrsim \frac{m}{2}.
\end{align}
This rounds to at least 1, indicating that we should increase the sparsity level.\\
\\
Based on the two estimates above and starting with sparsity level $S_e=1$ we should now be able to arrive at the correct sparsity level $S$. Unfortunately, the indicated update rule for the sparsity level is too simplistic in practice as it relies on thresholding always finding the correct support given the correct sparsity level. Assume that $S_e = S$ but thresholding fails to recover for instance one atom, $I^t = I_{i\leftrightarrow j}$. Then we still have $a = Q(\dico_{I^t})(x_i \atom_i + r) \approx x_i \atom_i + r$ and $|\ip{\atom_i}{a}|^2 \gtrsim \theta_K^2 \|a\|_2$. If thresholding constantly misses one atom in the support, for instance because the current dictionary estimate is quite coherent, $\mu \gg 1/\sqrt{d}$, or not yet very accurate, this will lead to an increase $S_e=S+1$. However, as we have discussed above, while increasing the sparsity level increases the chances for full recovery by thresholding, it also increases the atom estimation error which decreases the chances for full recovery. Depending on which effect dominates, this could lead to a vicious circle of increasing the sparsity level, which decreases the accuracy leading to more failure of thresholding and increasing the sparsity level. In order to avoid this risk we should take into account that thresholding might fail to recover the full support and be able to identify such failure. Further, we should be prepared to also decrease the sparsity level. \\
The key to these three goals is to also look at the coefficients of the signal approximation. Assume that we are given the correct sparsity level $S_e =S$ but recovered $I^t = I_{i\leftrightarrow j}$. Defining $I_{i\rightarrow }= I\setminus\{i\}$, the corresponding coefficients $\tilde x_{I^t} $ have the shape,
\begin{align} 
	\tilde x_{I^t} = \dico_{I^t}^\dagger( \dico_I x_I + r )& = \dico_{I^t}^\dagger( \dico_{I_{i\rightarrow }} x_{I_{i\rightarrow }} + \atom_i x_i + r) \notag \\
	& =  (x_{I_{i\rightarrow}}, 0) + (\dico_{I^t}^\star\dico_{I^t})^{-1} \dico_{I^t}^\star (\atom_i x_i + r),
\end{align}
meaning $|\tilde x_{I^t}(j)|^2 \leq( \mu^2 |x_i|^2 + |\ip{\atom_j}{r}|^2)/(1-\mu S)^2$ or even $|\tilde x_{I^t}(j)|^2 \lesssim \mu^2 |x_i|^2 + |\ip{\atom_j}{r}|^2$. Since the residual is again approximately $a \approx \atom_i x_i + r$, this means that for incoherent dictionaries the coefficient of the wrongly chosen atom is likely to be below the threshold $\theta_K^2 \|a\|_2$, while the one of the missing atom will be above the threshold, so we are likely to keep the sparsity level the same.\\
Similarly if we overestimate the sparsity level $S_e = S+1$ and recover an extra atom $I^t = I_{\leftarrow j}: = I \cup \{j\}$, we have $a = Q(\dico_{I^t})r \approx r$ while the coefficient of the extra atom will be of size $|\tilde x_{I^t}(j)|^2 \approx |\ip{\atom_j}{r}|^2 < \theta_K^2 \| a \|^2_2$. 
All in all our estimates suggest that we get a more stable estimate of the sparsity level by averaging 
the number of coefficients $\tilde x_{I^t} =  \dico_{I^t}^\dagger y$ and residual inner products $(\ip{\atom_i}{a})_{i\notin I^t}$ that have squared value larger than $\theta_K^2$ times the residual energy. 
However, the last detail we need to include in our considerations is the reason for thresholding failing to recover the full support given the correct sparsity level in first place. Assume for instance, that the signal does not contain noise, $y = \dico_I x_I$ but that the sparse coefficients vary quite a lot in size. We know (from Appendix~\ref{sec:app_tech} or \citep{bennett62}) that in case of i.i.d. random coefficient signs, $\P(\signop(x_i)=1)=1/2$, the inner products of the atoms inside resp. outside the support concentrate around,
\begin{align*}
	i \in I &\qquad |\ip{\atom_i}{\dico_I x_I}| \approx |x_i| \pm \big({\textstyle \sum_{k\neq i}} x_k^2 |\ip{\atom_i}{\atom_k}|^2\big)^{1/2} \approx |x_i| \pm \mu \| y\|_2  \\
	i\notin I &\qquad |\ip{\atom_i}{\dico_I x_I}| \approx \left({\textstyle \sum_k} x_k^2 |\ip{\atom_i}{\atom_k}|^2 \right)^{1/2} \approx \mu \| y\|_2.
\end{align*}
This means that thresholding will only recover the atoms corresponding to the $S_r$-largest coefficients for $S_r< S$, that is, $I_r =\{i \in I : |x_i| \gtrsim \mu \| y\|_2\}$. 
The good news is that these will capture most of the signal energy, $\| P(\dico_{I^t})y \|_2^2 \approx \|\dico_{I_r}x_{I_r} \|^2_2 \approx \| y\|^2_2$, meaning that in some sense the signal is only $S_r$ sparse. It also means that for $\mu^2\approx 1/d$, we can estimate the {\it recoverable} sparsity level of a given signal as the number of squared coefficients/residual inner products that are larger than
\begin{align}
	\frac{1}{d}  \| P(\dico_{I^t})y \|_2^2 + \frac{2 \log(4K)}{d}  \| Q(\dico_{I^t})y \|_2^2.
\end{align}
If $S_n$ is the estimated recoverable sparsity level of signal $y_n$, a good estimate of the overall sparsity level $S$ will be the rounded average sparsity level $\bar S = \lfloor \frac{1}{N} \sum_n S_n \rceil$. The corresponding update rule then is to increase $S_e$ by one if $\bar S > S_e$, keep it the same if $\bar S = S_e$ and decrease it by one if $\bar S < S_e$, formally
\begin{align}
	S_e^{new} = S_e + \signop(\bar S - S_e).
\end{align}
To avoid getting lost between numerical and explorative sections we will postpone an algorithmic summary to the appendix and testing of our adaptive sparsity selection to Subsection~\ref{sec:adap_synth}. Instead we next address the big question how to adaptively select the dictionary size.

\subsection{Adapting the dictionary size}\label{sec:adaptK}

The common denominator of all popular dictionary learning algorithms, from MOD to K-SVD, is that before actually running them one has to choose a dictionary size. This choice might be motivated by a budget, such as being able to store $K$ atoms and $S$ values per signal, or application specific, that is, the expected number of sources in sparse source separation. In applications such as image restoration $K$ (like $S$) is either chosen ad hoc or experimentally with an eye towards computational complexity, and one will usually find $d \leq K \leq 4d$, and $S = \sqrt{d}$. If algorithms include some sort of adaptivity of the dictionary size, this is usually in the form of not updating unused atoms, a rare occurence in noisy situations, and deleting them at the end.
Also this strategy can only help if $K$ was chosen too large but not if it was chosen too small.
\\
Underestimating the size of a dictionary obviously prevents recovery of the generating dictionary. For instance, if we provide $K-1$ instead of $K$ the best we can hope for is a dictionary containing $K-2$ generating atoms and a $1:1$ combination of the two missing atoms. The good news is that if we are using a replacement strategy one of the candidates will encode the $1:1$ complement, similar to the situation discussed in the last section, where we are given the correct dictionary size but had a double atom.\\
Overestimating the dictionary size does not prevent recovering the dictionary per se, but can decrease recovery precision, meaning that a bigger dictionary might not actually provide smaller approximation error. To get an intuition what happens in this case assume that we are given a budget of $K+1$ instead of $K$ atoms and the true sparsity level $S$. The most useful way to spend the extra budget is to add a $1:1$ combination of two atoms, which frequently occur together, meaning $\atom_0 \propto \atom_i + h \atom_j$ for $h = \signop(\ip{\atom_i}{\atom_j})$. The advantage of the augmented dictionary $\pdico = (\atom_0, \dico)$ is that some signals are now $S-1$ sparse. The disadvantage is that $\pdico$ is less stable since the extra atom $\atom_0$ will prevent $\atom_i$ or $\atom_j$ to be selected by thresholding whenever they are contained in the support in a $1:h$ ratio. This disturbs the averaging process and reduces the final accuracy of both $\atom_i$ and $\atom_j$. \\
The good news is that the extra atom $\atom_0$ is actually quite coherent with the dictionary $\absip{\atom_0}{\atom_{i(j)}} \geq 1/\sqrt{2}$, so if we have activated a replacement threshold of $\mu_{\max} \leq 1/\sqrt{2}$, the atom $\atom_0$ will be soon replaced, necessarily with another useless atom. \\
This suggests as strategy for adaptively choosing the dictionary size to decouple our replacement scheme into pruning and adding, which allows to both increase and decrease the dictionary size. We will first have a closer look at pruning.
\\

\noindent{\bf Pruning atoms.}\\
From the replacement strategy we can derive two easy rules for pruning: 1) if two atoms are too coherent, delete the less often used one or merge them, 2) if an atom is not used, delete it. Unfortunately, the second rule is too naive for real world signals, containing among other imperfections noise, which means also purely random atoms are likely to be used at least once by mistake. To see how we need to refine the second rule assume again that our sparse signals are affected by Gaussian noise (of a known level), that is, $y=\dico_I x_I + r$ with $\E(\|r\|^2_2) = \rho^2$ and that our current dictionary estimate has the form $\pdico = (\atom_0, \dico)$, where $\atom_{0}$ is some vector with admissible coherence to $\dico$. 
Whenever $\atom_0$ is selected this means that thresholding has failed. From the last subsection we also know that we have a good chance of identifying the failure of thresholding by looking at the coefficients $\dico_{I^t}^\dagger( \dico_I x_I + r )$. The squared coefficient corresponding to the incorrectly chosen atom $\atom_0$ is likely to be smaller than $\lesssim \| \dico_I x_I \|^2_2 /d + |\ip{\atom_0}{r}|^2$ while the squared coefficient of a correctly chosen atom $i\in I\cap I^t$ will be larger than $ |x_i|^2 +  |\ip{\patom_i}{r}|^2 \gtrsim \| \dico_I x_I \|^2_2/S + |\ip{\atom_i}{r}|^2$ at least half of the time. The size of the inner product of any atom with Gaussian noise can be estimated as
\begin{align}
	\P\left(|\ip{\atom_k}{r}| > \tau \|r\|_2 \right) \leq  2 \exp\left( - \frac{d\tau^2}{2}\right).
\end{align}
Taking again $\| P(\dico_{I^t}) y \|_2$ as estimate for $\| \dico_I x_I \|_2$ and $\|a\|_2 = \| Q(\dico_{I^t}) y\|_2$ as estimate for $\|r\|_2$ we  
can define the refined value function $\tilde v(k)$ as the number of times an atom $\atom_k$ has been selected and the corresponding
coefficient has squared value larger than $\| P(\dico_{I^t}) y \|^2_2/d+ \tau^2 \|a_n\|^2_2$. Based on the bound above we can then estimate that
for $N$ noisy signals the value function of the unnecessary or random atom $\atom_0$ is bounded by $\tilde v (0) \lesssim 2N \exp\left( - \frac{d\tau^2}{2}\right):=M$, leading to a natural criterion for deleting unused atoms. Setting for instance $\tau = \theta_K = \sqrt{2\log(4 K)/d}$ we get $M = N/(2d)$.
Alternatively, we can say that in order to accurately estimate an atom we need $M$ reliable observations and accordingly set the threshold
to $\tau = \sqrt{2\log(2N/M)/d}$. \\
The advantage of a relatively high threshold $\tau \approx \sqrt{2\log (4K)/d}$ is that in low noise scenarios, we can also find atoms that are rarely used. The disadvantage is that for high $\tau$ the quantities $\tilde v(\cdot)$ we have to estimate are relatively small and therefore susceptible to random fluctuations. In other words, the number of training signals $N$ needs to be large enough to have sufficient concentration such that for unnecessary atoms the value function $\tilde v(\cdot)$ is actually smaller than $M$. Another consideration is that at the beginning, when the dictionary estimate is not yet very accurate, also the approximate versions of frequently used atoms will not be above the threshold often enough. This risk is further increased if we also have to estimate the sparsity level. If $S_e$ is still small compared to the true level $S$ we will overestimate the noise, and even perfectly balanced coefficients $1/\sqrt{S}$ will not yet be above the threshold. Therefore, pruning of the dictionary should only start after an embargo period of several iterations to get a good estimate of the sparsity level and most dictionary atoms. \\
In the replacement section we have also seen that after replacing a double atom with the 1:1 complement $\atom_i - \atom_j$ of a 1:1 atom $\atom_i + \atom_j$, it takes a few iterations for the pair $(\atom_i \pm \atom_j)$ to rotate into the correct configuration $(\atom_i, \atom_j)$, where they are recovered most of the time. In the case of decoupled pruning and adding, we run the risk of deleting a missing atom or a $1:1$ complement one iteration after adding it simply because it has not been used often enough. Therefore, every freshly added atom should not be checked for its usefulness until after a similar embargo period of several iterations, which brings us right to the next question when to add an atom.
\\

\noindent{\bf Adding atoms.}\\
To see when we should add a candidate atom to the dictionary, we have a look back at the derivation of the replacement strategy.
There we have seen that the residuals are likely to be either 1-sparse in the missing atoms (or 1:1 complements of the atoms doing the job of
two generating atoms), meaning $a \approx |x_i|/2 (\atom_i - \atom_j)$ or in a more realistic situation $a \approx |x_i|/2 (\atom_i - \atom_j) + r$,
or zero, which again in the case of noise means $a \approx r$. To identify a good candidate atom we observe again that if the residual consists only of (Gaussian) noise we have for any vector/atom $\gamma_k$
\begin{align}
	\P\left(|\ip{\gamma_k}{r}| > \tau_\Gamma \|r\|_2 \right) \leq  2 \exp\left( - \frac{d\tau_\Gamma^2}{2}\right).
\end{align}
If on the other hand the residual consists of a missing complement, the corresponding candidate $\gamma_\ell \approx (\atom_i - \atom_j)/\sqrt{2}$ should have $\absip{a}{\gamma_\ell}\approx |x_i|/\sqrt{2}\gtrsim \tau_\Gamma \|a\|_2$. This means that we can use a similar strategy as for the dictionary atoms to distinguish between useful and useless candidates. In the last candidate iteration, using $N_\Gamma$ residuals, we count for each candidate atom $\gamma_k$ how often it is selected and satisfies $|\ip{\gamma_k}{a}| > \tau_\Gamma \|a\|_2 $. Following the dictionary update and pruning we then add all candidates to the dictionary whose value function is higher than $M_\Gamma = 2N_\Gamma \exp\left( - \frac{d\tau_\Gamma^2}{2}\right)$ and which are incoherent enough to atoms already in the dictionary.  \\
Now, having dealt with all aspects necessary for making ITKrM adaptive, it is time to test whether adaptive dictionary learning actually works.

\subsection{Experiments on synthetic data}\label{sec:adap_synth}

We first test our adaptive dictionary learning algorithm on synthetic data\footnote{Again we want to point all interested in reproducing the experiments to the matlab toolbox available at \toolboxlink}.
The basic setup is the same as in Subsection~\ref{sec:exp_replace}. However, one type of training signals will again consist of 4, 6 and 8-sparse signals in a 1:2:1 ratio with 5\% outliers, while the second type will consist of 8, 10 and 12-sparse signals in a 1:2:1 ratio with 5\% outliers. Additionally, we will consider the following settings.\\
The {\bf minimal number of reliable observations} $M$ for a dictionary atom is set to either $d$, $\round{d\log{d}}$ or $\round{2d\log{d}}$ with corresponding coefficient thresholds $\tau = \sqrt{2\log(2N/M)/d}$. For the candidate atoms the minimal number of reliable observations in the 4th (and last) candidate iteration is always set to $M_\Gamma = d$.\\
The {\bf sparsity level} is adapted after every iteration starting with iteration $m = \round{\log d}= 5$. The initial sparsity level is 1.\\
{\bf Promising candidate atoms} are added to the dictionary after every iteration, starting again in the $m$-th iteration. In the last $3m$ iterations no more candidate atoms are added to the dictionary.\\
{\bf Coherent dictionary atoms} are merged after every iteration, using the threshold $\mu_{\max} = 0.7$. As weights for the merging we use the value function of the atoms from the most recent iteration.\\
{\bf Unused dictionary atoms} are pruned after every iteration starting with iteration $2m$. An atom is considered unused if in the last $m$ iterations the number of reliable observations has always been smaller than $M$.  Candidate atoms, freshly added to the dictionary, can only be deleted because they are unused at least $m$ iterations later. In each iteration at most $\round{d/5}$ unused atoms are deleted, with an additional safeguard for very undercomplete dictionaries ($K_e < d/10$) that at most half of all atoms can be deleted.\\
The {\bf initial dictionary} is chosen to be either of size $K_e = d=128$, $K_e=4d = 512$ or the correct size $K_e=K$, with the atoms drawn i.i.d. from the unit sphere as before. Figure~\ref{fig:adap} shows the results averaged over 10 trials each using a different initial dictionary.\\
\begin{figure}[tbh]
	\centering
	\begin{tabular}{cc}
		$S \in \{4,6,8\}$ & $S \in \{8,10,12\}$\\
		\includegraphics[width=0.4\textwidth]{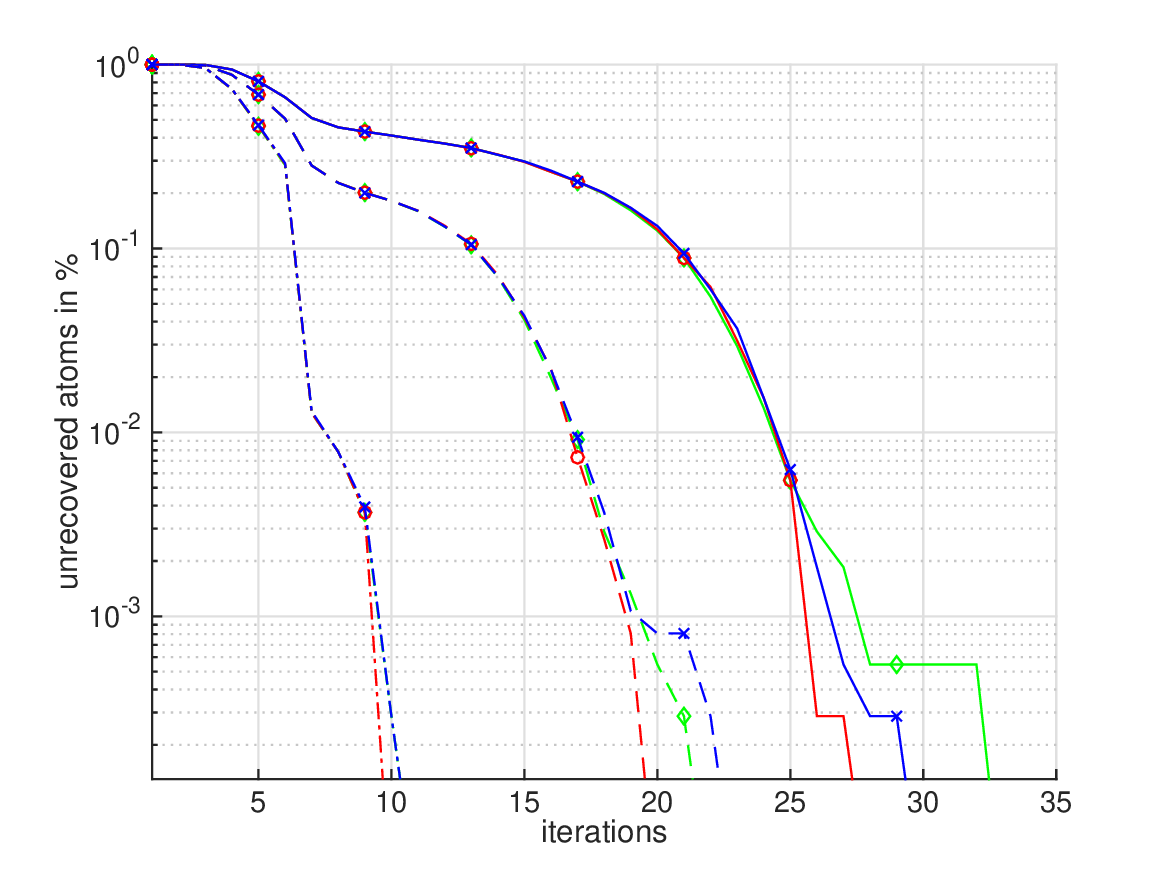} &\includegraphics[width=0.4\textwidth]{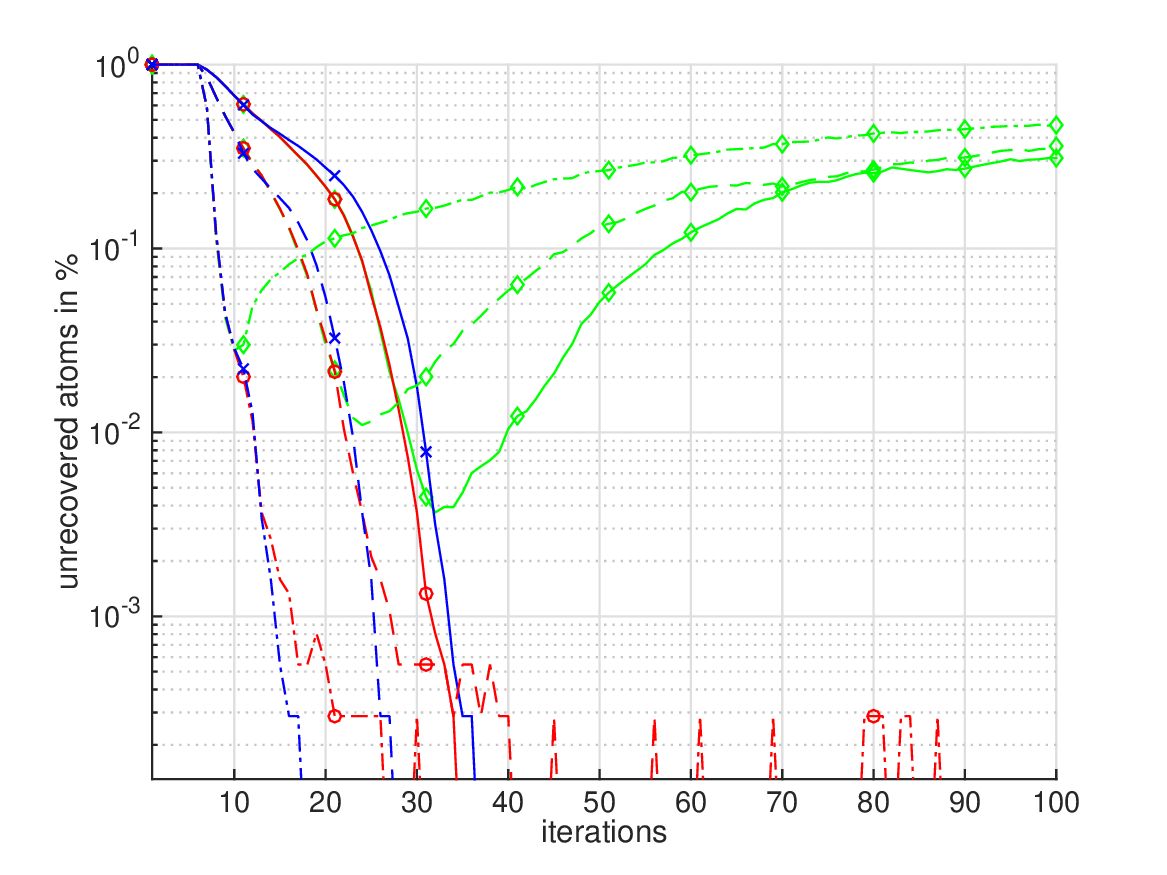}  \\ 
		\includegraphics[width=0.4\textwidth]{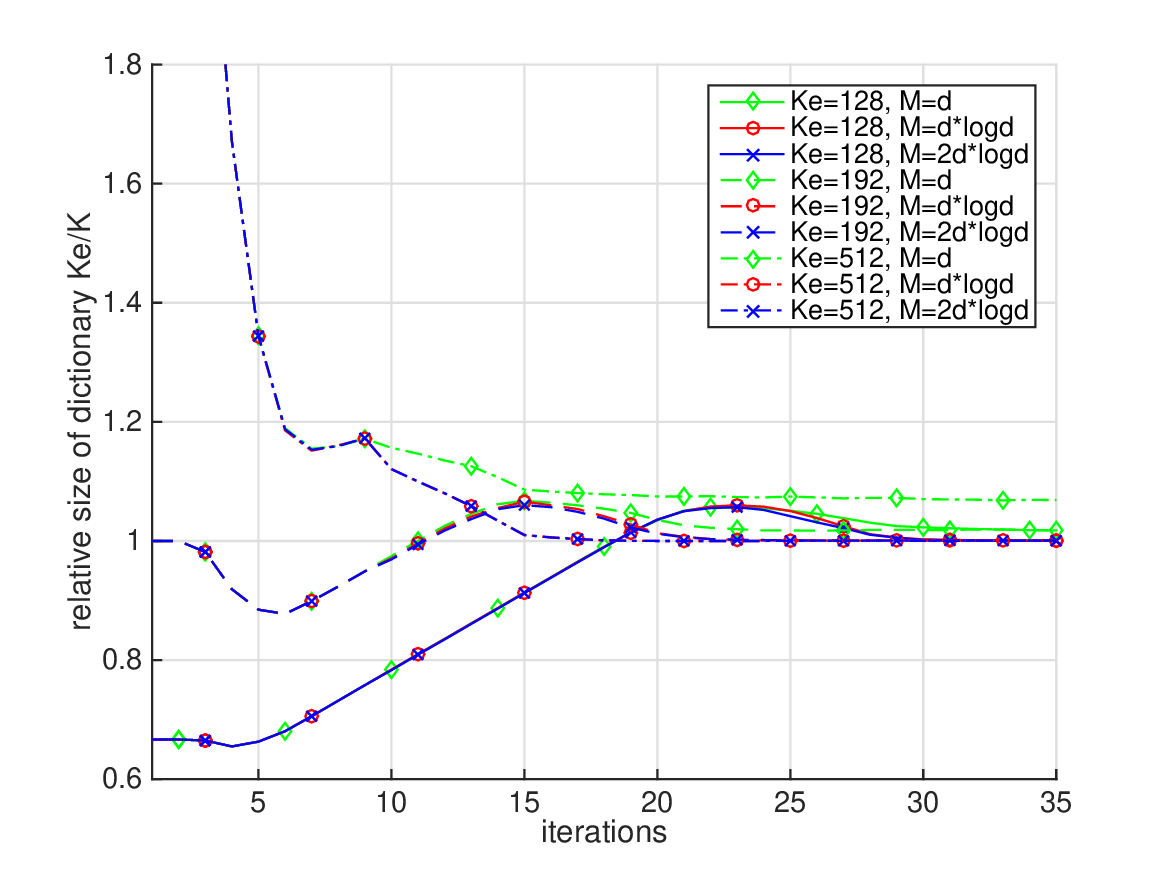}& \includegraphics[width=0.4\textwidth]{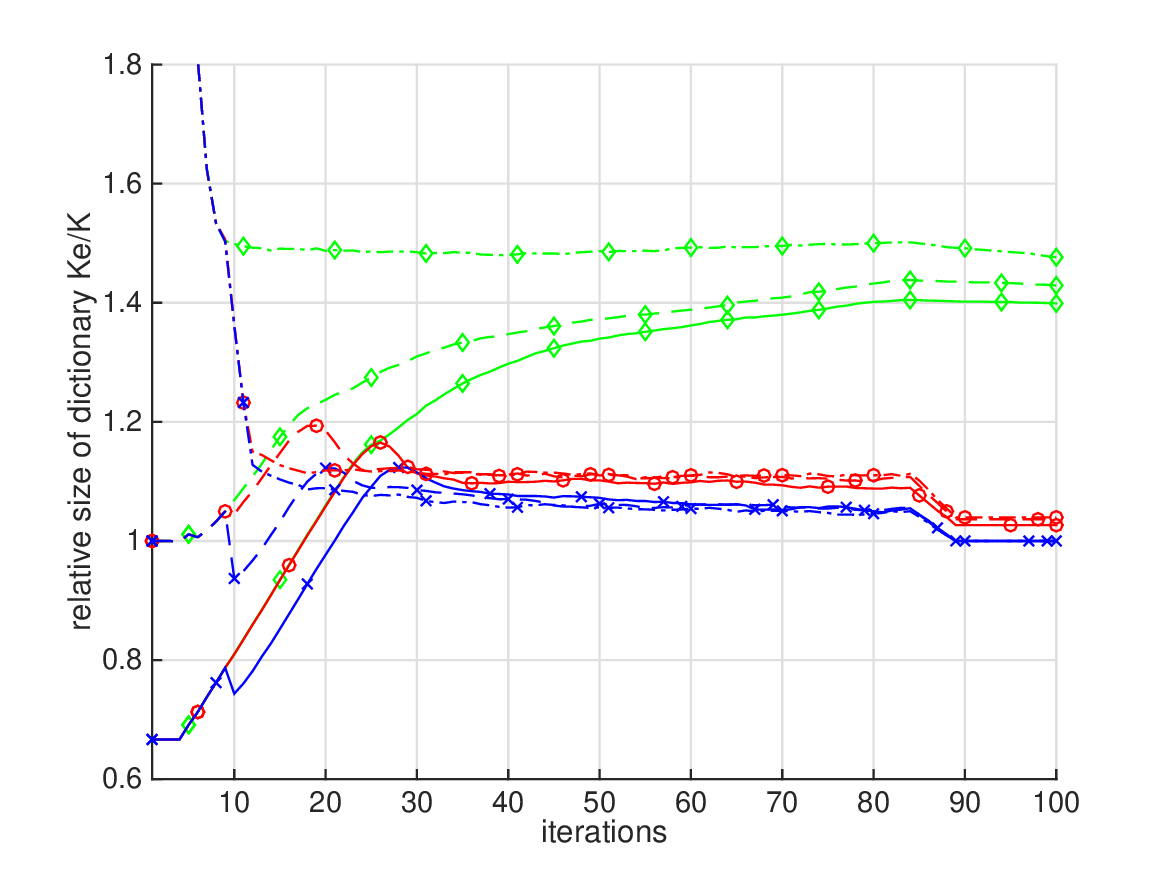}\\ 
	\end{tabular}
	\caption{Average recovery rates (top row) and dictionary sizes (bottom row) for adaptive dictionary learning based on ITKrM on signals with sparsity $S = 4,6,8$  (left column) resp. $S=8,10,12$ (right column) in a 1:2:1 ratio for various initial dictionary sizes $K_e$ and required number of observations per atom $M$.\label{fig:adap}}
\end{figure}
The first observation is that all our effort paid off and that adaptive dictionary learning works.
For the smaller average sparsity level $S=6$, adaptive ITKrM always recovers all atoms of the dictionary and only overshoots
and recovers more atoms for $M=d$. The main difference in recovery speed derives from the size of the initial dictionary, where a larger dictionary size leads to faster recovery. \\
For the more challenging signals with average sparsity level $S=10$, the situation is more diverse. So while the
initial dictionary size mainly influences recovery speed but less the final number of recovered atoms, the cut off threshold $M$ for the minimal number of reliable observations is critical for full recovery. So for $M=d$ adaptive ITKrM never recovers the full dictionary. We can also see that not recovering the full dictionary is strongly correlated with overestimating the dictionary size. Indeed, the higher the overestimation factor for the dictionary size is, the lower is the amount of recovered atoms. For example, for $M=d$, $K_e =512$ the dictionary size is overestimated by a factor $1.5$ and only about half of the dictionary atoms are recovered. To see that the situation is not as bad as 
it seems we have a look at the average sorted atom recovery error. That is, we sort the recovery errors $(d(\atom_k,\pdico))_k$ after 100 iterations in ascending order and average over the number of trials. The resulting curves are depicted in Figure~\ref{fig:sorterr}. 
\begin{figure}[tbh]
	\centering
	\includegraphics[width=0.4\textwidth]{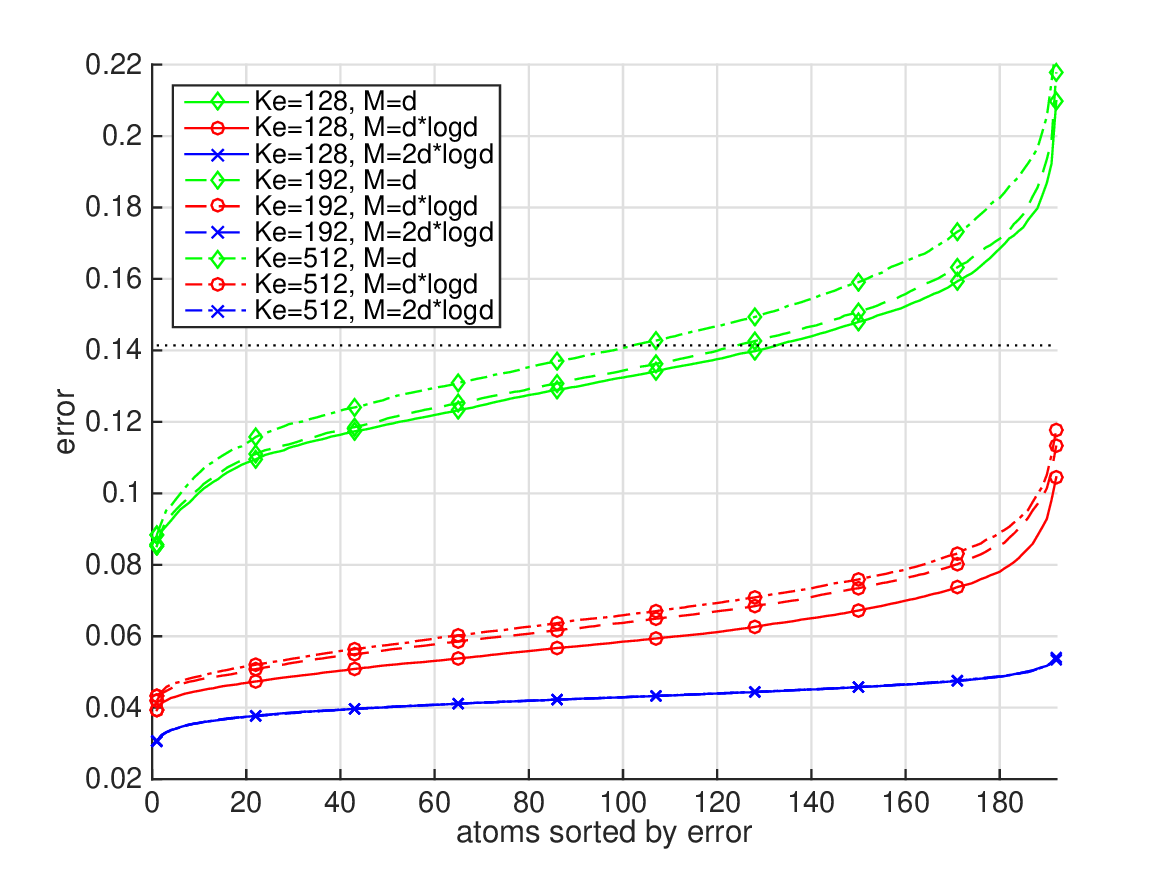}\\ 
	\caption{Average sorted recovery error $(d(\atom_k,\pdico))_k$ after 100 iterations of adaptive ITKrM on signals with sparsity $S=8,10,12$ in a 1:2:1 ratio for various initial dictionary sizes $K_e$ and required number of observations per atom $M$. \label{fig:sorterr}}
\end{figure}
As we can see, overestimating the dictionary size degrades the recovery in a gentle manner. For the unrecovered atoms in case $M=d$, $K_e =512$, the largest inner product with an atom in the recovered dictionary is below the cut-off threshold of $0.99$ which corresponds to an error of size $\approx 0.14$ but for almost all of them it is still above $0.98$ which corresponds to an error of $0.2$. Also the oscillating recovery behaviour for $M=\round{d\log(d)}$ before the final phase, where no more atoms are added, can be explained by the fact that the worst approximated atoms have average best inner product very close to $0.99$. So, depending on the batch of training signals in each iteration, their inner product is below or above $0.99$ and accordingly they count as recovered or not. In general, we can see that the more accurate the estimate of the dictionary size is, the better is the recovery precision of the learned dictionary. This is only to be expected. After all, whenever thresholding picks a superfluous atom instead of a correct atom, the number of observations for the missing atom is reduced and moreover, the residual error added to the correctly identified atoms is increased.\\
The relative stability of these spurious atoms can in turn be explained by the fact that $S=10$ is at the limit of admissible sparsity for a generating dictionary with $\mu(\dico)=0.32$, especially for sparse coefficients with a dynamic range of $0.9^{S-1}\approx 2.58$. In particular,
thresholding is not powerful enough to recover the full support, so the residuals still contain several generating atoms. This promotes candidate atoms that are a sparse pooling of all dictionary atoms. These poolings again have a good chance to be selected in the thresholding and to be above the reliability threshold, thus positively reinforcing the effect. A quick look at the estimated sparsity level as well as the average number of coefficients above the threshold, or in other words, the average number of (probably) correctly identified atoms in the support, denoted by $S_t$, also supports this theory. So for average sparsity level $S=6$ the estimated sparsity level is $S_e = 6 =\round{5.7} $ and the average number of correctly identified atoms is $S_t \approx 5$, regardless of the setting. This is quite close to the average number of correctly identifiable atoms given $S_e=6$, which is $0.95*(0.25 *4 + 0.75 * 6)=5.225$. For average generating sparsity level $S=10$ the table below lists $S_e : S_t$ for all settings. 
\begin{table}[h]
	\centering
	\begin{tabular}{c|ccc}
		&  $d$& $d\log(d)$&  $2d\log(d)$ \\ \hline
		$128$   &  8 : 5.5  &  9 : 7.2  &  9 : 7.2\\
		$192$   &  8 : 5.4  &  9 : 7.0  &  9 : 7.2\\
		$512$   &  7 : 4.8  &  9 : 6.3  &  9 : 7.2\\
	\end{tabular}
\end{table}
We can see that even for the settings where the full dictionary is recovered, the estimated sparsity level is below $10$ and the number of correctly identified atoms lags even more behind. For comparison, for $S_e =9$ the average number of correctly identifiable atoms is $0.95*(0.25 *8 + 0.75 * 9) = 8.3125$. \\
We also want to mention that for signals with average generating sparsity $S = 6$ and $S_e = 6$ we can at best observe each atom $5.225\cdot\frac{N}{K}\approx 3266$ times which is only about 5 times the threshold $d\log(d)$. For the signals with higher sparsity level and $S_e =9$ we can at best observe each atom $8.3125\cdot \frac{N}{K}\approx 5195$ times which is about 8 times the threshold $d\log(d)$, meaning that we are further from the critical limit where we would also remove an exactly recovered generating atom. In general, when choosing the minimal number of observations $M$, one needs to take into account that the number of recoverable atoms is limited by $K_{\max} \leq S_e*N/M$. On the other hand for larger $S/S_e$ the generating coefficients will be smaller, meaning that they will be less likely to be over the threshold $\tau = \sqrt{2\log(2N/M)/d}$ if $M$ is small. This suggests to also adapt $M,\tau$ in each iteration according to the current estimate of the sparsity level.
Another strategy to reduce overshooting effects is to replace thresholding by a different approximation algorithm in the last rounds. The advantage of thresholding over more involved sparse approximation algorithms is its stability with respect to perturbations of the dictionary, see \cite{pali21} for the case of OMP. The disadvantage is that it can only handle small dynamic coefficient ranges. However, we have seen that using thresholding, we can always get a reasonable estimate of the dictionary. Also in order to estimate $S_e$, we already have a good guess which atoms of the threshold support were correct and which atoms outside should have been included. This suggests to remove any atom from the support for which there is a more promising atom outside the support, or in other words, to update the support by thresholding $(\pdico_{I_t}^\dagger y, \pdico_{I_t^c} (\I_d - P(\pdico_I)y)$. Iterating this procedure until the support is stable is known as Hard Thresholding Pursuit (HTP), \cite{fo11}. Using 2 iterations of HTP would not overly increase the computational complexity of adaptive ITKrM but could help to weed out spurious atoms. Also by not keeping the $S_e$ best atoms but only those above the threshold $\tau$ one could deal with varying sparsity levels, which would increase the final precision of the recovered dictionary.\\
Such a strategy might also help in addressing the only case where we have found our adaptive dictionary learning algorithm to fail spectacularly. This is - at first glance surprisingly - the most simple case of exactly 1-sparse signals and an initial dictionary size smaller than the generating size. At second glance it is not that surprising anymore. In case of underestimating the dictionary size $K - K_e = K_m >0$ the best possible dictionary consists of $K-2K_m$ generating atoms and $K_m$ 1:1 combinations of 2 non-orthogonal atoms of the form $\atom_{ij} = (\atom_i + h \atom_j )/\alpha_{ij} $, where $h =\signop{\ip{\atom_i}{\atom_j}} $ and $\alpha_{ij} = \sqrt{2+2\absip{\atom_i}{\atom_j}}$. In such a situation the (non-zero) residuals are again 1-sparse in the 1:1 complements $\tilde \atom_{ij} =(\atom_i - h \atom_j )/ \tilde \alpha_{ij} $, where $\tilde \alpha_{ij} = \sqrt{2-2\absip{\atom_i}{\atom_j}}$, and so the replacement candidates will be the 1:1 complements. However, the problem is that $\tilde \atom_{ij}$ is never picked by thresholding since for both $y \approx \atom_i$ and $y \approx \atom_j$ the inner product with $\atom_{ij}$ is larger,
\begin{align}
	\absip{\atom_{ij}}{\atom_i} = \sqrt{\frac{1+\absip{\atom_i}{\atom_j}}{2}} > \sqrt{\frac{1-\absip{\atom_i}{\atom_j}}{2}} = |\ip{\tilde \atom_{ij}}{\atom_i}|.
\end{align}
Still the inner product of $\tilde \atom_{ij}$ with the residual has magnitude $\approx 1/2 > \tau$ and so would be included in the support in a second iteration of HTP, thus keeping the chance that the pair $(\atom_{ij}, \tilde\atom_{ij})$ rotates into the correct configuration $(\atom_i,\atom_j)$ alive. \\
We will postpone a more in-depth discussion of how to further stabilise and improve adaptive dictionary learning to the discussion in Section~\ref{sec:discussion}. Here we will first check whether adaptive dictionary learning is robust to reality by testing it on image data.

\subsection{Experiments on image data}\label{sec:adap_im}

In this subsection we will learn dictionaries for the images {\it Mandrill} and {\it Peppers}. For those interested in results on larger, more practically relevant datasets, we refer to \cite{pako20}, where our adaptive dictionary learning schemes are used for image reconstruction in accelerated 2D radial cine MRI.\\
The {\bf training signals} are created as follows. Given a $256\times 256$ image, we contaminate it with Gaussian noise of variance $\tilde \nsigma^2 =  \nsigma^2/255$ for $ \nsigma^2 \in \{0, 5,10,15,20\}$. From the noisy image we extract all $8\times 8$ patches (sub-images), vectorise them and remove their mean. In other words, we assume that the constant atom, $\atom_0 \equiv 1/8$, is always contained in the signal, remove its contribution and thus can only learn atoms that are orthogonal to it. \\
The {\bf set-up} for adaptive dictionary learning is the same as for the synthetic data, taking into account that for the number of candidates $L$ and the memory $m$, we have $L = m =\round{\log (d) } = 4$, since the signals have dimension $d=64$. Also based on the lesson learned on the more complicated data set with average sparsity level $S=10$, we only consider as minimal number of observations $M=\round{d\log(d)}$ and $M=2d\log(d)$. The initial dictionary size $K_e$ is either $8, 64$ or $256$ and in each iteration we use all available signals, $N=62001$. All results are averaged over $10$ trials, each using a different initial dictionary and - where applicable - a different noise-pattern.
\begin{figure}[tbh]
	\centering
	\begin{tabular}{ccc}
		\includegraphics[height =4.5cm]{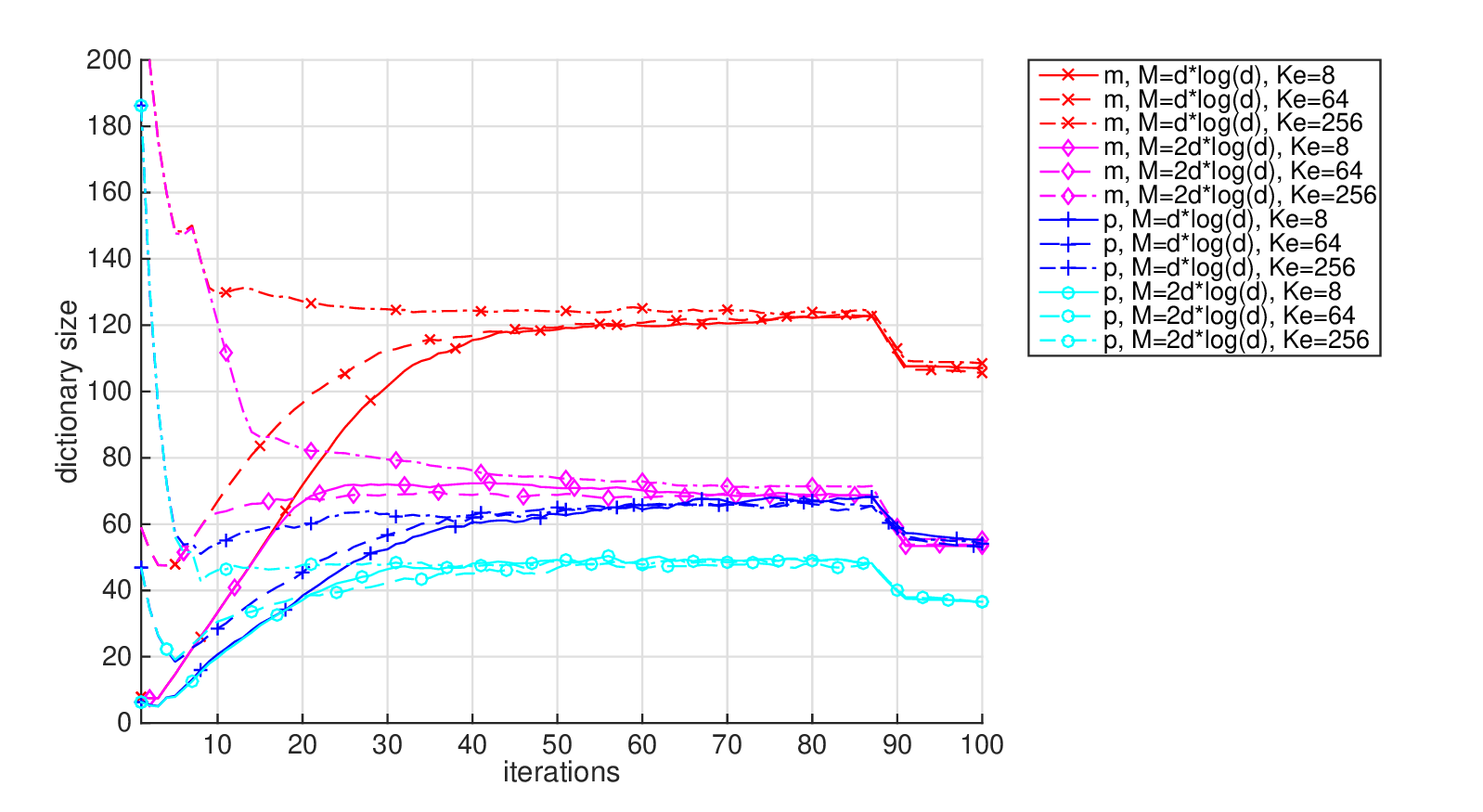} &\includegraphics[height=4.5cm]{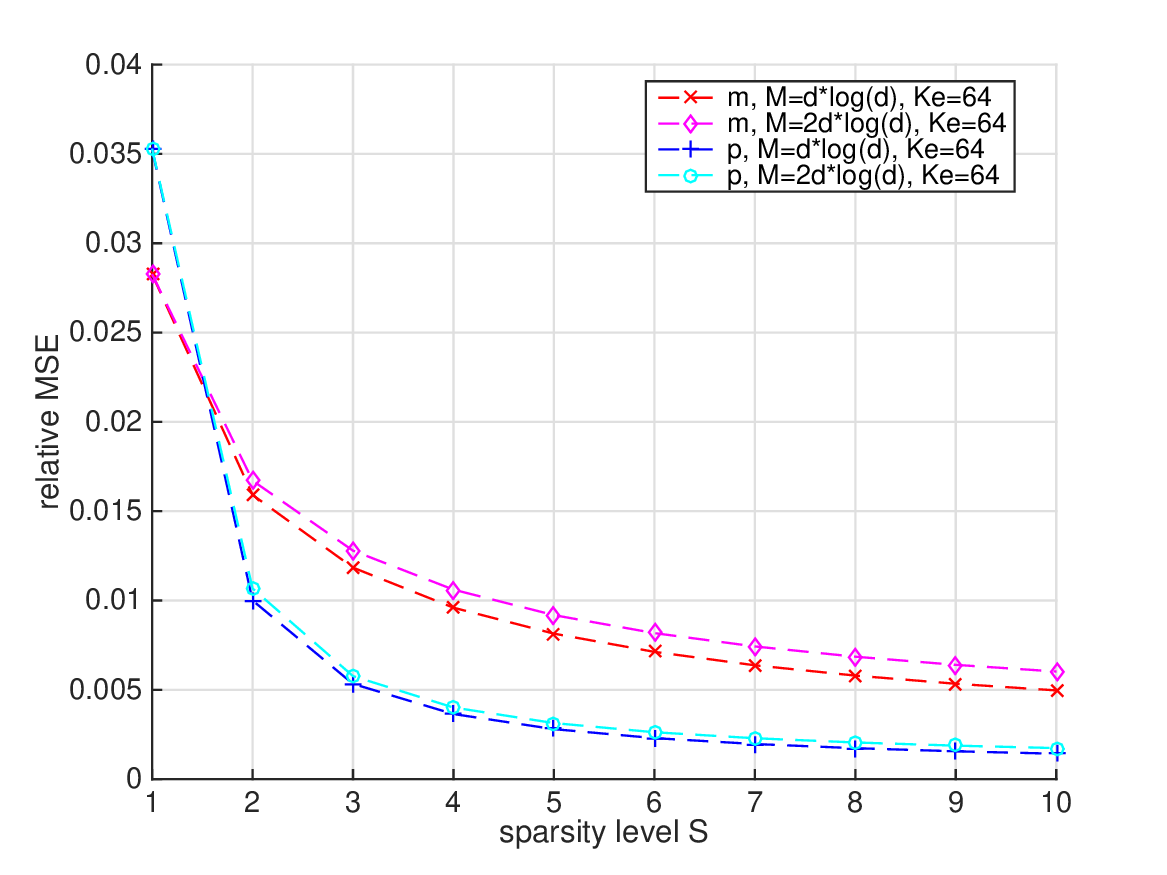} \\
	\end{tabular}
	\caption{Average dictionary sizes for adaptive dictionary learning based on ITKrM on all patches of {\it Mandrill}/{\it Peppers} for various initial dictionary sizes $K_e$ and required number of observations per atom $M$ (left). Average sparse approximation error of all patches of {\it Mandrill}/{\it Peppers} using OMP and the learned dictionaries with $K_e = 64$ and both choices of $M$ (right). \label{fig:adap_im}}
\end{figure}
In the first experiment we compare the sizes of the dictionaries learned on both clean images with various parameter settings as well as their approximation powers. The approximation power of a dictionary augmented by the flat atom $\atom_0$ for a given sparsity level $S$ is measured as $\| Y - \tilde Y\|_F^2/\|Y\|_F^2$, where $\tilde Y = (\tilde y_1, \ldots ,\tilde y_n)$ and $\tilde y_n$ is the S-sparse approximation to $y_n$ calculated by Orthogonal Matching Pursuit, \cite{parekr93}.\\
The results are shown in Figure~\ref{fig:adap_im}. We can see that as for synthetic data the final size of the learned dictionary does not depend much on the initial dictionary size, but does depend on the minimal number of observations. So for $M=\round{d\log(d)}$ the average dictionary size is about 106 atoms for {\it Mandrill} and 55 atoms for {\it Peppers}, while for $M=2d\log(d)$ we have about 54 atoms on {\it Mandrill} and 36 atoms on {\it Peppers}. The estimated sparsity level vs. average number of correctly identified atoms for {\it Mandrill} is $S_e = \round{2.1}=2$ vs. $S_t \approx1.5$ and for {\it Peppers} $S_e = \round{2.9}= 3$ vs. $S_t \approx 2.25$. Comparing the approximation power, we see that for both images the smaller (undercomplete) dictionaries barely lag behind the larger dictionaries. The probably most interesting aspect is that despite being smaller, the {\it Peppers}-dictionaries lead to smaller error than the {\it Mandrill}-dictionaries. This confirms the intuition that the smooth image {\it Peppers} has a lot more sparse structure than the textured image {\it Mandrill}. To better understand why for both images the larger dictionaries do not improve the approximation much, we have a look at the number of reliable observations for each atom in the last trial of $K_e=64$ in Figure~\ref{fig:adap_obs}. The corresponding dictionaries for {\it Mandrill}/{\it Peppers} can be found in Figures~\ref{fig:man_dico_obs}/\ref{fig:pep_dico_obs}.\\
\begin{figure}[h!]
	\centering
	\includegraphics[width=0.4\textwidth]{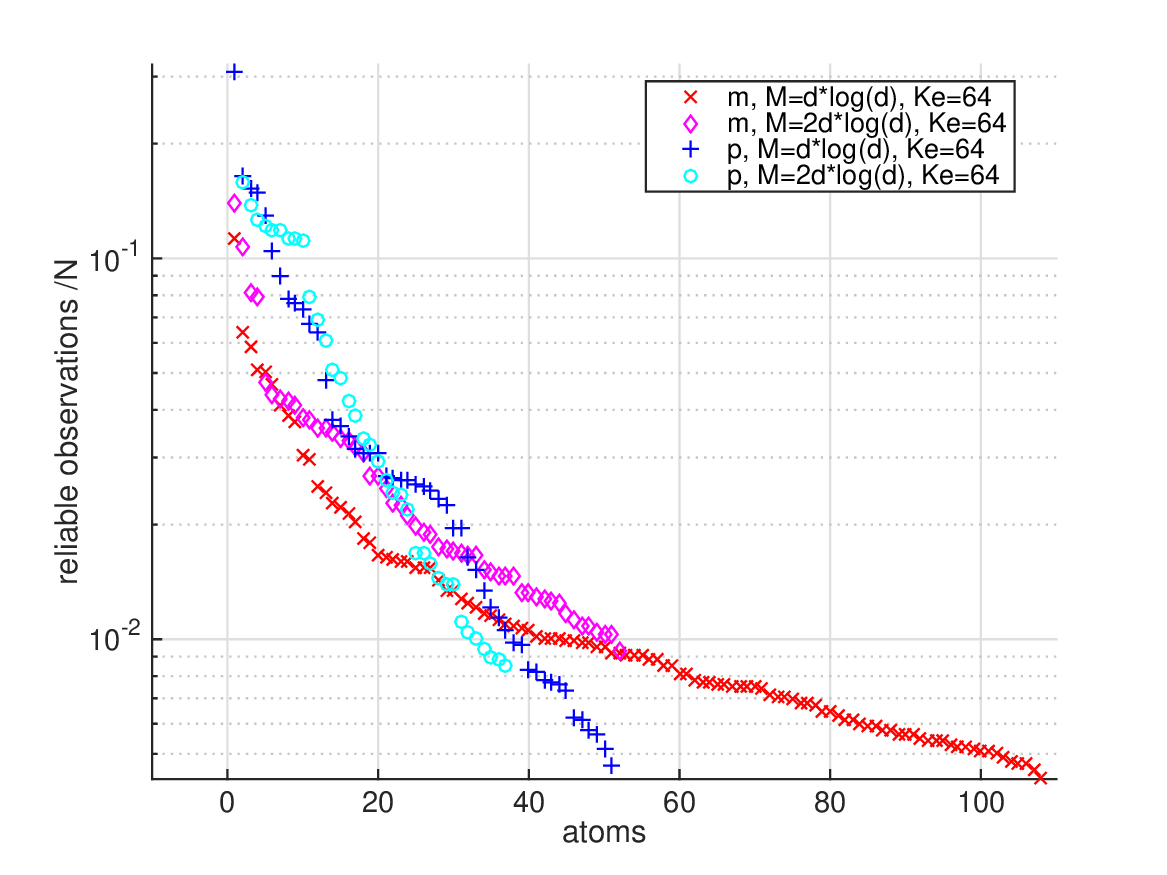} 
	\caption{Final number of reliable observations of the atoms in the dictionaries learned on {\it Mandrill}/{\it Peppers} with initial dictionary size
		$K_e=64$ in the last trial. \label{fig:adap_obs}}
\end{figure}
\begin{figure}[p]
	\centering
	\begin{tabular}{rr}
		\multicolumn{2}{c}{\includegraphics[height =3cm]{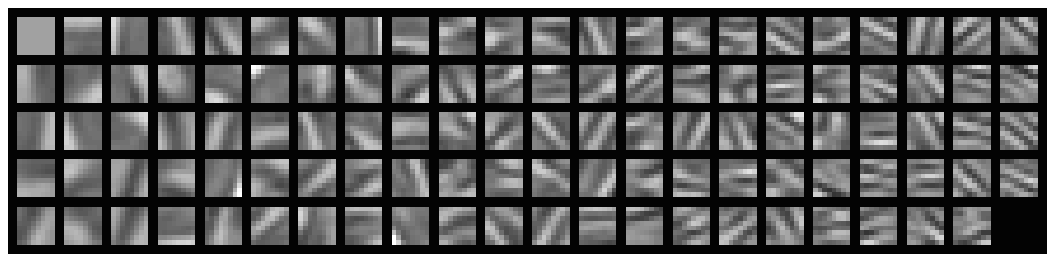}} \\
		\phantom{blablabla}\includegraphics[height=3cm]{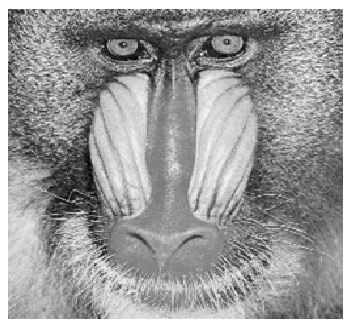}&\includegraphics[height=3cm]{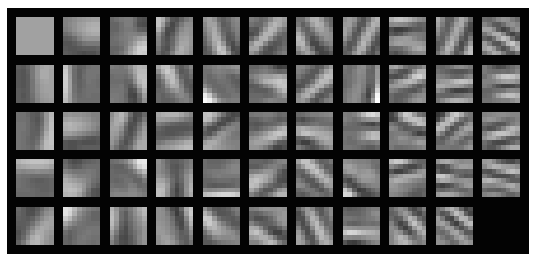}  \\ 
	\end{tabular}
	\caption{Dictionaries learned on {\it Mandrill} with initial dictionary size $K_e =64$ and required number of observations $M=\round{d\log(d)}$ (top) resp. $M=2d\log(d)$ (bottom). \label{fig:man_dico_obs}}
\end{figure}
\begin{figure}[p]
	\centering
	\begin{tabular}{lll}
		\includegraphics[height=3cm]{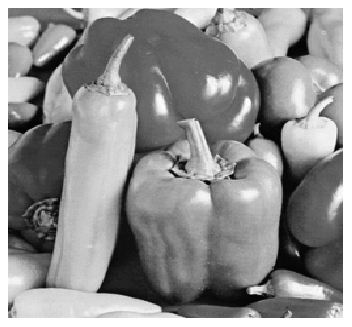} &\includegraphics[height =3cm]{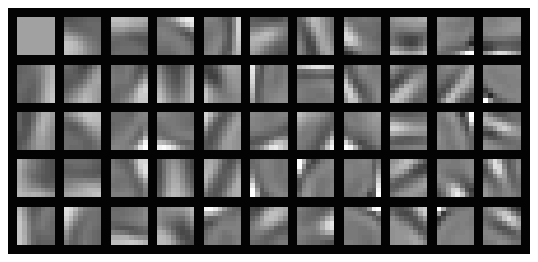} &
		\includegraphics[height=3cm]{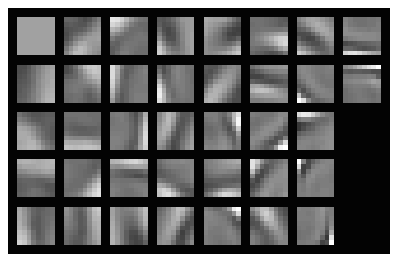}
	\end{tabular}
	\caption{Dictionaries learned on {\it Peppers} with initial dictionary size $K_e =64$ and required number of observations $M=\round{d\log(d)}$ (middle) resp. $M=2d\log(d)$ (right). \label{fig:pep_dico_obs}}
\end{figure}
\begin{figure}[p]
	\hspace{-1em}
	\begin{tabular}{llll}
		\includegraphics[height=3cm]{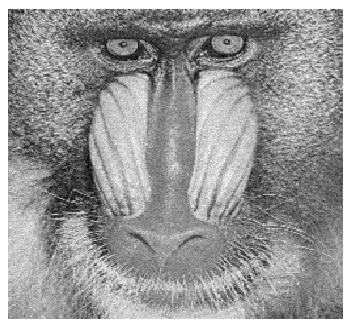} &\hspace{-1em} \includegraphics[height =3cm]{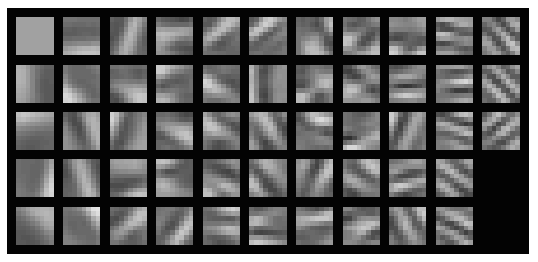} 
		\includegraphics[height=3cm]{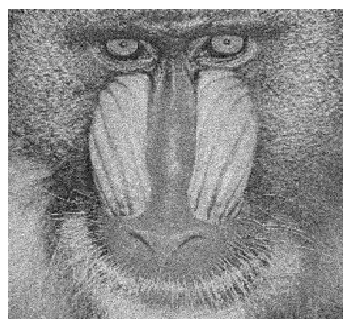} & \hspace{-1em} \includegraphics[height=3cm]{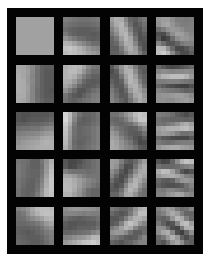}
	\end{tabular}
	\caption{Dictionaries learned on the {\it Mandrill} image contaminated with Gaussian noise of variance $\tilde \nsigma^2 = 10/255$  (left)
		and $\tilde \nsigma^2 =  20/255$ (right), initial dictionary size $K_e =64$ and required number of observations $M=\round{d\log(d)}$.  \label{fig:man_dico_noisy}}
\end{figure}
\noindent We can see that for both images the number of times each atom is observed strongly varies. If we interpret the relative number of observations as probability of an atom to be used, the first ten atoms are more than 10 times more likely to be used/observed than the last ten atoms. This accounts for the fact that by increasing the threshold for the minimal number of observations, we reduce the dictionary size without much affecting the approximation power.
\\
In our second experiment we learn adaptive dictionaries on {\it Mandrill} contaminated with Gaussian noise of variance $\tilde \nsigma^2 =  \nsigma^2/255$ for $ \sigma^2 \in \{5,10,15,20\}$, corresponding to average peak signal to noise ratios $\{34.15, 28.13, 24.61, 22.11\}$. We again compare their sizes and their approximation power for the clean image patches. The results are shown in Figure~\ref{fig:adap_noisy} and two example dictionaries are shown in Figure~\ref{fig:man_dico_noisy}.\\
\begin{figure}[t]
	\centering
	\begin{tabular}{cc}
		\includegraphics[width=0.4\textwidth]{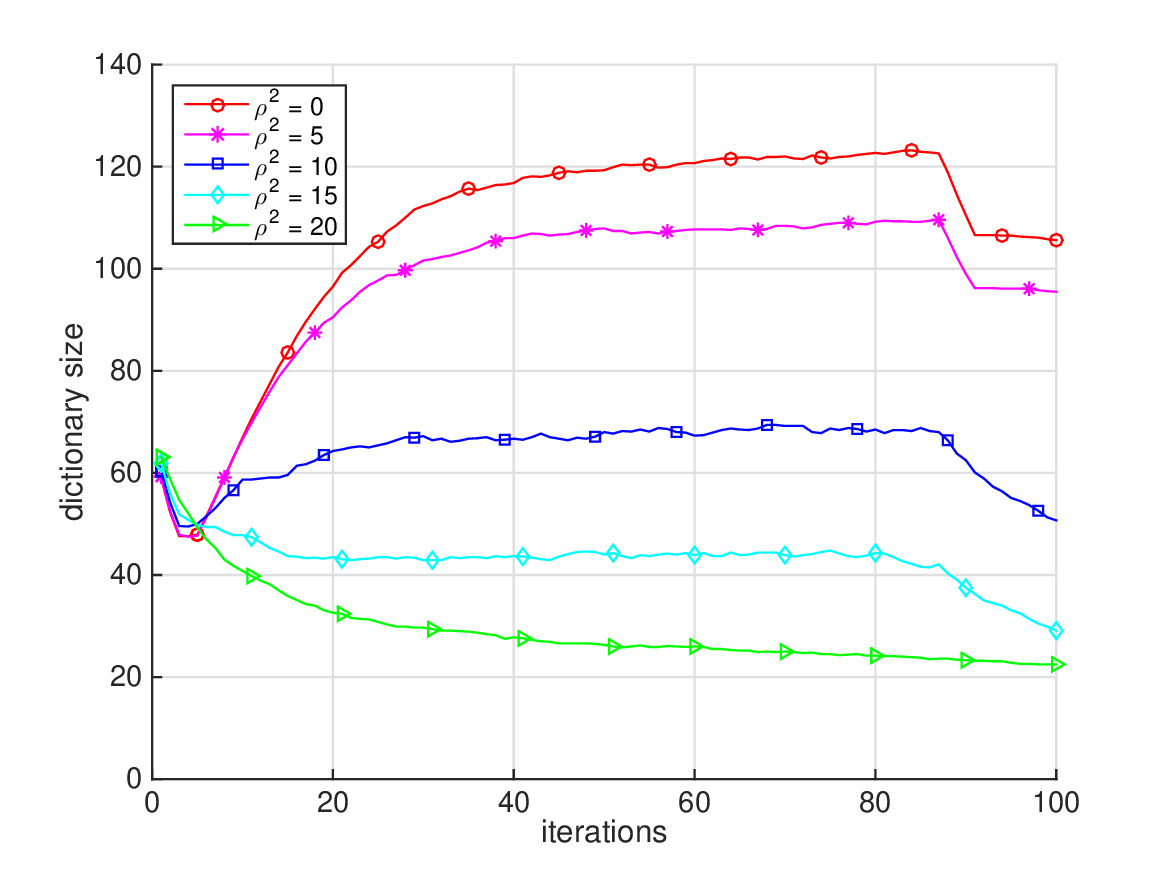}& \includegraphics[width=0.4\textwidth]{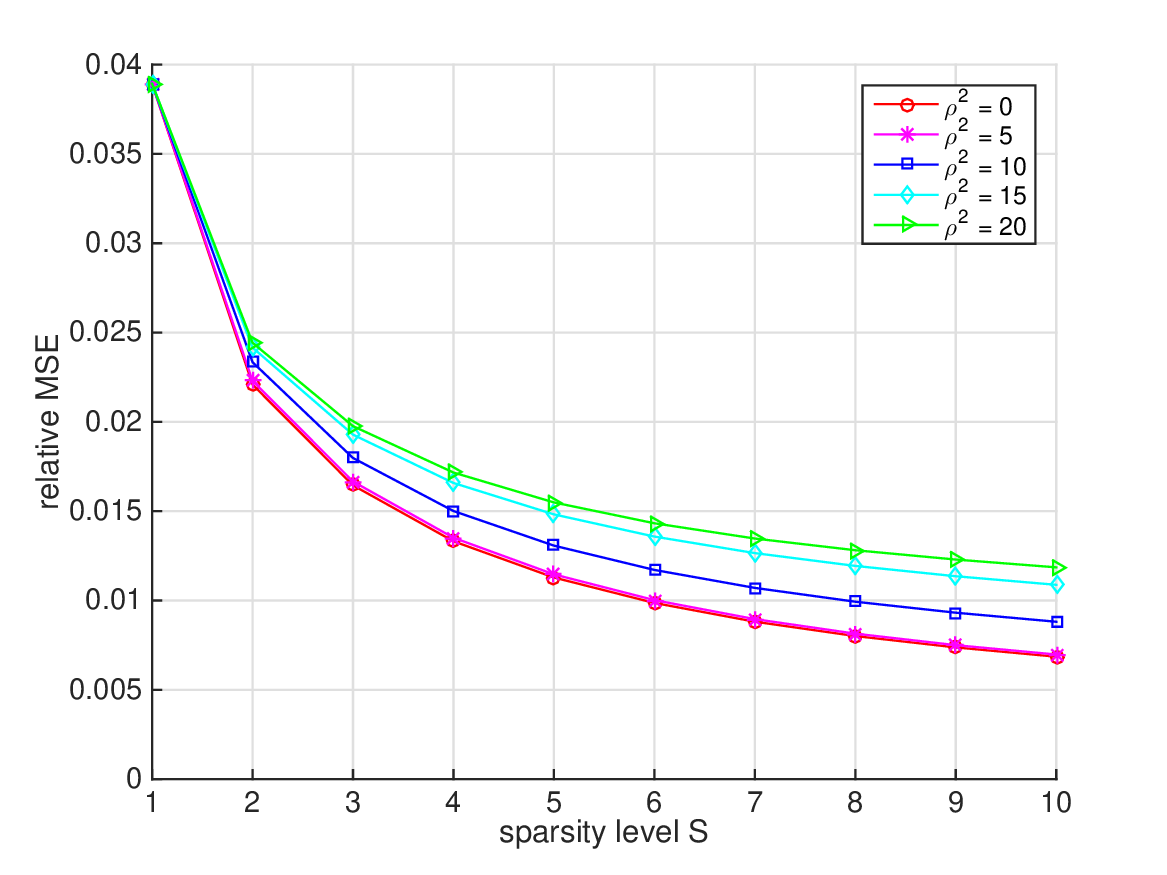}\\ 
	\end{tabular}
	\caption{Average dictionary sizes for adaptive dictionary learning based on ITKrM on all patches of the {\it Mandrill} image contaminated with Gaussian noise of variance $\tilde \nsigma^2 = \nsigma^2/255$, initial dictionary size $K_e =64$ and required number of observations $M=\round{d\log(d)}$ (left). Corresponding average sparse approximation error of all clean patches of {\it Mandrill} using OMP and the learned dictionaries. (right). \label{fig:adap_noisy}}
\end{figure}
We can see that the size of the dictionary decreases quite drastically with increasing noise, while the approximation power degrades only very gently. The sparsity level chosen by the algorithm is $S_e =  \round{1.80} = 2 $ for $\sigma_2 = 5$ and $S_e = \round{1.11} = \round{0.89} = 1$ for $\sigma_2 \in \{15,20\}$. For $\sigma^2 = 10$ the average recoverable sparsity level is $\approx 1.5$ so that for all trials the estimated sparsity level $S_e$ alternates between 1 and 2 in consecutive iterations. The fact that with increasing noise both the dictionary size and the sparsity level decrease but not the approximation power indicates that less often used atoms also tend to capture less energy per observation. This suggests an alternative value function, where each reliable observation is additionally weighted, for instance, by the squared coefficient or inner product. The corresponding cut-off threshold then is the minimal amount of energy that a reliable atom needs to capture from the training signals. Such an alternative value function could be useful in applications like dictionary based denoising, where every additional atom not only leads to better approximation of the signals but also of the noise.\\
Now that we have seen that adaptive dictionary learning produces sensible and noise robust results not only on synthetic but also on image data,
we will turn to a final discussion of our results.

\section{Discussion}\label{sec:discussion}

In this paper we have studied the global behaviour of the ITKrM (Iterative Thresholding and K residual means) algorithm for dictionary learning. We have proved that ITKrM contracts a dictionary estimate $\pdico$ towards the generating dictionary $\dico$ whenever the cross-Gram matrix $\pdico^\star \dico$ is diagonally dominant and $\pdico$ is incoherent and well-conditioned. Further, we have identified dictionaries, that are not equivalent to a generating dictionary but seem to be stable fixed points of ITKrM.
Using our insights that these fixed points always contain several atoms twice, meaning they are coherent, and that the residuals contain information about the missing atoms, we have developed a heuristic for finding good candidates which we can use
to replace one of two coherent atoms in a dictionary estimate. Simulations on synthetic data have shown that replacement using these candidates improved over random or no replacement, always leading to recovery of the full dictionary.
\\
Armed with replacement candidates, we have addressed one of the most challening problems in dictionary learning - how to automatically choose the sparsity level and dictionary size. We have developed a strategy for adapting the sparsity level from the initial guess $S_e=1$ and the dictionary size by decoupling replacement into pruning of coherent and unused atoms and adding of promising candidates. The resulting adaptive dictionary learning algorithm has been shown to perform very well in recovering a generating dictionary from random initialisations with various sizes on synthetic data with sparsity levels $S\geq 2$ and in learning meaningful dictionaries on image data.
\\
Note that our strategy for learning replacement candidates and adaptivity can be easily adapted to any other alternating minimisation algorithm for dictionary learning, such as MOD or K-SVD, \cite{enaahu99, ahelbr06}. Instead of using ITKsM on the residuals one simply has to use a small scale version with sparsity level $S=1$ of the respective algorithm. Preliminary experiments on synthetic data show that in the case of K-SVD and MOD this gives very similar results, \cite{rusc19}. Interestingly, however, it is much more difficult to stabilise the adaptive version of K-SVD on image data. The reason for this seems to be that the more sophisticated sparse approximation routine Orthogonal Matching Pursuit (OMP), \cite{parekr93}, promotes the convergence of candidate atoms to high energy subspaces. This leads to a concentration of many coherent atoms in these subspaces, without proportional increase in the approximation quality of the dictionary. These experimental findings are supported by theory, showing that while an excellent choice if the sparsifying dictionary is known, the performance of OMP rapidly degrades if the sparsifying dictionary is perturbed, \cite{pali21}.
Indeed on synthetic data without replacement procedures, K-SVD with OMP produces more double atoms than ITKrM. However, when using thresholding as sparse approximation procedures also for K-SVD and MOD, MOD produces the smallest number of double atoms followed by K-SVD and then ITKrM. With replacement MOD recovers the most accurate dictionary again followed by K-SVD and ITKrM.
\\
These observations suggest that the use of more sophisticated algorithms to bridge the gap between average sparsity level $\bar S$ and the average number of correctly identified atoms $S_t$ is only advisable in the last iterations when the number of atoms is not changed anymore.
\\
A promising direction to avoid the concentration of coherent atoms in high energy subspaces is to use a different value function for the dictionary atoms. The one proposed here is based on the assumption that all atoms are used equally often, which for image data was clearly not the case.
Alternatively, the current function could be scaled, that is, every reliable occurence of an atom is weighted with the squared coefficient which is related to the signal energy lost without this atom. Preliminary experiments with ITKrM using this weighted value function and an accordingly scaled cut-off indicate that it indeed helps to remove spurious atoms on synthetic data and leads to smaller dictionaries with the same approximation power on image data.\\
Another interesting question, which leads directly to our future theoretical research directions, is how to combine the value function of an atom with its coherence. The key to a solution lies in the analysis of dictionary learning from data where not all atoms are used equally often. A first step in this direction was laid in \cite{rusc20}, which provides results for the conditioning of random subdictionaries with non-homogeneous atom distributions together with an analysis for several sparse approximation algorithms. It also revealed the interplay between coherence structure of the dictionary and the probability of each atom to be used. Also to better model such a situation it will become necessary to look at a modification of thresholding, which uses as similarly noise inspired threshold $\tau$ as our value function, rather than a fixed sparsity level.\\
To put our replacement strategies on a more solid foundation, we are currently working on proofs showing the existence of stable spurious fixed points of ITKrM. The next steps are to show that all contractive areas are actually convergent as well as partial convergence of most atoms as long as the cross-Gram-matrix has the desired structure except for a few rows or columns. 
Finally, we want to transfer our characterisation of contractive areas to MOD and weighted ITKrM, where we replace the sign of the inner products in the update formula with the formula itself. This weighted version can also be seen as approximative K-SVD in the sense that only one power iteration is used to calculate the left singular vectors.

\acks{This work was supported by the Austrian Science Fund (FWF) under Grant no.~Y760.
The computational results presented have been achieved (in part) using the HPC infrastructure LEO of the University of Innsbruck.
We would like to thank Alexander Steinicke for his explanations concerning Freedman's inequality, quadratic variation and the Doob martingale
as well as Simon Ruetz for proof-reading the manuscript.
}

\newpage
\appendix

\section{Exact Statement and Proof of Theorem~\ref{maintheorem} }\label{appendix}

We only state and prove the exact version of the second part, since the first part consists literally of considering just one iteration of ITKrM and replacing in Theorem~4.2 of \citep{sc15} the assumption on the distance $d(\pdico,\dico)$ with the assumptions on the coherence and the operator norm of $\pdico$, ie. 
\begin{align}
	d(\pdico,\dico) \leq \frac{1}{32\sqrt{S}} \quad \rightsquigarrow \quad \mu(\pdico) \leq \frac{1}{20\log K} \quad \mbox{ and } \quad \|\pdico\|_{2,2}^2 \leq \frac{K}{134e^2 S \log K} -1.
\end{align}
Similarly the amendment to the proof consists in using Lemma~\ref{lemma_b8s} in Appendix~\ref{subsec:app_tech:oracle_res} instead of Lemma~B.8 of \citep{sc15} and potentially some tweaking of constants.

\begin{proposition}[Theorem~\ref{maintheorem}(b) exact]
	Assume that the signals $y_n$ follow model~\eqref{noisymodel2} for a dictionary $\dico $ with $\|\dico \|_{2,2}^2\leq\frac{K}{98S}$ and for coefficients with gap $c(S\!+\!1)/c(S) \leq \gap$, dynamic sparse range $c(1)/c(S) \leq \dynr$, noise to coefficient ratio $\rho/c(S)\leq \ncr$ and relative approximation error $\|c(\Sset^c)\|_2 / c(1)\leq \apperr \leq \frac{12}{7}\sqrt{\log K}$. Further, assume that the coherence and operator norm of the current dictionary estimate $\pdico$ satisfy,
	\begin{align}\label{pdico_mu_B_exact}
	\mu(\pdico) \leq \frac{1}{20\log K}  \quad \mbox{ and } \quad \|\pdico\|_{2,2}^2 \leq \frac{K}{134e^2 S \log K} -1.
	\end{align}
	If $d(\pdico,\dico) \geq \frac{1}{32\sqrt{S}}$ but the cross Gram matrix $\dico^\star \pdico$ is diagonally dominant in the sense that
	\begin{align}\label{diagdom_cond}
	\min_k \absip{\patom_k}{\atom_k}  \geq& \max\bigg\{8\, \gap \cdot \max_{k} \absip{\patom_k}{\atom_k} , \notag \\
	&\hspace{2cm} 40\, \ncr \cdot \sqrt{\log K} ,\notag\\
	&\hspace{3cm} 48\, \dynr  \cdot \log K \cdot\mu(\dico,\pdico), \notag\\
	&\hspace{4cm}  14\, \dynr \cdot\sqrt{\|\dico\|_{2,2}^2S\log K/(K\!-\!S)}\bigg\},
	\end{align}
	then one iteration of ITKrM using $N$ training signals will reduce the distance by at least a factor $\kappa \leq 0.95$, meaning $d(\ppdico,\dico) \leq 0.95\cdot  d(\pdico,\dico)$, except with probability 
	\begin{align*}
	3K\exp\left( -\frac{N  C^2_r\gamma^2_{1,S} \cdot  \eps}{768 K\max\{S,\|\dico\|_{2,2}^2\!+\!1\}^{\frac{3}{2}}}\right) + 4K\exp\left( -\frac{N  C^2_r\gamma_{1,S}^2 \cdot \eps^2}{512 K \max\{S,\|\dico\|_{2,2}^2\!+\!1\} \left(1+ d\nsigma^2\right)}\right).
	\end{align*} 
\end{proposition}

\begin{proof}
	We follow the outline of the proof for Theorem~4.2 in \citep{sc15}. However, to extend the convergence radius we need to introduce new ideas, first for bounding the difference between the oracle residuals based on $\pdico$ and $\dico$, replacing Lemma~B.8 of \citep{sc15}, and second for bounding the probability of thresholding with $\pdico$ not recovering the generating support or preserving the generating sign, replacing Lemma~B.3/4 of \citep{sc15}.
	We denote the thresholding residual based on $\pdico$ by 
	\begin{align} \label{defRt}
	R^t(\pdico, y_n, k) := \big[y_n - P(\pdico_{I_{\pdico,n}^t}) y_n + P(\patom_k) y_n\big] \cdot \signop(\ip{\patom_k}{y_n}) \cdot  \chi(I_{\pdico,n}^t, k),
	\end{align}
	and the oracle residual based on the generating support $I_n=p_n^{-1}(\Sset)$, the generating signs $\sigma_n$ and $\pdico$, by
	\begin{align}\label{defRo}
	R^o(\pdico, y_n, k) := \big[y_n - P(\pdico_{I_n}) y_n + P(\patom_k) y_n\big] \cdot \sigma_n(k) \cdot  \chi(I_n,k).
	\end{align}
	Abbreviating $s_k=\frac{1}{N} \sum_n \ip{y_n}{\atom_k} \cdot \sigma_n(k) \cdot  \chi(I_n,k)$ and setting $B:=\|\dico\|_{2,2}^2$ as well as $\eps:= d(\pdico,\dico)$ for conciseness, we know from the proof of Theorem~4.2 in \citep{sc15} that
	\begin{align}
	\| \bar \patom_k - s_k \atom_k\|_2 &\leq \frac{1}{N} \Big\| \sum_n \left[R^t(\pdico, y_n, k) - R^o(\pdico, y_n, k)\right] \Big\|_2\notag\\
	&\hspace{2cm} + \frac{1}{N} \Big\| \sum_n \left[R^o(\pdico, y_n, k) - R^o(\dico, y_n, k)\right] \Big\|_2\notag\\
	&\hspace{4cm} +\frac{1}{N} \Big\| \sum_n \big[y_n - P(\dico_{I_n}) y_n\big] \cdot \sigma_n(k) \cdot  \chi(I_n,k) \Big\|_2 .
	\end{align}
	By Lemma~B.6 from \citep{sc15} we have
	\begin{align}
	\P&\left( \left| \frac{1}{N}\sum_n \chi(I_n,k)  \sigma_n(k) \ip{y_n}{\atom_k}\right| \leq (1-t_0) \frac{C_r \gamma_{1,S}}{K} \right) \notag \\
	&\hspace{4cm}\leq\exp\left( -\frac{ N C_r^2 \gamma_{1,S}^2 \cdot t_0^2}{2K(1+ \frac{SB}{K}+S\nsigma^2 + t_0 C_r \gamma_{1,S}\sqrt{B\!+\!1}/3)}\right).
	\end{align}
	By Lemma~\ref{lemma_b4s} in Appendix~\ref{subsec:app_tech:thresholding_res} (substituting Lemma B.3/4 of \citep{sc15}) we have
	\begin{align}
	&\P \left( \frac{1}{N}\left\|\sum_{n}\left[ R^t(\pdico ,y_n,k)-R^o(\pdico ,y_n,k)\right]\right\|_2> \frac{18(S\!+\!1)\sqrt{B\!+\!1}}{K^3} + \frac{C_r \gamma_{1,S}}{K}t_1\eps\right)\nonumber \\ 
	&\phantom{\frac{1}{N}\Big\|\sum_{n}\left[ R^t(\pdico ,y_n,k)-R^o(\pdico)\right]\Big\|_2}\leq 2\exp\left( -\frac{NC_r^2\gamma_{1,S}^2t_1^2\varepsilon^2}{\frac{108(S\!+\!1)(B\!+\!1)}{K}+3t_1\varepsilon C_r\gamma_{1,S}K\sqrt{B\!+\!1}}\right) .
	\end{align}
	By Lemma~\ref{lemma_b8s} in Appendix~\ref{subsec:app_tech:oracle_res} (substituting Lemma B.8 of \citep{sc15}) we have
	that for $0\leq t_2 \leq 1/8$ 
	\begin{align}
	\P&\left(\frac{1}{N} \left\| \sum_n \left[R^o(\pdico, y_n, k)-R^o(\dico, y_n, k) \right]\right\|_2 \geq \frac{C_r\gamma_{1,S}}{K}(0.308\eps+ t_2\eps) \right)\notag \\
	&\hspace{7cm}\leq \exp\left(- \frac{N  C^2_r\gamma^2_{1,S} \cdot t_2^2 \eps}{12K\max\{S,B\}^{\frac{3}{2}}} +\frac{1}{4}\right),
	\end{align} 
	and by Lemma~B.7 from \citep{sc15} we have
	\begin{align}
	&\P\left( \left\| \frac{1}{N}\sum_n \big[y_n - P(\dico_{I_n})y_n \big]\cdot \sigma_n(k) \cdot \chi(I_n,k) \right\|_2 \geq   \frac{C_r\gamma_{1,S}}{ K}\, t_3 \eps \right)\notag\\
	&\hspace{3cm}\leq \exp\left(- \frac{N  C^2_r\gamma_{1,S}^2 \cdot t_3\eps}{8 K \max\{S,B\!+\!1\}} \min\left\{\frac{t_3 \eps }{\left(1-\gamma_{2,S} + d\nsigma^2\right)},1\right\}+\frac{1}{4}\right).
	\end{align}
	Thus, with high probability we have $s_k \geq (1-t_0) \frac{C_r \gamma_{1,S}}{K}$ and
	\begin{align}
	\left\| \bar \patom_k - s_k \atom_k\right\|_2 &\leq \frac{C_r\gamma_{1,S}}{K} \left(\frac{ 18(S\!+\!1)\sqrt{B\!+\!1}}{K^2 C_r\gamma_{1,S}\eps} +t_1 +0.308+ t_2 + t_3 \right) \eps.
	\end{align}
	Note that we only need to take into account distances $\eps >\frac{1}{32\sqrt{S}}$, so we will use some crude bounds on $C_r\gamma_{1,S}$ to show that the fraction with $\eps$ in the denominator above is small. The requirement that $\|c(\Sset^c)\|_2 / c(1)\leq \apperr \leq \frac{12}{7}\sqrt{\log K}$ ensures that $\gamma_{1,S}\geq (1+ 3\log K)^{-1/2}$ and we trivially have $\gamma_{1,S}\geq S c(S)$. Combining this with the bound on $C_r$ in \eqref{Cr_bound} we get
	\begin{align}
	\frac{1}{C_r \gamma_{1,S}} \leq \frac{\sqrt{1+5d\nsigma^2}}{(1-e^{-d}) \gamma_{1,S}} \leq \frac{\sqrt{1+ 3\log K}}{(1-e^{-d})} + \frac{\nsigma}{c(S)}\frac{\sqrt{5d}}{S(1-e^{-d})}.
	\end{align} 
	The conditions in \eqref{diagdom_cond} imply that $K \geq 14^2 SB \log K$, which in turn means that $\log K >7$, as well as $\nsigma/c(S)\leq \ncr \leq 1/(40\sqrt{\log K})$. Assuming additionally that $K\geq \sqrt{d}$, meaning the dictionary is not too undercomplete, this leads to 
	\begin{align}
	\frac{18(S\!+\!1)\sqrt{B\!+\!1}}{K^2 C_r\gamma_{1,S}\eps} \leq 0.025,
	\end{align}
	for $B\geq 1$ and $S\geq 2$.
	Setting $t_0 = t_1 = 1/20$ and $t_2= t_3 = 1/8$ we get
	\begin{align}
	\max_k \left\| \bar \patom_k -s_k\atom_k\right\|_2 &\leq 0.633\cdot \frac{C_r\gamma_{1,S}}{K}\eps \quad \mbox{and} \quad \min_k s_k \geq 0.95 \cdot \frac{C_r \gamma_{1,S}}{K},
	\end{align}
	which by Lemma B.10 from \cite{sc15} implies that 
	\begin{align}
	d(\bar \pdico, \dico)^2=\max_k \left\| \frac{\bar \patom_k}{\|\bar \patom_k\|_2} - \atom_k\right\|^2_2 \leq 2\left(1-\sqrt{1 - \frac{0.633^2 \eps^2}{0.95^2}} \right) \leq \frac{2 \cdot 0.633^2 \eps^2}{0.95^2}\leq 0.89 \eps^2,
	\end{align}
	except with probability
	\begin{align*}
	&K\exp\left( -\frac{N C_r^2 \gamma_{1,S}^2 }{K(801 + 14 C_r \gamma_{1,S}\sqrt{B\!+\!1})}\right)
	+ 2K\exp\left( -\frac{N C_r^2 \gamma_{1,S}^2\cdot  \eps^2}{K(\frac{1}{10} +60\eps C_r\gamma_{1,S}\sqrt{B\!+\!1})}\right)\\
	& \hspace{1cm}+ e^{\frac{1}{4}} K\exp\left( -\frac{N  C^2_r\gamma^2_{1,S} \cdot  \eps}{768 K\max\{S,B\}^{\frac{3}{2}}}\right) + e^{\frac{1}{4}} K\exp\left( -\frac{N  C^2_r\gamma_{1,S}^2 \cdot \eps^2}{512 K \max\{S,B\!+\!1\} \left(1+ d\nsigma^2\right)}\right).
	\end{align*}
	The final probability bound follows from the observations that $C_r\gamma_{1,S} \leq \sqrt{S}$, $B\!+\!1\geq 2$ and $\eps \leq \sqrt{2}$.
\end{proof}

\section{Technical Lemmata}\label{sec:app_tech}

Here we state the proofs of the two lemmata characterising the difference between the thresholding and the oracle residual resp. the difference between the oracle residuals based on the generating dictionary and a perturbation.

\subsection{Difference between thresholding and oracle residual}\label{subsec:app_tech:thresholding_res}

To prove Lemma~\ref{lemma_b4s} we will make use of the scalar version of Bernstein's inequality \citep{bennett62} and Hoeffding's inequality \citep{hoeffding}. We will also need Proposition~\ref{inner_product_freedman}, which is based on a version of Freedman's inequality, \citep{fr75}, to deal with sums of dependent random variables.

\begin{theorem}[Scalar Bernstein, \citep{bennett62}] \label{scalarbernstein}
	Let $v_n\in \R$, $n=1\ldots N$, be a finite sequence of independent random variables with zero mean. If 
	$\E(v^2_n)\leq m $ and $ \E(|v_n|^k )\leq \frac{1}{2} k!\,m M^{k-2}$ for all $k > 2$, then for all $t>0$ we have
	\begin{align}
	\P\left( \sum_n v_n \geq t\right)\leq \exp\left(- \frac{t^2}{2(N m + Mt) } \right).\notag
	\end{align}
\end{theorem}
To prove Proposition~\ref{inner_product_freedman} we need the following simplified version of Freedman's inequality.
\begin{theorem}[Freedman, \citep{fr75}]\label{freedman}
	Let $X_0,\dots ,X_S$ be a martingale sequence with bounded differences, that is $|X_k - X_{k-1}|\leq c$ almost surely for each $k$. Moreover, let the predictable quadratic variation $\langle X \rangle_S=\sum_{k=1}^S\E\left[\left(X_k-X_{k-1}\right)^2\middle|\mathcal{F}_{k-1}\right]$ be bounded by $b$. Then for all $t>0$
	\begin{equation*}
	\P\left( X_S - X_0 \geq t \right)\leq \exp\left( -\frac{t^2}{2(ct+b)} \right).
	\end{equation*}
\end{theorem}
\begin{proposition}\label{inner_product_freedman}
	Let $v \in \R^K$ be a vector, $I=(i_1,\ldots ,i_S)$ be a sequence of length $S$ obtained by sampling from $\mathbb{K}=\{ 1,\dots ,K \}$ without replacement, $\varepsilon$ with values in $\{-1,1\}^S$ a Rademacher vector independent from $I$ and $c \in \R^S$ a scaling vector. Then for any $t\geq 0$,
	\begin{align}
		\P\left(|\sum_{k = 1}^S c_k \varepsilon_k v_{i_k} | \geq t \right)\leq 2 \exp \left( \frac{-t^2}{2(\|c\|_\infty \|v\|_\infty t+ \|c\|_2^2\|v\|_2^2/(K-S))} \right) .
	\end{align}
\end{proposition}

\begin{proof} 
	We will use Theorem~\ref{freedman} on an appropriately constructed martingale. Let
	$I=(i_1,\dots ,i_S)$, be the random vector obtained by sampling from $\mathbb{K}=\{ 1,\dots ,K \}$ without replacement, that is, $I$ is drawn uniformly at random from the set
	\begin{equation*}
	\Omega :=\{ \omega \in \mathbb{K}^S:\omega_i\neq\omega_j\text{ for }i\neq j \} .
	\end{equation*}
	We equip $\Omega$ with the $\sigma$-algebra $\mathcal{F} :=\mathcal{P}(\Omega )$ and the point measure $\P(\{\omega\}) := |\Omega|^{-1}$, to get the probability space $(\Omega, \mathcal{F}, \P)$. We also set $\Delta := \{-1,1\}^S$, equip it with the $\sigma$-algebra $\mathcal{A} :=\mathcal{P}(\Delta)$, the point measure $\mathbb{Q}(\{\delta\})=2^{-S}$ and define the
	product space $(\Omega \times \Delta, \mathcal{F}\otimes \mathcal{A}, \P\otimes \mathbb{Q})$.
	On $\Omega$ we define the filtration $\{ \emptyset, \Omega \}=\mathcal{F}_0\subseteq\mathcal{F}_1\subseteq\dots\subseteq\mathcal{F}_S = \mathcal{F}$, where $\mathcal{F}_k$ is the $\sigma$-algebra induced by $\omega_1,\dots ,\omega_k$. To be exact, we define the random variables
	\begin{align*}
	i_j:\Omega \longrightarrow &\mathbb{K} \subseteq\R \quad \mbox{with} \quad i_j(\omega):=\omega_j, 
	\end{align*}
	and set $\mathcal{F}_k:=\sigma (i_1,\dots ,i_k)$ for all $1\leq k \leq S$. \\
	Since we also want to condition on the signs $\delta \in \Delta$, we define the random variables
	\begin{align*}
	\eps_j:\Delta \longrightarrow &\{-1,1\} \subseteq\R \quad \mbox{with} \quad \eps_j(\delta):=\delta_j, 
	\end{align*}
	and the corresponding filtration $\{ \emptyset,\Delta \}=\mathcal{A}_0\subseteq\mathcal{A}_1\subseteq\dots\subseteq\mathcal{A}_S = \mathcal{A}$, by setting $\mathcal{A}_k =\sigma (\eps_1,\dots ,\eps_k)$.
	On the product space $\Omega \times \Delta$ we then get the filtration $\mathcal{F}_k \otimes \mathcal{A}_k$.\\
	Next we define the bounded random variables 
	\begin{align*}
	X_k:\Omega \times \Delta \longrightarrow \R,
	\quad \mbox{with} \quad X_k(\omega, \delta)=\sum_{j=1}^{k} c_j \eps_j(\delta) v_{i_j(\omega)}.
	\end{align*}
	The random variables $X_k$ form a martingale sequence with resp. to the filtration $\mathcal{F}_k \otimes \mathcal{A}_k$, since by independence of $\eps_k$ to $\mathcal{F}\otimes \mathcal{A}_{k-1}$ we have
	\begin{align*}
	\E[X_k - X_{k-1} | \mathcal{F}_{k-1}\otimes \mathcal{A}_{k-1}] &= \E\big[\E[c_k \eps_k v_{i_k} | \mathcal{F}\otimes \mathcal{A}_{k-1}] \big| \mathcal{F}_{k-1}\otimes \mathcal{A}_{k-1}\big] \\
	&= c_k \E\left[v_{i_k}  \E[\eps_k | \mathcal{F}\otimes \mathcal{A}_{k-1}] \middle| \mathcal{F}_{k-1}\otimes \mathcal{A}_{k-1}\right]  = 0.
	\end{align*}
	The sum we want to estimate is $X_S$ with expectation $\E(X_S)= 0 = X_0$. Further we have $|X_k - X_{k-1}| = |c_k \eps_k v_{i_k}| \leq \|c\|_\infty\|v\|_\infty$ as well as $(X_k - X_{k-1})^2 = c^2_k v_{i_k}^2$, so we can bound the predictable quadratic variation as
	\begin{align*}
	\langle X\rangle_S&=\sum_{k=1}^S\E\left[(X_k - X_{k-1})^2 \middle|\mathcal{F}_{k-1}\otimes \mathcal{A}_{k-1}\right] \\
	&=\sum_{k=1}^S\E\left[c_k^2 v_{i_k}^2\middle|\mathcal{F}_{k-1}\otimes \mathcal{A}_{k-1}\right] = \sum_{k=1}^S c_k^2 \sum_{\ell \notin \{i_1,\ldots,i_{k-1}\}} v^2_\ell \frac{1}{K-k+1} \leq \|c\|_2^2\frac{\|v\|^2_2}{K-S} 
	\end{align*}
	The final result follows using the symmetry of $X_S$.
\end{proof}

Now we are ready to prove the lemma estimating the error originating from thresholding failing to recover the generating supports and signs.

\begin{lemma}\label{lemma_b4s}
	Assume that the signals $y_n$ follow model~\eqref{noisymodel2} for coefficients with gap $c(S\!+\!1)/c(S) \leq \gap$, dynamic sparse range $c(1)/c(S) \leq \dynr$, noise to coefficient ratio $\rho/c(S)\leq \ncr$ and relative approximation error $\|c(\Sset^c)\|_2 / c(1)\leq \apperr \leq \frac{12}{7}\sqrt{\log K}$. If the cross Gram matrix $\dico^\star \pdico$ is diagonally dominant in the sense that
	\begin{align}
	\min_k \absip{\patom_k}{\atom_k}  \geq &\max\bigg\{8\, \gap \cdot \max_{k} \absip{\patom_k}{\atom_k} , \notag \\
	&\hspace{2cm} 40\, \ncr \cdot \sqrt{\log K} ,\notag\\
	&\hspace{3cm} 48\, \dynr  \cdot\log K\cdot\mu(\dico,\pdico), \notag\\
	&\hspace{4cm}  14\, \dynr \cdot\sqrt{\|\dico\|_{2,2}^2S\log K/(K\!-\!S)}\bigg\} ,
	\end{align}
	then
	\begin{align}
	&\P \left( \frac{1}{N}\left\|\sum_{n}\left[ R^t(\pdico ,y_n,k)-R^o(\pdico ,y_n,k)\right]\right\|_2> \frac{18(S\!+\!1)\sqrt{\|\dico\|_{2,2}^2\!+\!1}}{K^3} + \frac{C_r \gamma_{1,S}}{K}t\eps\right)\nonumber \\ 
	&\hspace{4.3cm}\leq 2\exp\left( -\frac{NC_r^2\gamma_{1,S}^2t^2\varepsilon^2}{\frac{108(S\!+\!1)(\|\dico\|_{2,2}^2\!+\!1)}{K}+3t\varepsilon C_r\gamma_{1,S}K\sqrt{\|\dico\|_{2,2}^2\!+\!1}}\right) .
	\end{align}
\end{lemma}

\newcommand{\gset}{F}
\newcommand{\gsetm}{\mathcal{F}}
\newcommand{\rset}{G}
\newcommand{\rsetm}{\mathcal{G}}

\begin{proof}
	Throughout the proof we will use the abbreviations $B=\|\dico\|_{2,2}^2$ and $\hat{\mu}=\mu (\dico ,\pdico )$.
	To estimate the difference between the oracle and the thresholding residuals, we have to distinguish between four different cases, based on whether $k$ is in the oracle support or not and whether thresholding recovers the oracle support and sign, so we set 
	\begin{align*}
	\gsetm &= \{ n: k\in I_n \wedge \left( I^t_n \neq I_n \vee \signop(\ip{\patom_k}{y_n})\neq\sigma_n(k) \right) \} ,\\
	\rsetm &= \{ n: k\notin I_n \wedge k \in I_n^t\} .
	\end{align*} 
	Whenever a signal is not in one of the sets above, the residuals coincide, yielding 
	\begin{align}
	\Delta &=\Big\|\sum_{n}\left[ R^t(\pdico,y_n,k)-R^o(\pdico,y_n,k)\right]\Big\|_2= \Big\| \sum_{n\in \gsetm \cup  \rsetm } \left[ R^t(\pdico ,y_n,k)-R^o(\pdico ,y_n,k)\right] \Big\|_2 .
	\end{align}
	Further observing that operators of the form $\mathbb{I}_d-P(\pdico_J)+P(\patom_k)$ with $k\in J$ are orthogonal projections, and that our signals are bounded, $\|y_n\|_2 \leq \sqrt{B+1}$, as well as $R^o(\pdico ,y_n,k) = 0$ for $n \in \rsetm$ leads to
	\begin{align}\label{norm_thresholding_oracle_residual}
	\Delta &\leq  \sum_{n\in \gsetm \cup \rsetm} \left(\| R^t(\pdico ,y_n,k)\|_2 + \|R^o(\pdico ,y_n,k)\|\right)  \leq (2|\gsetm|  +  |\rsetm|)\sqrt{B+1} .
	\end{align}
	To upper bound the size of the set $\gsetm$, we apply Bernstein's inequality to the sum of $N$ i.i.d copies of the centered random variable $\mathbf{1}_F -\P(\gset)$, where 
	\begin{align}
	\gset = \left\{ y : k\in I \wedge \left( I^t \neq I \vee \signop(\ip{\patom_k}{y})\neq\sigma(k) \right) \right\},
	\end{align}
	which leads to 
	\begin{align}\label{probG}
	\P(|\gsetm|\geq N \P(\gset) +Nt )\leq \exp\left( -\frac{t^2N}{2\P(\gset)+t}\right).
	\end{align}
	Similarly defining $\rset = \{ y: k\notin I \wedge k \in I^t\}$, we get
	\begin{align}\label{probB}
	\P(|\rsetm|\geq N \P(\rset) +Nt )\leq \exp\left( -\frac{t^2N}{2\P(\rset)+t}\right).
	\end{align}
	So what remains to calculate is the probability of the events $\gset$ and $\rset$, that is of thresholding failing to recover the oracle support and sign when $k$ is in the support and of accidentally recovering $k$ when it is not in the support.

	\subsubsection*{Step~1 - Failure probability of the recovery of $I$ or the correct sign $\sigma(k)$}\label{subsec:failureprob_Ig}
	
	Here we will show that with high probability for a signal $y$ following the model in~\eqref{noisymodel2} with $k\in I$, we have $I^t = I = p^{-1}(\Sset)$ and $\signop(\ip{\patom_k}{y})= \sigma(k)$.
	\\
	To ensure $I^t=I$, this means the recovery of all $i\in I$, we need to have
	\begin{align}\label{condition_th_recovery_Ig}
	\min_{i\in I}|\ip{\patom_i}{y}| >\max_{i\notin I}
	|\ip{\patom_i}{y}| .
	\end{align}
	Expanding the inner product of a rescaled signal $y$ with an atom $\patom_i$ of the perturbed dictionary $\pdico$ yields
	\begin{align*}
	|\ip{\patom_i}{\dico x_{c,p,\sigma}+r}| &=|\sum_j\sigma(j)c(p(j))\ip{\patom_i}{\atom_j} +\ip{\patom_i}{r}| \\
	&=| c(p(i))\ip{\patom_i}{\atom_i} +\sigma(i)\sum_{j\neq i}\sigma(j)c(p(j))\ip{\patom_i}{\atom_j} +\sigma(i)\ip{\patom_i}{r}| .
	\end{align*}
	Depending on the index $i$ under consideration, we obtain the following bounds from below resp. above (remember that $\alpha_{\min}\leq\absip{\patom_i}{\atom_i}\leq\alpha_{\max}$),
	\begin{align*}
	i\in I :\; \; |\ip{\patom_i}{\dico x_{c,p,\sigma}+r}| &\geq c(S)\alpha_{\min} - \big|\sum_{j\neq i}\sigma(j)c(p(j))\ip{\patom_i}{\atom_j}\big| - |\ip{\patom_i}{r}| ,\\
	i\notin I:\; \; |\ip{\patom_i}{\dico x_{c,p,\sigma}+r}| &\leq c(S+1)\alpha_{\max}+\big|\sum_{j\neq i}\sigma(j)c(p(j))\ip{\patom_i}{\atom_j}\big| +|\ip{\patom_i}{r}|. 
	\end{align*}
	This means that a sufficient condition for the recovery of $I$ is that for all $i$
	\begin{equation}\label{condition_recovery_Ig2}
	\big|\sum_{j\neq i}\sigma(j)c(p(j))\ip{\patom_i}{\atom_j}\big| <\theta_1\cdot c(S)\alpha_{\min}\hspace{15pt}\text{and}\hspace{15pt}|\ip{\patom_i}{r}| <\theta_2\cdot c(S)\alpha_{\min},
	\end{equation}
	where $\theta_1$ and $\theta_2$ ensure that
	\begin{align}\label{condition_recovery_Ig}
	c(S)\alpha_{\min} - &\theta_1c(S)\alpha_{\min} - \theta_2c(S)\alpha_{\min} \overset{!}{\geq}c(S+1)\alpha_{\max} + \theta_1c(S)\alpha_{\min} + \theta_2c(S)\alpha_{\min} .
	\end{align}
	Since the conditions above also guarantee the recovery of the correct sign $\sigma(i)$ for all $i \in I$, so in particular the recovery of $\sigma(k)$, we can bound the probability of the event that thresholding fails while $k$ is in the generating support $I$ as 
	\begin{align*}
	&\P\left( \left[ I^t \neq I \vee \signop(\ip{\patom_k}{y})\neq\sigma(k) \right] \wedge k\in I \right) \\
	& \hspace{2cm}\leq  \sum_i \P\Big( \big|\sum_{j\neq i}\sigma (j)c(p(j))\ip{\patom_i}{\atom_j}\big| \geq \theta_1c(S)\alpha_{\min}\wedge k\in I \Big)\nonumber\\
	&\hspace{5cm} +\sum_i \P \big(|\ip{\patom_i}{r}| \geq\theta_2c(S)\alpha_{\min}\wedge k\in I \big) \\
	& \hspace{2cm}\leq  \sum_i \sum_{\ell \in \Sset}\P\Big( \big|\sum_{j\neq i}\sigma (j)c(p(j))\ip{\patom_i}{\atom_j}\big| \geq \theta_1c(S)\alpha_{\min}\big| p(k) = \ell\Big)\cdot \P(p(k) =\ell) \nonumber\\
	&\hspace{5cm} +\sum_i \sum_{\ell \in \Sset} \P \big(|\ip{\patom_i}{r}| \geq\theta_2 c(S)\alpha_{\min}\big| p(k) = \ell\big)\cdot \P(p(k) =\ell).
	\end{align*}
	Since every permutation is equally likely, each index is equally likely to be mapped to $\ell$, meaning $\P(p(k) =\ell)=1/K$. Using the independence of the noise from the remaining signal parameters and its sub-Gaussian property further leads to
	\begin{align}
	&\P\left( \left[ I^t \neq I \vee \signop(\ip{\patom_k}{y})\neq\sigma(k) \right] \wedge k\in I \right) \notag\\
	& \hspace{2cm}\leq  \frac{1}{K} \sum_i \sum_{\ell \in \Sset}\P\Big( \big|\sum_{j\neq i}\sigma (j)c(p(j))\ip{\patom_i}{\atom_j}\big| \geq \theta_1 c(S)\alpha_{\min}\big| p(k) = \ell\Big) \nonumber\\
	&\hspace{6cm} + \frac{S}{K} \sum_i \P \big(|\ip{\patom_i}{r}| \geq\theta_2 c(S)\alpha_{\min}\big)\notag \\
	& \hspace{2cm}\leq  \frac{1}{K} \sum_i \sum_{\ell \in \Sset}\P\Big( \big|\sum_{j\neq i}\sigma (j)c(p(j))\ip{\patom_i}{\atom_j}\big| \geq \theta_1c(S)\alpha_{\min}\big| p(k) = \ell\Big)\notag\\
	&\hspace{6cm} +  2S\exp\left(\frac{-(\theta_2 c(S)\alpha_{\min})^2}{2\rho^2}\right).
	\label{failureprobability_recovery_Iell}
	\end{align}
	To estimate the terms $\P\big( \big|\sum_{j\neq i}\sigma (j)c(p(j))\ip{\patom_i}{\atom_j}\big| \geq \theta_1c(S)\alpha_{\min}\big| p(k) = \ell\big)$,
	we split the sum into two parts; one over $j\in I\setminus \{i, k\}$, that captures most of the energy, and the other over $j\in (I^c\cup \{k\})\setminus\{ i \}$. For $m_1 \in (0,1)$ and $m_2 = 1-m_1$ we have
	\begin{align*}
	&\P\Big( \big|\sum_{j\neq i}\sigma (j)c(p(j))\ip{\patom_i}{\atom_j}\big| \geq \theta_1c(S)\alpha_{\min}\big| p(k) = \ell\Big)\\
	&\qquad\qquad\leq\P\Big(\big|\sum_{j\in I \setminus\{i, k\} }\sigma (j)c(p(j))\ip{\patom_i}{\atom_j}\big| \geq m_1\theta_1c(S)\alpha_{\min}\big| p(k) = \ell\Big) \\
	&\qquad\qquad\qquad\qquad+\P\Big( \big|\sum_{j\in (I^c\cup \{k\})\setminus
		\{ i \}}\sigma (j)c(p(j))\ip{\patom_i}{\atom_j}\big| \geq m_2 \theta_1 c(S)\alpha_{\min}\big| p(k)=\ell \Big) .
	\end{align*}
	The first term we estimate using Proposition~\ref{inner_product_freedman} and the second term using Hoeffding's inequality. With some small simplifications we get for all $i$, including $k$,
	\begin{align*}
	&\P\Big( \big|\sum_{j\neq i}\sigma (j)c(p(j))\ip{\patom_i}{\atom_j}\big| \geq \theta_1c(S)\alpha_{\min}\big| p(k) = \ell\Big)\\
	&\qquad\leq 2 \exp \left( \frac{-( m_1 \theta_1c(S)\alpha_{\min})^2}{2(c(1) \hat \mu \cdot m_1 \theta_1c(S)\alpha_{\min}+ \|c(\Sset)\|_2^2 \frac{B}{K-S})} \right) +2 \exp \left(\frac{-(m_2 \theta_1c(S)\alpha_{\min})^2}{2 \hat \mu^2 ( c(\ell)^2 + \|c(\Sset^c)\|_2^2)}\right) \\
	&\qquad\leq 2 \exp \left( -\frac{1}{4} \min\left\{ \frac{c(S) m_1 \theta_1\alpha_{\min}}{c(1) \hat \mu}, \frac{(K-S)( m_1 \theta_1c(S)\alpha_{\min})^2}{ B \|c(\Sset)\|_2^2} \right\} \right) \\
	&\hspace{5cm}+2 \exp \left( -\frac{1}{4}\min\left\{ \frac{(c(S) m_2 \theta_1\alpha_{\min})^2}{c(1)^2 \hat \mu^2 },  \frac{(c(S) m_2 \theta_1\alpha_{\min})^2}{ \hat{\mu}^2\|c(\Sset^c)\|_2^2}  \right\} \right) .
	\end{align*}
	Substituting the expression above into \eqref{failureprobability_recovery_Iell} we get 
	\begin{align*}
	&\P\left( \left[ I^t \neq I \vee \signop(\ip{\patom_k}{y})\neq\sigma(k) \right] \wedge k\in I \right) \notag\\
	&\hspace{2cm}\leq 2S \exp \left( -\frac{1}{4} \min\left\{ \frac{c(S) m_1 \theta_1\alpha_{\min}}{c(1) \hat \mu}, \frac{(K-S)( m_1 \theta_1c(S)\alpha_{\min})^2}{ B  \|c(\Sset)\|_2^2} \right\} \right) \\
	&\hspace{4cm}+2S \exp \left( -\frac{1}{4}\min\left\{ \frac{(c(S) m_2 \theta_1\alpha_{\min})^2}{c(1)^2 \hat \mu^2 },  \frac{(c(S) m_2 \theta_1\alpha_{\min})^2}{ \hat{\mu}^2\|c(\Sset^c)\|_2^2}  \right\} \right) \\
	&\hspace{6cm}+  2S\exp\left( \frac{-(\theta_2 c(S)\alpha_{\min})^2}{2\rho^2}\right),
	\end{align*}
	where $\theta_1$ and $\theta_2$ have to ensure (\ref{condition_recovery_Ig}) and $m_1 + m_2 =1$.
	From this, whenever 
	\begin{align*}
	\alpha_{\min}\geq\max\Bigg\{ & \frac{1}{1-2\theta_1-2\theta_2}\frac{c(S+1)}{c(S)}\alpha_{\max} , \;  \frac{4n}{m_1\theta_1}\frac{c(1)}{c(S)}\hat{\mu}\log K, \;  \frac{2\sqrt{n}}{m_1\theta_1}\frac{ \|c(\Sset)\|_2}{c(S)}\sqrt{\frac{B\log K}{K-S}}, \\ & \frac{2\sqrt{n}}{(1-m_1)\theta_1}\frac{c(1)}{c(S)}\hat{\mu}\sqrt{\log K}, \;  \frac{2\sqrt{n}}{(1-m_1)\theta_1}\frac{\|c(\Sset^c)\|_2}{c(S)}\hat{\mu}\sqrt{\log K}, \;  \frac{\sqrt{2n}}{\theta_2}\frac{\rho}{c(S)}\sqrt{\log K} \Bigg\} ,
	\end{align*}
	we get that
	\begin{align*}
	\P\left( \left[ I^t \neq I \vee \signop(\ip{\patom_k}{y})\neq\sigma(k) \right] \wedge k\in I \right) \leq 6S\cdot K^{-n} .
	\end{align*}
	Setting $\theta_1=\tfrac{6}{16}$, $\theta_2=\tfrac{1}{16}$, $m_1=\frac{2}{3}$, $n=3$, 
	the probability that thresholding fails to recover $I$ and/or the corresponding signs, restricted to the signals for which we have $k\in I$, is bounded by $6S\cdot K^{-3}$, whenever
	\begin{align*}
	\alpha_{\min}\geq\max\Bigg\{ 8\frac{c(S+1)}{c(S)}\alpha_{\max} ,\; & 48\frac{c(1)}{c(S)}\hat{\mu}\log K,\;  14\frac{c(1)}{c(S)}\sqrt{\frac{SB \log K}{K-S}},\; 40\frac{\rho}{c(S)}\sqrt{\log K} \Bigg\} ,
	\end{align*}
	and $\frac{\|c(\Sset^c)\|_2}{c(1)}\leq\frac{12}{7}\sqrt{\log K}$, where we have used that $\|c(\Sset)\|_2\leq\sqrt{S} c(1)$.

	\subsubsection*{Step~2 - Probability of wrongly recovering $k$ for $k\notin I$ - $\P(k\in I^t|k\notin I)$}
	
	As a second step we will bound the probability of wrongly recovering an atom $\patom_k$ when it is not in the generating support, meaning $k\notin I$. A sufficient condition for not recovering $k$ is that
	\begin{align}
	\min_{i\in I}|\ip{\patom_i}{y}| > |\ip{\patom_k}{y}|.
	\end{align}
	Using the bounds from step~1,
	\begin{align*}
	i\in I :\; \; |\ip{\patom_i}{\dico x_{c,p,\sigma}+r}| &\geq c(S)\alpha_{\min} - \big|\sum_{j\neq i}\sigma(j)c(p(j))\ip{\patom_i}{\atom_j}\big| -|\ip{\patom_i}{r}| , \\
	k \notin I:\; \; |\ip{\patom_k}{\dico x_{c,p,\sigma}+r}| &\leq c(S+1)\alpha_k + \big|\sum_{j\neq k}\sigma(j)c(p(j))\ip{\patom_k}{\atom_j}\big| +|\ip{\patom_k}{r}| ,
	\end{align*}
	we get as sufficient condition for not recovering $k$, that for all $i\in I\cup\{ k \}$
	\begin{equation*}
	\big|\sum_{j\neq i}\sigma(j)c(p(j))\ip{\patom_i}{\atom_j}\big| <\theta_1\cdot c(S)\alpha_{\min}\hspace{15pt}\text{and}\hspace{15pt}|\ip{\patom_i}{r}| <\theta_2\cdot c(S)\alpha_{\min},
	\end{equation*}
	where $\theta_1$ and $\theta_2$ again ensure that
	\begin{align*}
	c(S)\alpha_{\min}-\theta_1c(S)\alpha_{\min}-\theta_2 c(S)\alpha_{\min} 
	\overset{!}{\geq} c(S+1)\alpha_k+\theta_1 c(S)\alpha_{\min}+\theta_2 c(S)\alpha_{\min}.
	\end{align*}
	
	We now bound the probability of thresholding recovering $k$ when it is not in the generating support $I$ as
	\begin{align*}
	\P(k \in I^t \wedge k\notin I ) &= \sum_{\ell > S}\P(k \in I^t \big| p(k)=\ell)\cdot \P(p(k)=\ell)\\
	&\leq \frac{1}{K} \sum_{\ell > S} \sum_{i \in I \cup \{k\}} \P\Big(\big|\sum_{j\neq i}\sigma (j)c(p(j))\ip{\patom_i}{\atom_j}\big| \geq \theta_1\cdot c(S)\alpha_{\min}\big| p(k)=\ell\Big)  \\
	&\qquad +\frac{1}{K} \sum_{\ell > S}\sum_{i \in I \cup \{k\}} \P\big(|\ip{\patom_i}{r}| \geq\theta_2\cdot c(S)\alpha_{\min}\big| p(k)=\ell\big) .
	\end{align*}
	Using the same splitting technique as in step~1, and the sub-Gaussian property of $r$, we get
	\begin{align*}
	\P(k \in I^t \wedge k\notin I )& \leq 2(S+1)\exp \left( -\frac{1}{4} \min\left\{ \frac{c(S) m_1 \theta_1\alpha_{\min}}{c(1) \hat \mu}, \frac{(K-S)( m_1 \theta_1c(S)\alpha_{\min})^2}{ B\|c(\Sset)\|_2^2} \right\} \right) \\
	&+2(S+1)\exp \left(\frac{-(m_2 \theta_1c(S)\alpha_{\min})^2}{2 \hat \mu^2 \|c(\Sset^c)\|_2^2}\right) + 2(S+1) \exp\left( \frac{-(\theta_2 c(S)\alpha_{\min})^2}{2\rho^2}\right).\\
	\end{align*}
	To have this probability sufficiently small, we need to have
	\begin{align*}
	\alpha_{\min}\geq\max\Bigg\{ \frac{1}{1-2\theta_1-2\theta_2}\frac{c(S+1)}{c(S)}\alpha_k , & \; \frac{4n}{m_1\theta_1}\frac{c(1)}{c(S)}\hat{\mu}\log K, \;  \frac{2\sqrt{n}}{m_1\theta_1}\frac{\|c(\Sset)\|_2}{c(S)}\sqrt{\frac{B\log K}{K-S}}, \\ &  \frac{\sqrt{2n}}{(1-m_1)\theta_1}\frac{\|c(\Sset^c)\|_2}{c(S)}\hat{\mu}\sqrt{\log K}, \;  \frac{\sqrt{2n}}{\theta_2}\frac{\rho}{c(S)}\sqrt{\log K} \Bigg\} .
	\end{align*}
	Choosing the same values as before, $\theta_1=\tfrac{6}{16}$, $\theta_2=\tfrac{1}{16}$, $m_1=\frac{2}{3}$, $n=3$, we arrive at the bound
	\begin{align*}
	\P( k\in I^t\wedge k\notin I )\leq 6(S+1)\cdot K^{-3} ,
	\end{align*}
	whenever $\frac{\|c(\Sset^c)\|_2}{c(1)}\leq \frac{12}{5}\sqrt{\log K}$ and
	\begin{align*}
	\alpha_{\min}\geq\max\Bigg\{ 8\frac{c(S+1)}{c(S)}\alpha_k ,\; & 48\frac{c(1)}{c(S)}\hat{\mu}\log K ,\; 14\frac{c(1)}{c(S)}\sqrt{\frac{SB \log K}{K-S}},\; 40\frac{\rho}{c(S)}\sqrt{\log K} \Bigg\} .
	\end{align*}
	Using all these estimates, we are now ready to bound the error originating from the difference between the thresholding and the oracle residual.

	\subsubsection*{Step 3 - Putting it all together}
	
	Using all previous estimates, what remains to do is to estimate the size of $\gsetm$ and $\rsetm$ and finally put all pieces together.
	Inserting our probability estimates into (\ref{probG}) and (\ref{probB}), we get
	\begin{align*}
	\P\left(|\gsetm|\geq N\left( \frac{6S}{K^3}+\frac{C_r\gamma_{1,S}}{3K\sqrt{B+1}}t\varepsilon \right) \right)\leq \exp\left( -\frac{NC_r^2\gamma_{1,S}^2t^2\varepsilon^2}{\frac{108S(B+1)}{K}+3t\varepsilon C_r\gamma_{1,S}K\sqrt{B+1}}\right) 
	\end{align*}
	and
	\begin{align*}
	\P\left(|\rsetm|\geq N\left( \frac{6(S+1)}{K^3}+\frac{C_r\gamma_{1,S}}{3K\sqrt{B+1}}t\varepsilon \right) \right)\leq \exp\left( -\frac{NC_r^2\gamma_{1,S}^2t^2\varepsilon^2}{\frac{108(S+1)(B+1)}{K}+3t\varepsilon C_r\gamma_{1,S}K\sqrt{B+1}}\right) ,
	\end{align*}
	respectively. As we have
	\begin{equation*}
	\Big\|\sum_{n}\left[ R^t(\pdico,y_n,k)-R^o(\pdico,y_n,k)\right]\Big\|_2 \leq (2|\gsetm|  +  |\rsetm|)\sqrt{B+1},
	\end{equation*}
	in summary, we get
	\begin{align*}
	&\P\left(\frac{1}{N}\Big\|\sum_{n}\left[ R^t(\pdico ,y_n,k)-R^o(\pdico ,y_n,k)\right]\Big\|_2 > \frac{18(S+1)\sqrt{B+1}}{K^3} + \frac{C_r \gamma_{1,S}}{K}t\eps\right)\nonumber \\
	&\phantom{\frac{1}{N}\Big\|\sum_{n}\left[ R^t(\pdico ,y_n,k)-R^o(\pdico ,y_n,k)\right]\Big\|_2} \leq 2\exp\left( -\frac{NC_r^2\gamma_{1,S}^2t^2\varepsilon^2}{\frac{108(S+1)(B+1)}{K}+3t\varepsilon C_r\gamma_{1,S}K\sqrt{B+1}}\right) .
	\end{align*}
\end{proof}

\noindent Next we will prove the lemma yielding a bound for the error originating from the difference between the oracle residuals based on the generating dictionary and a perturbation of it.

\subsection{Difference between oracle residuals}\label{subsec:app_tech:oracle_res}

For the proof of Lemma~\ref{lemma_b8s} we will use the vector version of Bernstein's inequality.

\begin{theorem}[Vector Bernstein, \citep{kugr14, gr11, leta91}]\label{vectorbernstein}
	Let $(v_n)_n \in \R^d$ be a finite sequence of independent random vectors. If 
	$\|v_n\|_2 \leq M$ almost surely, $\|\E(v_n)\|_2\leq m_1$ and $\sum_n \E(\|v_n\|_2^2)\leq m_2$, then for all $0\leq t \leq m_2/(M+m_1)$, we have
	\begin{align}
	\P\left(\left\| \sum_n v_n - \sum_n \E(v_n) \right\|_2 \geq t\right)\leq \exp\left(- \frac{t^2}{8m_2}+\frac{1}{4}\right),
	\end{align}
	and, in general,
	\begin{align}
	\P\left(\left\| \sum_n v_n - \sum_n \E(v_n) \right\|_2 \geq t\right)\leq \exp\left(- \frac{t}{8}\cdot \min\left\{\frac{t}{m_2}, \frac{1}{M+m_1}\right\}+\frac{1}{4}\right).
	\end{align}
\end{theorem}
Note that the general statement is simply a consequence of the first part, since for $t\geq m_2/(M+m_1)$ we can choose $m_2=t (M+m_1)$.
\\\\
We next prove that, assuming incoherence and good conditioning of the perturbed dictionary, the oracle residuals based on the perturbed dictionary  $\pdico$ and the generating dictionary $\dico$ are close to each other.

\begin{lemma}\label{lemma_b8s}
	Assume that the signals $y_n$ follow the random model in~\eqref{noisymodel2}. 
	Further, assume that $ S\leq \min \Big\{\frac{K}{98\|\dico\|_{2,2}^2}, \frac{1}{98\nsigma^2}\Big\}$ and that the current estimate of the dictionary $\pdico$ has distance $d(\dico,\pdico)=\eps \geq \frac{1}{32\sqrt{S}}$ but is incoherent and well conditioned, meaning its coherence $\mu(\pdico)$ and its operator norm $\|\pdico\|_{2,2}$ satisfy
	\begin{align}
	\mu(\pdico) \leq \frac{1}{20\log K}  \quad \mbox{and} \quad \|\pdico\|_{2,2}^2 \leq \frac{K}{134e^2 S \log K} -1.
	\end{align}
	Then for all $0\leq t \leq 1/8$ we have
	\begin{align}
	\P&\left(\frac{1}{N} \left\| \sum_n \left[R^o(\pdico, y_n, k)-R^o(\dico, y_n, k) \right]\right\|_2 \geq \frac{C_r\gamma_{1,S}}{K}(0.308\eps+ t\eps) \right) \notag\\&\hspace{7cm}\leq \exp\left(- \frac{NC^2_r\gamma^2_{1,S}t^2 \eps}{12K\max\{S,\|\dico\|_{2,2}^2\}^{\frac{3}{2}}} +\frac{1}{4}\right).\notag
	\end{align} 
\end{lemma}

\begin{proof}
	Throughout the proof we will use the abbreviations $B=\|\dico\|_{2,2}^2$ and $\bar B=\|\pdico\|_{2,2}^2$. We apply Theorem~\ref{vectorbernstein} to $v_n=R^o(\pdico, y_n, k)-R^o(\dico, y_n, k)$ and drop the index $n$ for conciseness.
	From Lemma B.8 in \citep{sc15} we know that $v = T(I,k) y \cdot \sigma(k) \cdot  \chi(I,k)$, where $T(I,k):=P(\dico_{I})-P(\pdico_{I})- P(\atom_k) + P(\patom_k)$, and that 
	\begin{align}
	\E(v) &= \frac{C_r \gamma_{1,S}}{K}\: {K\!-\!1 \choose S\!-\!1}^{-1} \sum_{|I|=S, k\in I} \big[P(\patom_k) -P(\pdico_{I})\big]\atom_k.
	\end{align}
	Using the orthogonal decomposition $\atom_k = [P(\patom_k) + Q(\patom_k)]\atom_k$, where $P(\patom_k)Q(\patom_k)=0$, we get
	\begin{align}
	\E(v) &= \frac{C_r \gamma_{1,S}}{K}\: {K\!-\!1 \choose S\!-\!1}^{-1} \sum_{|I|=S, k\in I} - P(\pdico_{I})Q(\patom_k)\atom_k. 
	\end{align}
	Since the perturbed dictionary $\pdico$ is well-conditioned and incoherent, for most $I$ the subdictionary $\pdico_I$ will be a quasi isometry and
	$P(\pdico_{I}) \approx \pdico_I\pdico_I^\star$. We therefore expand the expectation above, using the abbreviation $p_{K,S}= {K - 1 \choose S-1}^{-1}$, as
	\begin{align*}
	\frac{K}{C_r \gamma_{1,S}}\:\E(v) &=  p_{K,S} \left( \sum_{|I|=S, k\in I} \big[ \pdico_I\pdico_I^\star - P(\pdico_{I})]Q(\patom_k)\atom_k -
	\sum_{|I|=S, k\in I} \pdico_{I\backslash k }\pdico_{I\backslash k }^\star Q(\patom_k)\atom_k\right) \notag \\
	&= p_{K,S} \left( \sum_{|I|=S, k\in I}\big[ \pdico_I\pdico_I^\star - P(\pdico_{I})]Q(\patom_k)\atom_k -
	{K\!-\!2 \choose S\!-\!2} \sum_{j\neq k} \patom_j \patom_j^\star Q(\patom_k)\atom_k\right) \notag \\
	&= p_{K,S} \sum_{|I|=S, k\in I} \big[ \pdico_I\pdico_I^\star - P(\pdico_{I})]Q(\patom_k)\atom_k -
	\frac{S\!-\!1}{K\!-\!1}(\pdico\pdico^\star - \patom_k \patom_k^\star )Q(\patom_k)\atom_k\notag \\
	&= p_{K,S}\sum_{|I|=S, k\in I \atop
		\delta(\pdico_I)\leq \delta_0} \big[ \pdico_I\pdico_I^\star - P(\pdico_{I})]Q(\patom_k)\atom_k   \notag \\
	&\hspace{2cm}+ p_{K,S} \sum_{|I|=S, k\in I \atop
		\delta(\pdico_I)> \delta_0} \big[ \pdico_I\pdico_I^\star - P(\pdico_{I})]Q(\patom_k)\atom_k  - \frac{S\!-\!1}{K\!-\!1}\pdico\pdico^\star  Q(\patom_k)\atom_k.
	\end{align*}
	Since for $\patom_k = \alpha_k\atom_k+\omega_k z_k$ we have $\|Q(\patom_k)\atom_k\|_2 = \omega_k \leq \eps$, we can bound the norm of the expectation above as 
	\begin{align}\label{Ev1}
	\|\E(v)\|_2 &\leq \frac{C_r \gamma_{1,S}}{K}\: \left[ \delta_0 + \P\big(\delta(\pdico_I)> \delta_0 \big| |I|=S, k\in I\big) \cdot (\bar B\!+\!1) + \frac{(S\!-\!1)\bar B }{K\!-\!1}\right] \eps.
	\end{align}
	To estimate the probability of a subdictionary being ill-conditioned we use Chretien and Darses's results on the conditioning of random subdictionaries, which are slightly cleaner and thus easier to handle than the original results by Tropp, \citep{tr08}. Theorem~3.1 of \citep{chda12} reformulated for our purposes and applied to $\pdico$ states that 
	\begin{align}
	\P\big(\delta(\pdico_I)> \delta_0 \big| |I|=S\big) \leq 216 K \exp\left( - \min\left\{ \frac{\delta_0}{2 \mu(\pdico)}, \frac{\delta_0^2 K}{4e^2 S \bar B}\right\}\right).
	\end{align}
	Together with the union bound,
	\begin{align}
	\P\big(\delta(\pdico_I)> \delta_0 \big| |I|=S, k\in I\big)\leq \frac{K}{S} \cdot \P\big(\delta(\pdico_I)> \delta_0 \big| |I|=S\big),
	\end{align}
	this leads to 
	\begin{align}\label{Ev2}
	\|\E(v)\|_2 &\leq \frac{C_r \gamma_{1,S}}{K}\: \left[ \delta_0 +  \frac{216 K^2 (\bar B\!+\!1)}{S} \exp\left( - \min\left\{ \frac{\delta_0}{2 \mu(\pdico)}, \frac{\delta_0^2 K}{4e^2 S \bar B}\right\}\right) + \frac{S \bar B}{K}\right] \eps.
	\end{align}
	Choosing $\delta_0=3/10$, as long as $\bar B \leq \frac{K}{134e^2 S \log K} -1$ and $\mu(\pdico) \leq \frac{1}{20\log K} $ we have 
	\begin{align}
	\|\E(v)\|_2 &\leq 0.308  \cdot \frac{C_r \gamma_{1,S}}{K} \cdot\eps ,
	\end{align}
	where we used that $S\geq 2$ and $\log K\geq 7$.
	The second quantity we need to bound is the expected energy of $v= T(I,k) y \cdot \sigma(k) \cdot  \chi(I,k)$. Combining Eqs.~(115-118) from Lemma B.8 in \citep{sc15} we get that
	\begin{align}
	\E(\|v\|_2^2) &\leq\E_{p}\left(\chi(I,k)\left[4 \gamma_{2,S}\eps^2 + \left( \frac{B(1-\gamma_{2,S})}{K\!-\!S}+\nsigma^2\right)  \|T(I,k)\|^2_F \right]\right) .
	\end{align}
	Since we are only interested in the regime $\eps>O(1/\sqrt{S})$ we will accept an additional factor $S$ in the final sample complexity in return for a crude but painless estimate. Concretely, we use that $T(I,k)$ is the difference of two orthogonal projections onto subspaces of dimension $S\!-\!1$, namely $P(\dico_I)-P(\atom_k)$ and $P(\pdico_I)-P(\patom_k)$. This leads to the bound $\|T(I,k)\|^2_F \leq 2(S\!-\!1)\leq 2S$ and we get
	\begin{align}
	\E(\|v\|_2^2) &\leq \frac{S}{K}\left(4 \gamma_{2,S}\eps^2 + \frac{2BS}{K\!-\!S}(1-\gamma_{2,S})+2S\nsigma^2 \right)
	\leq \frac{S}{K} (4\eps^2 + 1/24),
	\end{align}
	where for the second inequality we have used the assumption $ S\leq \min \{\frac{K}{98B}, \frac{1}{98\nsigma^2}\}$.\\ 
	Combining the estimates for $\| \E(v)\|_2$ and $\E(\|v\|_2^2)$ with the norm bound $\|v\|_2\leq 2\sqrt{B\!+\!1}$,
	we get that for $\eps\geq\frac{1}{32\sqrt{S}}$ and 
	$0\leq t \leq 1/8$ 
	\begin{align}
	&\P\left(\frac{1}{N} \left\| \sum_n \left[R^o(\pdico, y_n, k)-R^o(\dico, y_n, k) \right]\right\|_2 \geq \frac{C_r\gamma_{1,S}}{K}(0.308\eps+ t\eps) \right)\notag\\
	&\hspace{3cm} \leq \exp\left(- \frac{NC_r\gamma_{1,S}t\eps}{8K} \min\left\{ \frac{C_r\gamma_{1,S}t\eps}{S(4 \eps^2 + 1/24 )} ,\frac{1}{\eps +2\sqrt{B\!+\!1}}\right\} +\frac{1}{4}\right)\notag\\
	&\hspace{3cm} \leq \exp\left(- \frac{NC^2_r\gamma^2_{1,S}t^2\eps}{8K} \min\left\{\frac{1 }{S(4\eps +(24\eps)^{-1})} ,\frac{1}{3t \gamma_{1,S}\sqrt{B\!+\!1}}\right\} +\frac{1}{4}\right)\notag\\
	&\hspace{3cm} \leq \exp\left(- \frac{NC^2_r\gamma^2_{1,S}t^2\eps}{8K\max\{S,B\}} \min\left\{ \frac{ 1}{4 \eps + (24\eps)^{-1}} ,\frac{1}{3t\sqrt{2}}\right\} +\frac{1}{4}\right)\notag\\
	&\hspace{3cm} \leq \exp\left(- \frac{NC^2_r\gamma^2_{1,S}t^2 \eps}{12K\max\{S,B\}^{\frac{3}{2}} } +\frac{1}{4}\right).\notag
	\end{align} 
\end{proof}

\newpage
\section{Pseudocode }\label{sec:app_code}

\begin{algorithm}[h]
	\small
	\BlankLine
	\caption{ITKrM augmented for {\bf r}eplacement/{\bf a}daptivity - one iteration} \label{algo:itkrmplus}
	\BlankLine
	\BlankLine
	\nlset{\bf a/r}\KwIn{$\pdico, Y, S, \Gamma, M $ \tcp*{dictionary, signals, sparsity, candidates, \\minimal observations (only for adaptive)}} 
	
	Set: $m =\round{\log d}$, $N_\Gamma = \floor{N/m}$\;
	
	Initialise: $\ppdico = 0$, $\bar{\Gamma} = 0$, $\bar S = 0$\;
	\BlankLine
	\ForEach{$ n $}{
		\BlankLine
		\tcp{basic ITKrM steps}
		$I_{n}^t= \arg\max_{I: | I |=S} \| \pdico_I^\star y_n\|_1$ \tcp*{thresholding}
		$ x_n = \pdico_{I_n^t}^\dagger y_n$ \tcp*{sparse coefficients}
		$ a_n = y_n - \pdico_{I_n^t}x_n$ \tcp*{residual}
		\BlankLine
		\nlset{\bf r}$\tau =0$ \tcp*{simple counter for replacement}
		\nlset{\bf a}$\tau = \left(2\log\big(\frac{2N}{M}\big) \|a_n\|_2^2 +\|\pdico_{I_n^t}x_n\|_2^2\right)/d$ \tcp*{advanced counter for adaptivity}
		\BlankLine
		\ForEach{$k \in I_{n}^t$}{
			$\ppatom_{k} \leftarrow \ppatom_{k} + \big[a_n + P(\patom_k) y_n\big] \cdot \signop(\ip{\patom_k}{y_n})$ \tcp*{atom update}
			\BlankLine
			\If{$|x_n(k)|^2 \geq  \tau$}{
				$v(k)\leftarrow v(k)+1$ \tcp*{atom value update}
			}
		}
		\BlankLine
		\tcp{steps for replacement candidates}
		$i_n=\arg\max_\ell |\ip{\ratom_\ell}{a_n}|$ \tcp*{residual thresholding}
		$\bar \ratom_{i_n}  \leftarrow \bar \ratom_{i_n} + a_n \cdot \signop(\ip{\ratom_{i_n}}{a_n})$ \tcp*{candidate update}
		\BlankLine
		\nlset{\bf r}$\tau_\Gamma = 2\log(2K)/d$ \tcp*{simple counter for replacement}
		\nlset{\bf a}$\tau_\Gamma = 2\log(\frac{2N_\Gamma}{d})/d$ \tcp*{advanced counter for adaptivity}
		\BlankLine
		\If{$|\ip{\ratom_{i_n}}{a_n}|^2\geq \tau_\Gamma \,\| a_n \|_2^2 $}{
			$v_\Gamma(i_n)\leftarrow v_\Gamma(i_n)+1$ \tcp*{candidate value update} 
		}
		\BlankLine
		\If{$n\pmod {N_\Gamma} == 0 \wedge  n < mN_\Gamma$}{
			$\Gamma \leftarrow \left( \bar\ratom_{1}/\| \bar\ratom_{1} \|_2, \dots, \bar\ratom_{L}/\| \bar\ratom_{L} \|_2 \right)$ \tcp*{candidate normalisation}
			$\bar \Gamma = 0$ \tcp*{cand.~iteration restart}
			\nlset{\bf a}$v_\Gamma = 0$ \tcp*{skip for replacement}
		}
		\BlankLine
		\nlset{\bf a}\tcp{steps for estimating sparsity level, skip for replacement}
		\nlset{\bf a}$ \theta = \left(2\log(4K)\|a_n\|_2^2 + \|\pdico_{I_n^t}x_n\|_2^2\right)/d$\;
		\nlset{\bf a}$\bar S \leftarrow \bar S + \sharp\{k : |x_n(k)|^2 \geq  \theta \}$ \tcp*{correct atoms}
		\nlset{\bf a}$\bar S \leftarrow \bar S + \sharp\{k : |\ip{\patom_k}{a_n}|^2 \geq  \theta\}$ \tcp*{missed atoms}
	}
	\BlankLine
	$\pdico \leftarrow \left( \ppatom_{1}/\| \ppatom_{1} \|_2, \dots, \ppatom_{\natoms}/\| \ppatom_{\natoms} \|_2 \right)$ \tcp*{atom normalisation}
	\nlset{\bf a}$\bar S \leftarrow \round{\bar S/N}$ \tcp*{average sparsity level, skip for replacement}
	\BlankLine
	\nlset{\bf a/r}\KwOut{$\pdico, v, \Gamma, v_\Gamma,\bar{S}$\tcp*{estimated sparsity only for adaptivity} }
\end{algorithm}

\begin{algorithm}[p]
	\small
	\BlankLine
	\caption{Replacing coherent atoms ({\bf d}elete, {\bf m}erge, {\bf a}dd)} \label{algo:replacement}
	\BlankLine
	\KwIn{$\pdico, v, \Gamma, v_\Gamma, \mu_{\max}$\tcp*{dict., score, cand., cand.score, threshold}}
	\BlankLine
	Reorder $\Gamma, v_\Gamma$ s.t. $ v_\Gamma(1)\geq v_\Gamma(2) \geq \ldots \geq v_\Gamma(L)$\;
	Find: $(k,k')=\arg\max_{i < j}|\ip{\patom_i}{\patom_{j}}| $\tcp*{most coherent atom pair}
	\BlankLine
	\While{$|\ip{\patom_{k}}{\patom_{k'}}| > \mu_{\max} \wedge \Gamma \neq []$}{
		
		$h = \signop(\ip{\patom_k}{\patom_{k'}})$\;
		\nlset{\bf d} $\patom_{\bar k} = \round{\frac{v(k')}{v(k')+v(k)}} \patom_{k'} + h \round{\frac{v(k)}{v(k')+v(k)}} \patom_{k}$ \tcp*{delete}
		\nlset{\bf m} $\patom_{\bar k} = v(k') \patom_{k'} + h v(k) \patom_{k} $ \tcp*{or merge}
		\nlset{\bf a} $\patom_{\bar k} = \patom_{k'} + h \patom_{k'} $ \tcp*{or add}
		\BlankLine
		
		$\patom_{\bar k} \leftarrow \patom_{\bar k}/\|\patom_{\bar k}\|_2$\;
		Find $\Lambda = \{\ell: \mu_\ell := \max_{i\neq k,k'}\absip{\gamma_\ell}{\patom_i}>|\ip{\patom_{k}}{\patom_{k'}}|\}$\;
		$\Gamma_\Lambda \leftarrow [],  v_\Gamma(\Lambda)\leftarrow []$\tcp*{discard coherent candidates}
		\BlankLine
		\If{$\Gamma \neq []$}{
			$\patom_k \leftarrow \patom_{\bar k}$\tcp*{replace with merged atom}
			$v(k) \leftarrow v(k)+v(k')$ \tcp*{update score of merged atom}
			$\patom_{k'} \leftarrow \gamma_1$\tcp*{replace with most useful candidate}
			\eIf{$\mu_1<\mu_{max}$}{
				$v(k') = v_\Gamma(1)$\tcp*{update with candidate score}} 
			{$v(k') = 0$\tcp*{preferably replaced again}}
			$\gamma_1\leftarrow [],  v_\Gamma(1) \leftarrow []$\tcp*{discard used candidate}
		} 
		Find: $(k,k')=\arg\max_{i < j}|\ip{\patom_i}{\patom_{j}}|$ \tcp*{update most coherent atom pair}
	}
	\BlankLine
	\KwOut{$\pdico, v, \Gamma, v_\Gamma$}
\end{algorithm}

\begin{algorithm}[p]
	\small
	
	\BlankLine
	\caption{Pruning coherent atoms ({\bf m}erge)} \label{algo:prune_coherent}
	\BlankLine
	\KwIn{$\pdico, V=(v_1,\ldots,v_K),\mu_{\max}$\tcp*{dictionary, last $m$ scores, threshold}} 
	\BlankLine
	$\Delta = \emptyset$ \tcp*{initialise set of atoms to delete}
	$H =  \pdico^\star \pdico - I_K $ \tcp*{hollow Gram matrix in absolute}
	
	Find: $(k,k') = \arg \max_{i<j}|H(i,j)|$\tcp*{most coherent atom pair}
	\BlankLine
	\While{$|H(k,k')| > \mu_{\max}$}{
		\BlankLine
		\nlset{\bf m}$\patom_{k} \leftarrow v_{k'}(1) \patom_{k'} + \signop(\ip{\patom_k}{\patom_{k'}}) v_k(1) \patom_{k}$  \tcp*{merge according to most recent score}
		$\patom_{k} \leftarrow \patom_k/\|\patom_k\|_2$\;
		$v_k(1) \leftarrow v_{k'}(1)+ v_k(1)$ \tcp*{update most recent score}
		$\Delta \leftarrow \Delta \cup \{k'\}$\; 
		$H(k,\cdot)\leftarrow 0, H(k',\cdot)\leftarrow 0, H(\cdot,k)\leftarrow 0, H(\cdot,k')\leftarrow 0$ \tcp*{update hollow Gram matrix}
		\BlankLine
		Find: $(k,k') = \arg \max_{j<i}|H(i,j)|$ \tcp*{update most coherent atom pair}
	}
	$\pdico_\Delta \leftarrow [], V_\Delta \leftarrow []$\tcp*{delete atoms}
	\BlankLine
	\KwOut{$\pdico, V$}
\end{algorithm}

\begin{algorithm}[h]
	\small
	\BlankLine
	\caption{Pruning unused atoms \label{algo:prune_unused}}
	\BlankLine
	\KwIn{$\pdico, V=(v_1,\ldots,v_K), M , \delta$\tcp*{dictionary, last $m$ scores, threshold,\\
			maximally pruned atoms}} 
	\ForEach{$k$}{
		$\hat{v}(k) = \max_{i} v_{k} (i)$ \tcp*{maximum of last $m$ scores}
	}
	$\Delta = \{k: \hat{v}(k) < M\}$  \tcp*{atoms with max.scores below threshold}
	\BlankLine
	\If{$|\Delta| > \delta $}{
		Sort: $\hat{v}(i_1)\leq \hat{v}(i_2)\leq \dots \leq \hat{v}(i_K)$ \tcp*{sort max.score}
		$\Delta = \{i_1\ldots i_\delta\})$ \tcp*{$\delta$ atoms with smallest max.scores}
	}
	
	$\pdico_\Delta \leftarrow [], V_\Delta \leftarrow []$\tcp*{delete atoms}
	\BlankLine
	\KwOut{$\pdico, V$}
\end{algorithm}

\begin{algorithm}[h]
	\small
	\BlankLine
	\caption{Adding atoms \label{algo:add}}
	\BlankLine
	\KwIn{$\pdico, V, \Gamma, v_\Gamma, \mu_{\max}, M$}
	\BlankLine
	Sort: $v_\Gamma(i_1)\geq v_\Gamma(i_2)\geq \ldots\geq  v_\Gamma(i_L)$\tcp*{sort according to score}
	$\bar{L} = |\{\ell : v_\Gamma(\ell)\geq d\}|$\tcp*{score above $d$}
	\BlankLine
	\For{$\ell = 1\ldots \bar L$ }{
		\If{$\max_k \absip{\patom_k}{\gamma_{i_\ell}}\leq\mu_{\max}$}{ 
			$\pdico \leftarrow (\pdico, \gamma_{i_\ell})$ \tcp*{in order of score add if incoherent}
			$V\leftarrow (V, M \cdot \mathbf{1})$\tcp*{set last $m$ scores of added atoms to $M$}
		}
	}
	\BlankLine
	\KwOut{$\pdico, V$}
\end{algorithm}


\end{document}